\documentclass[11pt]{article}
\usepackage[left=1in, right=1in, top=1in]{geometry}

\PassOptionsToPackage{numbers, compress}{natbib}

\usepackage[utf8]{inputenc} 
\usepackage[T1]{fontenc}    
\usepackage{url}            
\usepackage{booktabs}       
\usepackage{amsfonts}       
\usepackage{amsmath,mathtools,amsthm}	
\usepackage{amssymb,dsfont,bbm}	
\usepackage{nicefrac}       
\usepackage{microtype}      
\usepackage{xcolor}         
\usepackage{tikz-cd}        

\usepackage{xcolor}
\usepackage{srcltx}
\usepackage{etoolbox}
\usepackage{nameref}
\usepackage{comment}
\usepackage{mathrsfs}
\usepackage{enumerate}
\usepackage{amsfonts} 
\usepackage{nicefrac} 
\usepackage{mathtools}
\usepackage{dsfont,mathrsfs}
\usepackage{xargs}
\usepackage{multirow}
\usepackage{bm}
\usepackage{subfigure}

\usepackage{xr}
\usepackage[colorlinks=true,breaklinks=true,bookmarks=true,urlcolor=blue,citecolor=blue,linkcolor=blue,bookmarksopen=false,draft=false]{hyperref}   
\usepackage{cleveref}
\usepackage{aliascnt}
\usepackage{autonum}
\usepackage{upgreek}
\usepackage{appendix}
\usepackage{natbib}

\numberwithin{equation}{section}

\newtheorem{assum}{A\hspace{-2pt}}

\newtheorem{theorem}{Theorem}
\crefname{theorem}{theorem}{theorems}
\Crefname{theorem}{Theorem}{Theorems}

\newtheorem{lemma}{Lemma}
\crefname{lemma}{lemma}{lemmas}
\Crefname{lemma}{Lemma}{Lemmas}

\newtheorem{remark}{Remark}
\crefname{remark}{remark}{remarks}
\Crefname{remark}{Remark}{Remarks}

\newtheorem{corollary}{Corollary}
\crefname{corollary}{corollary}{corollaries}
\Crefname{corollary}{Corollary}{Corollaries}

\newtheorem{proposition}{Proposition}
\crefname{proposition}{proposition}{propositions}
\Crefname{proposition}{Proposition}{Propositions}

\crefname{definition}{definition}{definitions}
\Crefname{definition}{Definition}{Definitions}

\crefname{example}{example}{examples}
\Crefname{Example}{Example}{Examples}

\crefname{figure}{figure}{figures}
\Crefname{Figure}{Figure}{Figures}

\crefname{table}{table}{tables}
\Crefname{Table}{Table}{Tables}

\crefname{assum}{A\hspace{-2pt}}{A\hspace{-2pt}}
\crefname{assumb}{B\hspace{-2pt}}{B\hspace{-2pt}}
\crefname{assumUGE}{UGE\hspace{-1pt}}{UGE\hspace{-1pt}}
\crefname{assumID}{IND\hspace{-1pt}}{IND\hspace{-1pt}}
\crefname{assumUE}{UE\hspace{-1pt}}{UE\hspace{-1pt}}

\def\supconsteps{\supnorm{\funnoisew}}

\newcommand{\PE}{\mathbb{E}}
\newcommand{\var}{\operatorname{Var}}
\newcommand{\PP}{\mathbb{P}}
\newcommandx{\genericb}[1][1=]{b_{#1}}
\newcommandx{\Constq}[1][1=]{\tilde{\operatorname{C}}_{#1}}
\newcommandx{\Constros}[1][1=]{\operatorname{C}_{\operatorname{Ros},#1}}
\newcommandx{\Constburk}[1][1=]{\operatorname{C}_{\operatorname{Burk}}}
\newcommandx{\driftW}[1][1=]{W_{#1}}

\newcommandx{\metricd}[1][1=]{\mathsf{d}_{#1}}

\newcommandx\invmeasure[1][1=]{\Pi_{#1}}
\newcommandx{\PPjoint}[1][1=]{\PP^{\MKjoint[#1]}}
\newcommandx{\PEjoint}[1][1=]{\PE^{\MKjoint[#1]}}
\newcommandx{\PEMID}[1][1=\alpha]{\PE^{\MK[#1]}}
\newcommandx{\PPMID}[1][1=\alpha]{\PP^{\MK[#1]}}

\newcommand{\supnorm}[1]{\norm{ #1 }[\infty]}
\newcommandx{\MKjoint}[1][1=]{\bar{\operatorname{P}}_{#1}}
\newcommandx\costw[1][1=]{\mathsf{c}_{#1}}

\newcommandx\Intergrdist[1][1=]{\mathbb{M}_{1}(#1)}

\newcommandx{\mmarkov}[1][1=0]{m^{(\Markov)}_{#1}}

\def\Deltafl{\Delta^{\operatorname{(fl)}}}

\def\Conv{\operatorname{Conv}}

\def\H{\mathcal{H}}

\newcommand{\ConstJ}[1]{\mathsf{C}^{(\operatorname{J},#1)}_{\ref{lem:Jk_ell:p_moment}}}
\newcommand{\ConstH}[1]{\mathsf{C}^{(\operatorname{H},#1)}_{\ref{lem:Jk_ell:p_moment}}}

\newcommand{\ConstJb}[2]{\mathsf{C}^{(\boot,\operatorname{J},#1)}_{\ref{lem:expansion}, #2}}
\newcommand{\ConstHb}[2]{\mathsf{C}^{(\boot,\operatorname{H},#1)}_{\ref{lem:expansion}, #2}}

\def\Xset{\mathsf{X}}

\def\Zset{\mathsf{Z}}
\def\Zsigma{\mathcal{Z}}

\def\rset{\mathbb{R}}

\def\nset{\ensuremath{\mathbb{N}}}
\def\nsets{\ensuremath{\mathbb{N}^*}}

\newcommand{\msi}{\mathsf{I}}

\newcommand{\Mat}[1]{{\bf{#1}}}
\def\MatB{B}

\renewcommand{\S}{\mathcal{S}}
\newcommand{\A}{\mathcal{A}}

\def\PMDP{\MKQ}

\newcommand{\bConst}[1]{\operatorname{C}_{{\bf #1}}}

\newcommandx\sequence[4][2=,3=,4=]
{\ifthenelse{\equal{#3}{}}{\ensuremath{\{ #1_{#2 #4}\}}}{\ensuremath{\{ #1_{#2 #4} \}_{#2 \in #3}}}}

\newcommandx\sequenceD[2][2=]
{\ifthenelse{\equal{#2}{}}{\ensuremath{\{ #1\}}}{\ensuremath{\{ #1\!~:\!~#2  \}}}}
\newcommandx\sequenceDouble[4][3=,4=]
{\ifthenelse{\equal{#3}{}}{\ensuremath{\{ (#1_{#3},#2_{#3})\}}}{\ensuremath{\{ (#1_{#3},#2_{#3})\}_{#3 \in #4}}}}

\newcommandx{\sequencen}[2][2=n\in\nset]{\ensuremath{\{ #1, \eqsp #2 \}}}
\newcommandx\sequencens[2][2=n]
{\ensuremath{\{ #1_{#2} \!~:\!~#2\in\nsets\}}}
\newcommandx\sequencet[4]
{\ensuremath{\{ #1{#2_{#3}} \, : \, \eqsp #3 \in #4 \}}}
\def\PE{\mathbb{E}}
\def\P{\mathbb{P}}
\def\ProdB{\Gamma}

\newcommandx{\PVar}[1][1=]{\ensuremath{\operatorname{Var}_{#1}}}
\newcommandx\conststab[1][1=p]{\varkappa_{#1}}

\def\noisecov{\Sigma_\varepsilon}

\def\ProdBdet{G}

\newcommandx{\MK}[1][1=\alpha]{\mathrm{P}_{#1}}
\newcommandx\MKK[1][1=\alpha]{\mathrm{K}_{#1}}
\def\MKQ{\mathrm{P}}

\newcommandx{\PEtilde}[1][1=]{\PE^{\mathrm{K}_{#1}}}
\newcommandx{\PPtilde}[1][1=]{\PP^{\mathrm{K}_{#1}}}

\newcommandx{\norm}[2][2=]{\Vert#1 \Vert_{{#2}}}
\newcommandx{\normLigne}[2][2=]{\Vert#1 \Vert_{{#2}}}
\newcommandx{\normLine}[2][2=]{\Vert#1 \Vert_{{#2}}}
\newcommandx{\normop}[2][2=]{\Vert{#1}\Vert_{{#2}}}
\newcommandx{\normopLigne}[2][2=]{\Vert{#1}\Vert_{{#2}}}
\newcommandx{\normopLine}[2][2=]{\Vert{#1}\Vert_{{#2}}}
\newcommandx{\osc}[2][1=]{\mathrm{osc}_{#1}(#2)}
\newcommandx{\normopadapt}[2][2=]{\left\Vert{#1}\right\Vert_{{#2}}}

\newcommandx{\normlip}[2][2=\operatorname{Lip}]{\Vert#1 \Vert_{{#2}}}
\newcommand{\lip}{\operatorname{L}}
\newcommandx{\lipspace}[1]{\lip_{#1}}

\newcommandx{\CPP}[3][1=]
{\ifthenelse{\equal{#1}{}}{{\mathbb P}\left(\left. #2 \, \right| #3 \right)}{{\mathbb P}_{#1}\left(\left. #2 \, \right | #3 \right)}}
\newcommandx{\CPPtilde}[3][1=]
{\ifthenelse{\equal{#1}{}}{{\tilde{\mathbb P}}\left(\left. #2 \, \right| #3 \right)}{{\tilde{\mathbb P}}_{#1}\left(\left. #2 \, \right | #3 \right)}}

\def\iid{i.i.d.}

\newcommandx{\as}[1][1=\PP]{\ensuremath{#1\, -\mathrm{a.s.}}}

\newcommand{\eqsp}{\;}

\newcommand{\Id}{\mathrm{I}}
\def\prtheta{\bar{\theta}}

\def\utheta{\tilde{\theta}^{\sf (tr)}}
\def\vtheta{\tilde{\theta}^{\sf (fl)}}

\newcommand{\ConstC}{\mathsf{C}}

\newcommandx{\boundmetric}[1][1=]{\kappa_{\MKK[#1]}}

\newcommand{\Jnalpha}[2]{J_{#1}^{(#2)}}
\newcommand{\Hnalpha}[2]{H_{#1}^{(#2)}}

\newcommandx{\Nnorm}[2][1=V]{[ #2]_{#1}}
\newcommandx{\lipnorm}[2][1=g]{[ #1]_{#2}}

\newcommandx{\CPE}[3][1=]{{\mathbb E}^{#3}_{#1}\left[#2\right]}
\newcommandx{\CPEext}[3][1=]{\tilde{\mathbb E}^{#3}_{#1}\left[#2\right]}
\newcommandx{\CPEtilde}[3][1=]{{\tilde{\mathbb E}}^{#3}_{#1}\left[#2\right]}
\newcommandx{\CPEs}[3][1=]{{\mathbb E}^{#3}_{#1}[#2]}

\def\thetalim{\theta^\star}

\def\trace{\operatorname{Tr}}

\newcommand{\rme}{\mathrm{e}}
\newcommand{\rmd}{\mathrm{d}}
\def\funcAw{\mathbf{A}}

\newcommand{\funcA}[1]{\funcAw(#1)}

\def\funcbw{\mathbf{b}}
\newcommand{\funcb}[1]{\funcbw(#1)}
\newcommandx{\zmfuncA}[2][1=]{\tilde{\funcAw}^{#1}(#2)}
\newcommandx{\zmfuncAw}[1][1=]{\tilde{\funcAw}_{#1}}
\newcommandx{\zmfuncb}[2][1=]{\tilde{\funcbw}^{#1}(#2)}
\def\funnoisew{\varepsilon}
\newcommand{\funcnoise}[1]{\funnoisew(#1)}

\newcommandx{\funcct}[2][1=]{\funcctilde^{#1}(#2)}

\def\qcond{\kappa_{Q}}

\def\State{Z}

\newcommand{\1}{\boldsymbol{1}}

\newcommandx{\CovC}[1][1=u]{\operatorname{C}_{#1}}

\def\msz{\mathsf{Z}}

\def\plusinfty{+\infty}

\DeclareMathAlphabet{\mathpzc}{OT1}{pzc}{m}{it}

\def\lyapW{\mathpzc{W}}

\newcommandx{\bias}[1][1=\alpha]{\operatorname{B}_{#1}}
\newcommand{\proba}[1]{\mathbb{P}\left( #1 \right)}

\newcommandx\probaMarkovTilde[2][2=]
{\ifthenelse{\equal{#2}{}}{{\widetilde{\mathbb{P}}_{#1}}}{\widetilde{\mathbb{P}}_{#1}\left[ #2\right]}}

\newcommand{\expe}[1]{\PE \left[ #1 \right]}

\newcommand{\expep}[2]{\PE^{1/#2} \left[ {#1}^{#2} \right]}

\def\mcf{\mathcal{F}}

\newcommand{\indi}[1]{\1_{#1}}

\def\bA{\bar{\mathbf{A}}}

\def\X{{\bf X}}
\def\Y{{\bf Y}}

\def\thetas{\thetalim}

\def\Am{{\bf A}}
\def\bm{{\bf b}}
\def\funcctilde{\tilde{c}_u}

\def\barb{\bar{\mathbf{b}}}
\newcommandx{\driftb}[1][1=p]{\bar{b}_{#1}}

\def\barA{\bar{A}}

\def\Zbf{\mathbf{Z}}

\def\eps{\varepsilon}

\newcommandx{\boldb}[1][1={q}]{\mathsf{b}_{#1}}

\newcommandx{\ConstGW}[1][1={n,\lyapW}]{\operatorname{G}_{#1}}

\newcommandx{\ConstMW}[1][1={n,\lyapW}]{\operatorname{M}_{#1}}

\Crefname{assumprime}{\textbf{A'}\hspace{-1pt}}{\textbf{A'}\hspace{-1pt}}
\crefname{assumprime}{\textbf{A'}}{\textbf{A'}}

\Crefname{assumTD}{\textbf{TD}\hspace{-1pt}}{\textbf{TD}\hspace{-1pt}}
\crefname{assumTD}{\textbf{TD}}{\textbf{TD}}

\Crefname{assumptionC}{\textbf{C}\hspace{-1pt}}{\textbf{C}\hspace{-1pt}}
\crefname{assumptionC}{\textbf{C}}{\textbf{C}}

\Crefname{assumptionM}{\textbf{UGE}\hspace{-1pt}}{\textbf{UGE}\hspace{-1pt}}
\crefname{assumptionM}{\textbf{UGE}}{\textbf{UGE}}

\def\distance{\mathsf{d}}

\newcommandx{\vartconstwas}[1][1=V]{c_{#1}}

\newcommandx{\deltawas}[1][1=*]{\delta_{#1}}

\newcommandx{\wasser}[4][1=\distance,4=]{\mathbf{W}_{#1}^{#4}\left(#2,#3\right)}
\newcommandx{\covcoeff}[2]{\rho_{#1}^{(#2)}}

\newcommand{\dobrush}{\mathsf{\Delta}}
\newcommandx{\dobru}[3][1=,3=]{\dobrush_{#1}^{#3}( #2)}  

\def\eqdef{:=}

\def\qexponent{q}
\def\ppexponent{p}
\def\Markov{\mathrm{M}}

\def\btheta{\bar{\theta}}

\newcommandx{\dlim}[1]{\ensuremath{\stackrel{#1}{\Longrightarrow}}}

\def\boot{\mathsf{b}}
\newcommand{\PPb}{\mathbb{P}^\boot}
\newcommand{\PEb}{\mathbb{E}^\boot}

\def\kolmogorov{\rho^{\Conv}}

\begin{document}


\title{Improved Central Limit Theorem and Bootstrap Approximations for Linear Stochastic Approximation}

\author{Bogdan Butyrin~\footnote{HSE University, Moscow, Russia, \texttt{bbutyrin@hse.ru}. }, \, Eric Moulines~\footnote{Ecole Polytechnique, Paris, France, and Mohamed bin Zayed University of Artificial Intelligence (MBZUAI), UAE,  \texttt{eric.moulines@polytechnique.edu}.}, \, Alexey Naumov~\footnote{HSE University, Moscow, Russia,  \texttt{anaumov@hse.ru}.}, \, Sergey Samsonov~\footnote{HSE University, Moscow, Russia,  \texttt{svsamsonov@hse.ru}.}, \\ Qi-Man Shao~\footnote{Shenzhen International Center of Mathematics, Southern University of Science and Technology,
Xueyuan Blvd., 518000, Shenzhen, P.R. China, \texttt{shaoqm@sustech.edu.cn}.}, \, Zhuo-Song Zhang~\footnote{Shenzhen International Center of Mathematics, Southern University of Science and Technology,
Xueyuan Blvd., 518000, Shenzhen, P.R. China,  \texttt{zhangzs3@sustech.edu.cn}.}
}

\maketitle

\begin{abstract}
In this paper, we refine the Berry–Esseen bounds for the multivariate normal approximation of Polyak–Ruppert averaged iterates arising from the linear stochastic approximation (LSA) algorithm with decreasing step size. We consider the normal approximation by the Gaussian distribution with covariance matrix predicted by the Polyak-Juditsky central limit theorem and establish the rate up to order $n^{-1/3}$ in convex distance, where $n$ is the number of samples used in the algorithm. We also prove a non-asymptotic validity of the multiplier bootstrap procedure for approximating the distribution of the rescaled error of the averaged LSA estimator. We establish approximation rates of order up to $1/\sqrt{n}$ for the latter distribution, which significantly improves upon the previous results obtained by Samsonov et al. (2024).
\end{abstract}

\section{Introduction}
\label{sec:intro}
In this paper we consider the Linear Stochastic Approximation (LSA) algorithm, a simple yet foundational method with various applications in statistics and machine learning \cite{eweda:macchi:1983, benveniste2012adaptive, guo1994stability, kushner2003stochastic}. The LSA procedure addresses the problem of approximating the unique solution $\thetas$ to a linear system of equations given by
\[
\textstyle 
\bA \thetas = \barb\eqsp,
\]
where $\bA \in \rset^{d \times d}$ is a non-degenerate matrix. This approximation is based on a sequence of observations \(\{(\funcA{Z_k}, \funcb{Z_k})\}_{k \in \nset}\), where \(\Am: \Zset \to \rset^{d \times d}\) and \(\bm: \Zset \to \rset^{d}\) are measurable mappings. The sequence \((Z_k)_{k \in \nset}\) consists of independent and identically distributed (\iid\ ) random variables defined on a measurable space \((\Zset, \Zsigma)\) with distribution \(\pi\), satisfying \(\PE[\funcA{Z_k}] = \bA\) and \(\PE[\funcb{Z_k}] = \barb\). Often in the applications \((Z_k)_{k \in \mathbb{N}}\) are not independent and instead form a Markov chain, see \cite{durmus2022finite,mou2021optimal,wu2025markovchains}. In this paper, we do not consider this setting and postpone it as a direction for a future work. Given a sequence of decreasing step sizes \((\alpha_k)_{k \in \nset}\) and an initialization \(\theta_0 \in \rset^{d}\), we define the iterative estimates \((\theta_k)_{k \in \nset}\) and their Polyak–Ruppert averaged counterparts \((\bar{\theta}_n)_{n \in \nset}\) by
\begin{equation}
\label{eq:lsa}
\theta_k = \theta_{k-1} - \alpha_k \left(\funcA{Z_k} \theta_{k-1} - \funcb{Z_k}\right),\quad k \geq 1, \quad \bar{\theta}_n = n^{-1} \sum_{k=0}^{n-1} \theta_k, \quad n \geq 1.
\end{equation} 
The idea of using averaged estimates $\bar{\theta}_n$ was proposed in the works of Ruppert \cite{ruppert1988efficient} and Polyak and Juditsky \cite{polyak1990new, polyak1992acceleration}. Using the averaged iterates $\bar{\theta}_n$ instead of the last iterate $\theta_n$ has been shown to stabilize stochastic approximation procedures and accelerate their convergence. Moreover, it is known (see \cite{polyak1992acceleration}) that the estimator $\bar{\theta}_n$ is asymptotically normal under appropriate regularity conditions on the step sizes \((\alpha_k)_{k \in \mathbb{N}}\) and the noise observations $(\funcA{Z_k})_{k \in \mathbb{N}}$, that is,
\begin{equation}
\label{eq:CLT_fort_prelim}
\sqrt{n}(\bar{\theta}_n - \thetas) \xrightarrow{d} \mathcal{N}(0, \Sigma_{\infty})\,.
\end{equation}
The expression for $\Sigma_{\infty}$ is given below in \Cref{sec:clt_lsa_pr} and corresponds to the preconditioned version of the sequence $\theta_k$, which uses the optimal preconditioner $\bA^{-1}$, see \cite{fort:clt:markov:2015,polyak1992acceleration}.
\par 
Both asymptotic \cite{polyak1992acceleration,borkar:sa:2008} and non-asymptotic \cite{lakshminarayanan2018linear,srikant:1tsbounds:2019,mou2020linear,durmus2022finite} properties of the averaged LSA errors $\bar{\theta}_n - \thetas$ attained lot of research interest. Many of the mentioned works primarily focus on providing the moment bounds and concentration inequalities for the scaled estimation error \(\sqrt{n}(\bar{\theta}_n - \thetas)\). The primary aim of these concentration bounds is to obtain results with explicit dependence on the number of samples $n$, the problem dimension $d$, and other problem-specific quantities related to $\bA$ and the noise observations $(\funcA{Z_k})_{k \in \nset}$. It is also important to study the rate of convergence in \eqref{eq:CLT_fort_prelim} in a sense of appropriate distance between the probability distributions. Recent papers consider approximation either in Wasserstein distance \cite{srikant2024rates}, class of smooth test functions \cite{pmlr-v99-anastasiou19a}, or in convex distance \cite{shao2022berry,samsonov2024gaussian,wu2024statistical}. The latter type of results can be directly applied when ensuring the non-asymptotic validity of the confidence sets for $\thetas$, and we follow the same direction in our paper. 
\par 
The primary aim of the analysis of the approximation rate in \eqref{eq:CLT_fort_prelim} is the need to construct confidence intervals for $\thetas$. The principal difficulty is the fact that $\Sigma_{\infty}$ is unknown in practice, hence, \eqref{eq:CLT_fort_prelim} can not be applied directly. Classical approaches suggest to approximate $\Sigma_{\infty}$ directly based on either plug-in estimates \cite{chen2020aos,wu2024statistical}, or various modifications of batch-mean approach \cite{chen2020aos,zhu2023online_cov_matr,li2024asymptotics}. Typically these methods constructs an estimator $\hat{\Sigma}_n$ of  $\Sigma_{\infty}$, and often provide non-asymptotic on the closeness between $\hat{\Sigma}_n$ and $\Sigma_{\infty}$. Yet the there are only asymptotic guarantees on coverage probabilities of $\thetas$ with constructed confidence sets. The notable exceptions are recent works \cite{samsonov2024gaussian} and \cite{wu2024statistical}, where the authors provide non-asymptotic error bounds for coverage probabilities. The paper \cite{samsonov2024gaussian} considers general LSA setting and multiplier bootstrap procedure adopted from \cite{JMLR:v19:17-370}, while the authors of \cite{wu2024statistical} considered a plug-in based approach for estimating $\Sigma_\infty$ and focused on the particular setting of the temporal difference (TD) learning algorithm. In this paper we revisit the analysis of \cite{samsonov2024gaussian}, derive the error rates in coverage probabilities of order up to $1/\sqrt{n}$. Our contributions can be summarized as follows:
\begin{itemize}
    \item We refine the high-order moment bounds for $\sqrt{n}(\bar{\theta}_n - \thetas)$, improving the previous results of \cite{mou2020linear} and \cite{durmus2022finite}. Namely, our results yield, for $p \geq 2$, the bound
    \begin{equation}
    \label{eq:bound_trace_sigma_inf_optimized}
    \PE^{1/p}\bigl[\norm{\bar{\theta}_{n} - \thetas}^p\bigr] \lesssim \frac{\sqrt{p}  \sqrt{\trace{\Sigma_\infty}}}{\sqrt{n}} + \frac{p^{3/2}}{n^{5/6}}\eqsp,
    \end{equation}
    provided that the step sizes $\alpha_k$ are appropriately chosen. Note that the leading term of this bound aligns with the moment bound for the Gaussian vector $\mathcal{N}(0,\Sigma_{\infty})$.
    \item We establish a Berry–Esseen bound characterizing the rate of normal approximation in \eqref{eq:CLT_fort_prelim} in a sense of convex distance (see \Cref{sec:clt_lsa_pr}) between distributions. We show the approximation rate in \eqref{eq:CLT_fort_prelim} of order up to $n^{-1/3}$, up to logarithmic factors in $n$. This convergence rate improves the previous rate of order $n^{-1/4}$ obtained in \cite{samsonov2024gaussian} for the general LSA procedure, and aligns with the rate achieved in \cite{wu2024statistical} for the particular setting of the temporal difference (TD) learning algorithm. Similar to \cite{samsonov2024gaussian} and \cite{wu2024statistical}, our proof approach builds upon the techniques developed for nonlinear statistics in \cite{shao2022berry}.

    \item We derive an approximation of the distribution of the scaled Polyak–Ruppert estimator $\sqrt{n}(\bar{\theta}_n - \thetalim)$ based on a multiplier bootstrap procedure. In particular, we show that the coverage probabilities of the true value $\thetas$ under the true distribution $\sqrt{n}(\bar{\theta}_n - \thetalim)$ can be approximated by its bootstrap-based counterpart with a rate approaching $n^{-1/2}$ up to logarithmic factors in $n$. This rate is achieved for step sizes of the form $\alpha_k = c_0 / (k+k_0)^{\gamma}$ when $\gamma \to 1$. Our results provide an improvement over the existing non-asymptotic bounds obtained in \cite{samsonov2024gaussian} for similar procedure. The main reason for this improvement is the observation that the distribution of \(\sqrt{n}(\bar{\theta}_n - \thetalim)\) can be effectively approximated by a normal distribution $\mathcal{N}(0, \Sigma_n )$ with a suitably chosen covariance matrix \(\Sigma_n\), bypassing the direct approximation with $\mathcal{N}(0,\Sigma_\infty)$. The obtained rate is in sharp contrast with \cite{wu2024statistical} and other related works based on direct approximating of the limiting covariance $\Sigma_{\infty}$.
\end{itemize}

\par 
\textbf{Notations.} For matrix $A \in \rset^{d \times d}$ we denote by $\norm{A}$ its operator norm. Given a sequence of matrices $\{A_{\ell}\}_{\ell \in \nset}$, $A_{\ell} \in \rset^{d \times d}$, we use the following convention for matrix products: $\prod_{\ell=m}^{k}A_{\ell} = A_{k} A_{k-1} \ldots A_{m}$, where $m \leq k$. For symmetric and positive-definite matrix $Q = Q^\top \succ 0\eqsp, \eqsp Q \in \rset^{d \times d}$, and $x \in \rset^{d}$ we define the corresponding norm $\|x\|_Q = \sqrt{x^\top Q x}$, and define the respective matrix $Q$-norm of the matrix $B \in \rset^{d \times d}$ by $\normop{B}[Q] = \sup_{x \neq 0} \norm{Bx}[Q]/\norm{x}[Q]$. For sequences $a_n$ and $b_n$, we write $a_n \lesssim b_n$ if there exist a constant $c > 0$ such that $a_n \leq c b_n$ for any $n \in \nset$.  In the present text, the following abbreviations are used: "w.r.t." stands for "with respect to", "\iid\ " - for "independent and identically distributed".

\section{Related works}
\label{sec:related-work}
Asymptotic properties of Linear Stochastic Approximation (LSA) algorithms were studied in \cite{polyak1992acceleration, kushner2003stochastic, borkar:sa:2008, benveniste2012adaptive}. These works established asymptotic normality and almost sure convergence under both i.i.d.\ and Markovian noise. Non-asymptotic analyses of LSA (and of the non-linear setting, corresponding to the SGD algorithm) have been carried out in \cite{rakhlin2012making, nemirovski2009robust, bhandari2018finite, lakshminarayanan2018linear, mou2021optimal}, where mean squared error (MSE) bounds for LSA iterates and their Polyak-Ruppert averaged versions were obtained. Further works \cite{mou2020linear, durmus2021tight, durmus2022finite} establish high-probability bounds (moment bounds or Bernstein-type bounds) for the estimation error $\bar{\theta}_n - \thetalim$. However, the concentration bounds for the LSA error given in \cite{mou2020linear, durmus2021tight, durmus2022finite, mou2021optimal} do not yield convergence rates of the rescaled error $\sqrt{n}(\bar{\theta}_n - \thetalim)$ to the normal distribution in Wasserstein or Kolmogorov distance.
\par 
Non-asymptotic convergence rates towards normality were investigated in \cite{pmlr-v99-anastasiou19a} using Stein’s method and measured in terms of the integral probability metric associated with smooth test functions (smoothed Wasserstein distance). Recent advances include \cite{srikant2024rates}, which studied convergence rates in Wasserstein distance for LSA with Markov observations. The bounds derived in these works exhibit less favorable dependence on the trajectory length $n$ than those presented here. Further, \cite{samsonov2024gaussian} analyzed normal approximation rates for $\sqrt{n}(\bar{\theta}_n - \thetalim)$ and obtained convex distance bounds of order $n^{-1/4}$ for general LSA. This result was later improved by \cite{wu2024statistical} for the specific setting of the temporal difference (TD) learning algorithm. In this paper, we show that the actual rate of normal approximation for LSA is also $n^{-1/3}$ up to logarithmic factors, matching the result of \cite{wu2024statistical}. A detailed comparison with these works is provided in the discussion following \Cref{th:shao2022_berry}.
\par 
The bootstrap approach \cite{efron1992bootstrap} is one of the widely used methods for constructing confidence intervals in parametric models. This method has been extensively studied theoretically; see \cite{Chernozhukov2013,Chernozhukov2015,spokoiny2015,Bernolli2019}. In these works, the validity of the bootstrap relies on Gaussian comparison techniques and anticoncentration results, tailored to particular subclasses of convex sets (spherical or rectangular). Bootstrap validity has also been analyzed in the context of spectral projectors of covariance matrices \cite{PTRF2019, jirak2022quantitative}. At the same time, extending classical bootstrap methods to online learning algorithms poses considerable theoretical and practical challenges. In particular, the iterates \( \{\theta_k\}_{k \in \mathbb{N}} \) generated by the iterative scheme \eqref{eq:lsa} are typically not stored in memory, making standard bootstrap methods inapplicable. Instead, one can employ the multiplier bootstrap technique introduced in \cite{JMLR:v19:17-370}, designed specifically for the iterates of Stochastic Gradient Descent (SGD). A non-asymptotic analysis of this procedure was carried out in \cite{samsonov2024gaussian}, which established approximation rates for the distribution of $\sqrt{n}(\bar{\theta}_n - \thetalim)$ of order up to $n^{-1/4}$ in convex distance. In this paper, we show that the actual approximation rate can be significantly faster, up to $n^{-1/2}$. However, the attempt in \cite{JASA2023} to generalize this procedure to the case of Markovian noise leads to an inconsistent method, as demonstrated in \cite[Proposition~1]{liu2023statistical}. Thus, the question of appropriate generalizations of the multiplier bootstrap approach to stochastic approximation algorithms with Markov data remains, to our knowledge, open.
\par 
Other methods for constructing confidence intervals, not based on the bootstrap approach, rely on the direct estimation of the asymptotic covariance matrix $\Sigma_{\infty}$; see, e.g., \cite{chen2020aos,pmlr-v178-li22b,zhu2023online_cov_matr}. In this approach, the authors typically construct an estimator $\widehat{\Sigma}_n$ of $\Sigma_{\infty}$ and provide bounds on $\PE[\norm{\widehat{\Sigma}_n - \Sigma_{\infty}}]$ with explicit dependence on $n$. To our knowledge, within this approach there are no error bounds for the coverage probabilities of $\thetas$ or error rates for approximating, for example, the distribution of the true statistic $\sqrt{n}(\bar{\theta}_n - \thetas)$ with $\mathcal{N}(0,\widehat{\Sigma}_n)$.

\section{Main results}
\label{sec:independent_case}
We begin this section by specifying the set of assumptions that will be used for the non-asymptotic central limit theorem for LSA iterates. To simplify notation and whenever clarity permits, we write simply \( \funcAw_k = \funcA{\State_k} \) and \( \funcbw_k = \funcb{\State_k} \). Starting from equation \eqref{eq:lsa}, algebraic manipulations yield the recurrence
\begin{equation}
\label{eq:main_recurrence_1_step}
\theta_{k} - \theta^{\star} = (\Id - \alpha_{k} \funcAw_k)(\theta_{k-1} - \theta^{\star}) - \alpha_{k} \funnoisew_{k},
\end{equation}
where we have introduced the noise term \( \funnoisew_k= \funcnoise{\State_k} \), defined by
\begin{equation}
\label{eq:def_center_version_and_noise}
\textstyle
\funcnoise{z} =  \zmfuncA{z} \theta^{\star} - \zmfuncb{z}, \quad \zmfuncA{z}  = \funcA{z} - \bA, \quad \zmfuncb{z} = \funcb{z} - \barb.
\end{equation}
The random variable \( \funcnoise{\State_k} \) corresponds to the noise measured at the solution \( \theta^{\star} \). We introduce the following assumptions on $\{Z_k\}$ and mappings $\funcA{\cdot}, \funcb{\cdot}$:

\begin{assum}
\label{assum:iid}
The sequence \(\{\State_k\}_{k \in \mathbb{N}}\) consists of independent and identically distributed (\iid) random variables defined on a probability space \((\Omega,\mathcal{F},\mathbb{P})\) with common distribution \(\pi\).
\end{assum}

\begin{assum}
\label{assum:noise-level}
$\int_{\Zset}\funcA{z}\rmd \pi(z) = \bA$ and $\int_{\Zset}\funcb{z}\rmd \pi(z) = \barb$, with the matrix $-\bA$ being Hurwitz. Moreover, $\supconsteps = \sup_{z \in \msz}\normop{\funcnoise{z}} < \plusinfty$, and the mapping $z \to \funcA{z}$ is bounded, that is, 
\begin{equation}
\label{eq:a_matr_bounded}
\bConst{A} = \sup_{z \in \msz} \normop{\funcA{z}} \vee \sup_{z \in \msz} \normop{\zmfuncA{z}} < \infty\eqsp.
\end{equation}
Moreover, the smallest eigenvalue of the noise covariance matrix $\noisecov = \int_{\Zset} \funcnoise{z}\funcnoise{z}^\top \rmd \pi(z)$ is bounded away from $0$, that is,
\begin{equation}
\label{eq:eig_sigma_eps}
\textstyle \lambda_{\min}:= \lambda_{\min}(\noisecov) > 0\eqsp.
\end{equation}
\end{assum}
The fact that the matrix $-\bA$ is Hurwitz implies that the linear system $\bA \theta = \barb$ has a unique solution $\thetalim$. Moreover, this fact is sufficient to show that
$\normop{\Id - \alpha \bA}[Q]^2 \leq 1 - \alpha a$ for appropriately chosen matrix $Q = Q^{\top} > 0$ and $a > 0$, provided that $\alpha > 0$ is small enough. Precisely, the following proposition holds:
\begin{proposition}[Proposition 1 in \cite{samsonov2024gaussian}]
\label{prop:hurwitz_stability}
Let $-\bA$ be a Hurwitz matrix. Then for any $P = P^{\top} \succ 0$, there exists a unique matrix $Q = Q^{\top} \succ 0$, satisfying the Lyapunov equation $\bA^\top Q + Q \bA = P$. Moreover, setting
\begin{equation}
\label{eq:alpha_infty_def}
\textstyle 
a = \frac{\lambda_{\min}(P)}{2\normop{Q}}\eqsp, \quad
\text{and} \quad \alpha_\infty = \frac{\lambda_{\min}(P)}{2\qcond \normop{\bA}[Q]^{2}}\wedge \frac{\normop{Q}}{\lambda_{\min}(P)} \eqsp,
\end{equation}
where $\qcond = \lambda_{\max}(Q)/\lambda_{\min}(Q)$, it holds for any $\alpha \in [0, \alpha_{\infty}]$ that $\alpha a \leq 1/2$, and
\begin{equation}
\label{eq:contractin_q_norm}
\normop{\Id - \alpha \bA}[Q]^2 \leq 1 - \alpha a\eqsp.   
\end{equation}
\end{proposition}
\begin{remark}
\label{rem:TD_learning}
One of the important particular examples of the LSA procedure is the setting of the temporal difference (TD) learning algorithm \cite{sutton1988learning,sutton:book:2018}. In the TD algorithm, we consider a discounted MDP (Markov Decision Process) given by a tuple $(\S,\A,\PMDP,r,\gamma)$. Where $\S$ and $\A$ stand for state and action spaces, and $\gamma \in (0,1)$ is a discount factor, and we want to evaluate the value function of a \emph{policy} $\nu(\cdot|s)$, which is the distribution over the action space $\A$ at a fixed state $s \in \S$. Many recent contributions to the analysis of TD learning deal with the linear function approximation when $V^{\nu}(s) \approx \varphi^\top(s) \theta$, where $\theta \in \rset^{d}$ and $\varphi(s): \S \to \rset^{d}$ is a feature mapping. Under these conditions, the problem of finding optimal approximation parameters $\thetas$ is reduced to an instance of a linear stochastic approximation problem by the projected Bellman equation \cite{tsitsiklis:td:1997}. All the results given below in \Cref{sec:moment_bounds_lsa_pr,sec:clt_lsa_pr,sec:bootstrap} apply directly to the TD learning with linear function approximation under the generative model assumptions studied in \cite{samsonov2024gaussian} and \cite{wu2024statistical}. Namely, the assumptions \Cref{assum:iid} and \Cref{assum:noise-level} hold, and \Cref{prop:hurwitz_stability} holds with $Q = \Id$ and $P = \bA + \bA^{\top}$, where $\bA$ is a system matrix corresponding to the projected TD learning equations, see \cite[Section~5]{samsonov2024gaussian}.
\end{remark}

We also consider the family of assumptions on the step sizes $\alpha_{k}$. Namely, for $p \geq 2$ consider the following assumptions \Cref{assum:step-size}($p$): 
\begin{assum}[$p$]
\label{assum:step-size}
The step sizes $\{\alpha_{k}\}_{k \in \nset}$ have a form $\alpha_{k} = \frac{c_0}{(k+k_0)^\gamma}$, where $\gamma \in (1/2;1)$ and $c_{0} \in (0; \alpha_{\infty}]$. Assume additionally that
\[
 k_0 \geq \left(\frac{16}{ac_0}\right)^{1/(1-\gamma)} \vee \left(\frac{2p\qcond \bConst{A}^2}{ac_0}\right)^{1/\gamma} \eqsp.
\]
\end{assum}
In our main results we often apply \Cref{assum:step-size}($p$) with $p = \log{d}$. This particular choice of $p$ imposes a logarithmic dependence of $k_0$ upon the problem dimension $d$. This relaxes the polynomial bounds on $d$, which were previously considered in \cite{mou2021optimal}. At this stage we assume that $k_0$ is a fixed constant that does not depend on time horizon $n$ used in \eqref{eq:lsa}.

\subsection{Moment bounds for Polyak-Ruppert averaged LSA iterates.}
\label{sec:moment_bounds_lsa_pr}
We first present results for the \( p \)-th norm of the averaged  LSA error, that is, \( \mathbb{E}^{1/p}[\|\bar{\theta}_n - \theta^{\star}\|^{p}] \), where $\bar{\theta}_n$ is given in \eqref{eq:lsa}. We first define the product of random matrices
\begin{equation}
\label{eq:prod_rand_matr}
\ProdB_{m:k} = \prod_{\ell = m}^k (\Id - \alpha_\ell \funcAw_\ell)\eqsp,\,  m \leq k\eqsp,\ \text{ and } \ProdB_{m:k} = \Id\eqsp, \quad m > k\eqsp.
\end{equation}
Using the recurrence relation \eqref{eq:main_recurrence_1_step},  we obtain the following decomposition of the LSA error:
\begin{equation}
\label{eq:lsa_error}
\theta_k - \thetas = \utheta_k + \vtheta_k\eqsp, \quad \utheta_k = \ProdB_{1:k}(\theta_0 - \thetas)\eqsp, \quad \vtheta_k = - \sum_{\ell=1}^{k}\alpha_{\ell} \ProdB_{\ell+1:k} \funnoisew_\ell\eqsp.
\end{equation}
The term $\utheta_k$ above is a transient term, which reflects the forgetting of the initial error $\theta_0 - \thetas$, while $\vtheta_k$ is a fluctuation term. Controlling the $p$-th order moments of the transient component $\utheta_k$ is essentially equivalent to bounding the $p$-th moment of the product of random matrices $\ProdB_{m:k}$. For this purpose, we use techniques for proving the stability of products of random matrices from \cite{huang2020matrix} and \cite{durmus2021stability}. We establish the following bound, which is referred to as the \emph{exponential stability} of the product of random matrices:
\begin{lemma}
\label{lem:matr_product_as_bound}
Let $p \geq 2$ and assume \Cref{assum:iid}, \Cref{assum:noise-level} , \Cref{assum:step-size}($p \vee \log{d}$). Then for any $k \leq n$, $1 \leq m \leq k$, it holds that 
\begin{equation}
\label{eq:matrix_pth_moment_bound}
\PE^{1/p}\left[ \normop{\ProdB_{m:k}}^{p} \right]  
\leq \sqrt{\qcond} \rme \prod_{\ell=m}^{k}\bigl(1 - \frac{a \alpha_{\ell}}{2}\bigr) \leq \sqrt{\qcond} \rme \exp\bigl\{-\frac{a}{2} \sum_{\ell=m}^{k}\alpha_\ell\bigr\} \eqsp.
\end{equation}
\end{lemma}
The proof of \Cref{lem:matr_product_as_bound} is given in \Cref{supplement:moment_bounds}. We further decompose $\vtheta_k$ based on the perturbation-expansion approach of \cite{aguech2000perturbation}, see also \cite{durmus2022finite}. Namely, we notice that $\vtheta_k$ satisfies the recurrence
$\vtheta_k = (\Id - \alpha_k \funcAw_{k}) \vtheta_{k-1} - \alpha_{k} \funnoisew_{k}$, with $\vtheta_{0} = 0$.
Extracting its linear part, we represent $\vtheta_{k}$ as 
\begin{equation}
\label{eq:decomp_fluctuation}
\vtheta_{k} = \Jnalpha{k}{0}+ \Hnalpha{k}{0} \eqsp,
\end{equation}
where the latter terms are defined by the following pair of recursions
\begin{align}
\label{eq:jn0_main}
&\Jnalpha{k}{0} =\left(\Id - \alpha_{k} \bA\right) \Jnalpha{k-1}{0} - \alpha_{k} \funnoisew_{k}\eqsp, && \Jnalpha{0}{0}=0\eqsp, \\[.1cm]
\label{eq:hn0_main}
&\Hnalpha{k}{0} =\left( \Id - \alpha_{k} \funcAw_{k} \right) \Hnalpha{k-1}{0} - \alpha_{k} \zmfuncAw[k] \Jnalpha{k-1}{0}\eqsp, && \Hnalpha{0}{0}=0\eqsp.
\end{align}
Here the term $\Jnalpha{k}{0}$ represents the leading (w.r.t. $\alpha_k$) part of the error $\vtheta_k$. Informally, one can show that $\PE^{1/2}[\norm{J_{k}^{(0)}}^{2}] \lesssim \alpha_{k}^{1/2}$, and similarly $\PE^{1/2}[\norm{H_k^{(0)}}^{2}] \lesssim \alpha_{k}$. Thus, $J_{k}^{(0)}$ is a leading term of $\vtheta_{k}$ in terms of its moments, and $H_k^{(0)}$ is a remainder one, a phenomenon, that is referred to as a \emph{separation of scales}.
\par 
The linear part $\Jnalpha{k}{0}$ plays an important role in our further analysis. In particular, we note that the outlined representation of the last iterate error \eqref{eq:lsa_error} implies that
\begin{equation}
\label{eq:pr_error_decomposition}
\sqrt{n}(\bar{\theta}_{n} - \thetas) = \frac{1}{\sqrt{n}} 
\sum_{k=1}^{n-1} J_k^{(0)} +  \frac{1}{\sqrt{n}} \sum_{k=1}^{n-1} H_k^{(0)} + \frac{1}{\sqrt n} \sum_{k = 0}^{n-1} \ProdB_{1:k} (\theta_0 - \thetas)\eqsp.
\end{equation}
The representation \eqref{eq:pr_error_decomposition} plays a key role in our subsequent analysis of both the moment bounds and Gaussian approximation for $\sqrt{n}(\bar{\theta}_{n} - \thetas)$. Indeed, this representation allows us to represent the statistic $\sqrt{n}(\bar{\theta}_{n} - \thetas)$, which is non-linear as a function of $Z_1,\ldots,Z_{n-1}$, as a sum of a linear statistic $\frac{1}{\sqrt{n}} 
\sum_{k=1}^{n-1} J_k^{(0)}$ and a remainder non-linear part, which is of smaller scale. We further denote 
\begin{equation}
\label{eq:sigma_n_definition}
\Sigma_n = \frac{1}{n}\PVar[]\bigl[\sum_{k=1}^{n-1} J_k^{(0)}\bigr] = \frac{1}{n} \sum_{k=1}^{n-1} Q_k \noisecov Q_k^\top \eqsp, \quad Q_\ell = \alpha_\ell \sum_{j=\ell}^{n-1} G_{\ell+1:j}\eqsp,  \quad G_{m:k} = \prod_{\ell=m}^k (\Id - \alpha_\ell \bA)\eqsp.
\end{equation}
We also define the sequence $\varphi_n, n \in \nset$, as follows:
\begin{equation}
\label{eq:varphi_n_def}
\varphi_n = 
\begin{cases}
    \frac{2 c_{0}^{3/2}}{(1-3\gamma/2)n^{3\gamma/2 -1/2}}\eqsp, \quad &1/2 < \gamma < 2/3; \\
    \frac{c_{0}^{3/2} \log{n}}{n^{1/2}}\eqsp, \quad &\gamma = 2/3; \\
    \frac{c_{0}^{3/2}}{(3\gamma/2 - 1) n^{1/2}}\eqsp, \quad &2/3 < \gamma < 1\eqsp.
\end{cases}
\end{equation}
As a first main result of this section, we obtain the following $p$-th moment bound with the leading term given by the trace of the covariance matrix $\Sigma_n$. Precisely, the following bound holds:
\begin{theorem}
\label{th:pth_moment_bound}
Let $p \geq 2$ and assume \Cref{assum:iid}, \Cref{assum:noise-level}, and \Cref{assum:step-size}($p \vee \log{d}$). Then, it holds that 
\begin{equation}
\label{eq:p_th_moment_bound_averaged}
\PE^{1/p}\bigl[\norm{\bar{\theta}_{n} - \thetas}^p\bigr] \leq \frac{\ConstC_{\ref{th:pth_moment_bound}, 1} \sqrt{p}  \sqrt{\trace{\Sigma_n}}}{\sqrt{n}} + \Deltafl(n,p,\gamma) + \frac{\ConstC_{\ref{th:pth_moment_bound}, 5} \norm{\theta_0 - \thetas}  }{n}\eqsp,
\end{equation}
where we set
\begin{equation}
\Deltafl(n,p,\gamma) = \frac{\ConstC_{\ref{th:pth_moment_bound}, 2} p^{3/2}}{n^{1/2+\gamma/2}} + \frac{\ConstC_{\ref{th:pth_moment_bound}, 3} p^{5/2} \varphi_n}{n^{1/2}} + \frac{\ConstC_{\ref{th:pth_moment_bound}, 4} p}{n}\eqsp,
\end{equation}
and the constants $\{\ConstC_{\ref{th:pth_moment_bound}, i}\}_{i=1}^5$, depending on $\gamma, \qcond, a, \bConst{A}, c_0, k_0 \text{ and } \supconsteps$, are given in \Cref{sec:proof:proofs_moments}, see \eqref{def:c_i_thrm1}.
\end{theorem}
\begin{remark}
\label{rem:th_1}
In order to study the scaling of the bound \eqref{eq:p_th_moment_bound_averaged} with the problem dimension $d$, we assume the natural scaling $\supconsteps \leq \sqrt{d} \ConstC_\eps$, where $\ConstC_\eps$ is dimension-free. Then \Cref{th:pth_moment_bound} implies that 
\[
\PE^{1/p}\bigl[\norm{\bar{\theta}_{n} - \thetas}^p\bigr] \lesssim \frac{\sqrt{p}  \sqrt{\trace{\Sigma_n}}}{\sqrt{n}} + \frac{p^{3/2} \sqrt{d}}{n^{1/2+\gamma/2}} + \frac{p^{5/2} \sqrt{d} \varphi_n}{n^{1/2}} + \frac{p\sqrt{d}}{n}  \eqsp,
\]
where $\lesssim$ stands for constant not depending upon $p,n,$ and $d$.
\end{remark}

The proof of \Cref{th:pth_moment_bound} is provided in \Cref{sec:proof:proofs_moments}. Note that the leading in $n$ term of the above bound appears with the coefficient $\sqrt{\trace{\Sigma_n}}$, where $\Sigma_n$ is the variance of the linear statistic extracted in the representation \eqref{eq:pr_error_decomposition}. It is possible to switch from the bound provided by \Cref{th:pth_moment_bound} to the moment bound with the leading term matching the CLT covariance given by 
\begin{equation}
\label{eq:asympt_cov_matr} 
\Sigma_{\infty} = \bA^{-1} \noisecov \bA^{-\top},
\end{equation}
and $\noisecov$ is defined in \Cref{assum:noise-level}. Precisely, the following bound holds:
\begin{corollary}
\label{cor:pth_moment_bound}
Assume \Cref{assum:iid}, \Cref{assum:noise-level}, \Cref{assum:step-size}($p \vee \log{d}$). Then, it holds that 
\begin{align}
\label{eq:2_nd_moment_sigma_infty_bound}
\PE^{1/p}\bigl[\norm{\bar{\theta}_{n} - \thetas}^p\bigr] \leq \frac{\ConstC_{\ref{th:pth_moment_bound}, 1} \sqrt{p}  \sqrt{\trace{\Sigma_\infty}}}{\sqrt{n}} &+ \frac{\ConstC_{\ref{cor:pth_moment_bound}} d \sqrt{p}}{n^{3/2-\gamma} } + \Deltafl(n,p,\gamma) + \frac{\ConstC_{\ref{th:pth_moment_bound}, 5} \norm{\theta_0 - \thetas}}{n} \eqsp.
\end{align}
 where the constant $\ConstC_{\ref{cor:pth_moment_bound}}$ is defined in \eqref{def:c5}.
\end{corollary}
The proof of \Cref{cor:pth_moment_bound} is provided in \Cref{sec:proof:proofs_moments}. Optimizing the r.h.s. of \eqref{eq:2_nd_moment_sigma_infty_bound} over $\gamma$, we obtain that the optimal value is $\gamma = 2/3$. This choice implies the moment bound
\begin{equation}
\label{eq:bound_trace_sigma_inf_optimized_moment}
\PE^{1/p}\bigl[\norm{\bar{\theta}_{n} - \thetas}^p\bigr] \lesssim \frac{\sqrt{p}  \sqrt{\trace{\Sigma_\infty}}}{\sqrt{n}} + \frac{p^{3/2}}{n^{5/6}}\eqsp,
 \end{equation}
 where $\lesssim$ stands for constant not depending upon $n$ and $p$. The bound \eqref{eq:bound_trace_sigma_inf_optimized_moment} improves upon previous bounds of this type obtained in \cite{mou2020linear} and \cite{durmus2022finite}. Both of these papers considered constant step-size LSA. \cite[Proposition 5]{durmus2022finite} showed a bound of the form \eqref{eq:bound_trace_sigma_inf_optimized_moment} with a residual term of order $\mathcal{O}(p^{2}/n^{3/4})$. The improvement in the dependence on $n$, compared to the latter paper, arises from the fact that the authors used a summation by parts formula applied to $\bar{\theta}_{n} - \thetas$, which yields a counterpart of \eqref{eq:pr_error_decomposition} with a different linear statistic identified as the leading term. \cite[Theorem 2]{mou2020linear} obtained a counterpart of \eqref{eq:bound_trace_sigma_inf_optimized_moment} for one-dimensional projections of the error. Unlike typical results in linear stochastic approximation, where stepsizes often decay as $n^{-\gamma}$ for $\gamma \in (1/2, 1)$, \cite{mou2020linear} requires a slower rate of $n^{-1/3}$, leading to a second-order term of order $\mathcal{O}(n^{-5/6})$, similar to \eqref{eq:bound_trace_sigma_inf_optimized_moment}.

\subsection{Gaussian approximation for Polyak-Ruppert averaged LSA iterates.}
\label{sec:clt_lsa_pr}
In this section, we analyze the rate of Gaussian approximation for the statistic \(\sqrt{n}(\bar{\theta}_{n} - \theta^\star)\). The result of Polyak and Juditsky \cite{polyak1992acceleration} states that, under assumptions \Cref{assum:iid}-\Cref{assum:step-size}, it holds that 
\begin{equation}
\label{eq:CLT_fort} 
\sqrt{n}(\bar{\theta}_{n} - \thetas) \overset{d}{\rightarrow} \mathcal{N}(0,\Sigma_{\infty})\eqsp, 
\end{equation}
where the asymptotic covariance matrix $\Sigma_{\infty}$ is defined in \eqref{eq:asympt_cov_matr}.
We are interested to quantify the rate of convergence in \eqref{eq:CLT_fort} w.r.t. the available sample size $n$ and other problem parameters, such as dimension $d$. To measure the approximation quality, we use the convex distance, defined for a pair of probability measures $\mu, \nu$ on $\rset^{d}$ as
\begin{equation}
\label{eq:berry-esseen}
\kolmogorov(\mu, \nu) = \sup_{B \in \Conv(\rset^{d})}\left|\mu(B) - \nu(B)\right|\,,
\end{equation}
where \(\Conv(\rset^{d})\) denotes the collection of all convex sets in \(\rset^{d}\). With a slight abuse of notation, we write $\kolmogorov(X, Y)$ for random vectors $X$ and $Y$ defined on the same probability space $(\Omega,\mathcal{F},\PP)$ instead of their distributions under $\PP$ whenever there is no risk of confusion.
\par 
\noindent \textbf{Gaussian approximation with randomized concentration inequalities.} To establish the Gaussian approximation for  \(\sqrt{n}(\bar{\theta}_{n} - \theta^\star)\), we consider it as a non-linear statistic of independent random variables $\State_k$ outlined in \eqref{eq:lsa}. Then we consider this statistic as a sum of a linear term and a remainder term of smaller order in $n$. This framework is presented in \cite{shao2022berry}, and we summarize below the key results that will be later used to establish our findings. In this paragraph, we present all results for statistics defined in terms of the random variables $X_1, \dots, X_n$, rather than $\State_1, \ldots, \State_n$ as used in the remainder of the paper.
\par 
Let \(X_1, \dots, X_n\) be independent random variables taking values in a measurable space \(\Xset\), and consider a \(d\)-dimensional statistic \(T = T(X_1, \dots, X_n)\), which admits the decomposition \(T = W + D\), where
\begin{equation}
\label{eq:W-D-decomposition} 
W = \sum_{\ell = 1}^n \xi_\ell, \quad D: = D(X_1, \ldots, X_n) = T - W\eqsp.
\end{equation}
Here \(\xi_\ell = h_\ell(X_\ell)\), and \(h_\ell: \Xset \to \mathbb{R}^d\) are measurable functions. The term \(D\) represents a potentially nonlinear component of the statistic $T$, which is assumed to be small compared to \(W\) in an appropriate sense. Assume that \(\mathbb{E}[\xi_\ell] = 0\) and \(\sum_{\ell=1}^n \mathbb{E}[\xi_\ell \xi_\ell^\top] = \mathrm{I}_d\). Define \(\Upsilon_n = \sum_{\ell=1}^n \mathbb{E}[\|\xi_\ell\|^3]\). Then, letting \(\eta \sim \mathcal{N}(0,\mathrm{I}_d)\), \cite[Theorem~2.1]{shao2022berry} yields that
\begin{equation}
\label{eq:shao_zhang_bound}
\kolmogorov(T, \eta) \le 259 d^{1/2} \Upsilon_n + 2 \PE[\|W\| \|D\|] + 2 \sum_{\ell=1}^n \PE[\|\xi_\ell\| \|D - D^{(\ell)}\|],
\end{equation}
where $D^{(\ell)} = D(X_1, \ldots, X_{\ell-1}, X_{\ell}^{\prime}, X_{\ell+1}, \ldots, X_n)$, and $(X_1^{\prime}, \dots, X_n^{\prime})$ is an independent copy of $(X_1, \dots, X_n)$. One can modify the bound \eqref{eq:shao_zhang_bound} for the setting when $\sum_{\ell=1}^n \PE[\xi_\ell \xi_\ell^\top] = \Sigma \succ 0$, see \cite[Corollary~2.3]{shao2022berry}. 
\par 
\noindent\textbf{Gaussian approximation for the LSA algorithm}. In the setting of linear stochastic approximation we use the decomposition \eqref{eq:W-D-decomposition}, identify $T = T(\State_1,\ldots,\State_{n-1}) = \sqrt{n} \Sigma_{n}^{-1/2} (\bar{\theta}_{n} - \theta^\star)$, and write
\begin{equation}
\label{eq:linear and nonlinear terms}
W =  \frac{1}{\sqrt{n}} \Sigma_{n}^{-1/2} \sum_{k=1}^{n-1} J_k^{(0)}, \eqsp \quad D =  \frac{1}{\sqrt{n}} \Sigma_{n}^{-1/2} \sum_{k=1}^{n-1} H_k^{(0)} + \frac{1}{\sqrt n} \Sigma_{n}^{-1/2} \sum_{k = 0}^{n-1} \ProdB_{1:k} (\theta_0 - \thetas)\,.
\end{equation}
Changing the order of summation, we get with $Q_{\ell}$ defined in \eqref{eq:sigma_n_definition}, that 
\begin{equation}
\label{eq: linear part}
W = - \frac{1}{\sqrt{n}}
\sum_{\ell=1}^{n-1} \Sigma_{n}^{-1/2} Q_\ell \funnoisew_\ell\eqsp,
\end{equation}
i.e. $W$ is a weighted sum of i.i.d. random vectors with mean zero and $\PE[W W^{\top}] = \Id_{d}$. The decomposition \eqref{eq:linear and nonlinear terms} and \eqref{eq: linear part} allows to apply the general Gaussian approximation result of \eqref{eq:shao_zhang_bound}. Application of the above result requires that the matrix $\Sigma_n$ is non-degenerate, which is guaranteed by the following lemma:
\begin{lemma}
\label{lem:sigma_n_bound}
Let $p \geq 2$ and assume \Cref{assum:iid}, \Cref{assum:noise-level}, \Cref{assum:step-size}($p$). Let also $n \geq k_0 + 1$. Then it holds that
\begin{align}
\label{bound:norm_of_bL_minus_Lstar}
\norm{\Sigma_n - \Sigma_{\infty}} \leq \ConstC_{\ref{lem:sigma_n_bound}} n^{\gamma-1}\eqsp,
\end{align}
where the constant $\ConstC_{\ref{lem:sigma_n_bound}}$ is given in \eqref{eq:def_mathcal_c_inf}.
\end{lemma}
The proof of \Cref{lem:sigma_n_bound} is given in \Cref{appendix:proof_sigma_n_bound}. With Lidskii’s inequality, we obtain that 
\[
\lambda_{\min}(\Sigma_n) \geq \lambda_{\min}(\Sigma_\infty) - \normop{\Sigma_\infty-\Sigma_n}\eqsp. 
\]
Therefore, using \Cref{lem:sigma_n_bound}, we can lower bound $\lambda_{\min}(\Sigma_n)$, provided that $n$ is large enough. This is formalized in the following assumption:
\begin{assum}
\label{assum:sample_size}
The sample size $n$ satisfies the conditions $n \geq k_0 + 1$ and  
$n^{1-\gamma} \geq 2 \ConstC_{\ref{lem:sigma_n_bound}}/\lambda_{\min}(\Sigma_\infty)$.
\end{assum}

With the assumptions above, we obtain the following Gaussian approximation result. 

\begin{theorem}
\label{theo:GAR}
Assume \Cref{assum:iid}, \Cref{assum:noise-level}, \Cref{assum:step-size}($2 \vee \log d$),  \Cref{assum:sample_size}. Then, with $\eta \sim \mathcal{N}(0,\Id)$, 
\begin{align}
\kolmogorov(\sqrt{n}(\bar{\theta}_{n} - \thetas), \Sigma_n^{1/2} \eta) \leq \frac{\ConstC_{\ref{theo:GAR}, 1}}{\sqrt{n}} + \frac{\ConstC_{\ref{theo:GAR}, 2}}{n^{\gamma/2}} + \ConstC_{\ref{theo:GAR}, 3} \varphi_{n} + \frac{\ConstC_{\ref{theo:GAR}, 4} \norm{\theta_0 - \thetas}}{n} \eqsp,
\end{align}
where $\varphi_{n}$ is defined in \eqref{eq:varphi_n_def} and $\ConstC_{\ref{theo:GAR}, 1}$, $\ConstC_{\ref{theo:GAR}, 2}$, $\ConstC_{\ref{theo:GAR}, 3}$, $\ConstC_{\ref{theo:GAR}, 4}$ are constants defined in \eqref{eq:constants_th_2_def}.
\end{theorem}
The constants $\ConstC_{\ref{theo:GAR}, 2} - \ConstC_{\ref{theo:GAR}, 4}$ contain factors that scale as $1/(1-\gamma)$, and the result in the stated form is not valid when setting $\gamma = 1$. At the same time, following the technique of Shao and Zhang \cite[Theorem~3.4]{shao2022berry}, it is possible to show that a counterpart of \Cref{theo:GAR} holds when $\gamma = 1$, at the cost of additional $\log{n}$ factors arising in the r.h.s. of the bound and under additional constraints on the constant $c_0$, which cannot be chosen too small in this case.
\begin{remark}
\label{rem:th_2}
Under a natural scaling $\supconsteps \leq \sqrt{d} \ConstC_\eps$, where $\ConstC_\eps$ is dimension-free, \Cref{theo:GAR} implies that
\begin{align}
    \kolmogorov(\sqrt{n}(\bar{\theta}_{n} - \thetas), \Sigma_n^{1/2} \eta) \lesssim \frac{d^2}{\sqrt{n}} + \frac{d^{3/2}}{n^{\gamma/2}} + d  \varphi_{n} + \frac{d \log(d) \norm{\theta_0 - \thetas}}{n}\eqsp,
\end{align}
where $\lesssim$ stands for inequality up to a constant not depending upon $n$ and $d$.
\end{remark}
The proof of \Cref{theo:GAR} is provided in \Cref{sec:proof:GAR}. Note that the term $\frac{\ConstC_{\ref{theo:GAR}, 1}}{\sqrt{n}}$ above corresponds to the summand $\Upsilon_n$ from \eqref{eq:shao_zhang_bound}, which is related with the sum of third moments of random vectors forming the linear statistic $W$. The result of \Cref{theo:GAR} shows, that the rate of normal approximation of $\sqrt{n}(\bar{\theta}_{n} - \thetas)$ by $\mathcal{N}(0,\Sigma_n)$ improves when the step sizes $\alpha_k$ are chosen to be less aggressive, that is, when the power $\gamma$ approaches $1$ in \Cref{assum:step-size}. As already mentioned, constants $\ConstC_{\ref{theo:GAR}, 2} - \ConstC_{\ref{theo:GAR}, 4}$ scales with $1/(1-\gamma)$, so the latter conclusion applies when the available number of observations $n$ is large. This aligns with the phenomenon, previously observed for the SGD algorithm \cite{shao2022berry}, \cite{sheshukova2025gaussian} and TD learning \cite{wu2024statistical}. 
\par 
Given the result of \Cref{theo:GAR} and \Cref{lem:sigma_n_bound}, it is possible to quantify the rate of convergence in \eqref{eq:CLT_fort}. Precisely, the following result holds.

\begin{theorem}
\label{th:shao2022_berry} 
Assume \Cref{assum:iid}, \Cref{assum:noise-level}, \Cref{assum:step-size}($2 \vee \log d$),  \Cref{assum:sample_size}. Then, with $\eta \sim \mathcal{N}(0,\Id)$, 
\begin{align}
\label{eq:kolmogorov_bound_non_optimized}
\kolmogorov(\sqrt{n}(\bar{\theta}_{n} - \thetas), \Sigma_{\infty}^{1/2} \eta) \leq \frac{\ConstC_{\ref{theo:GAR}, 1}}{\sqrt{n}} + \frac{\ConstC_{\ref{theo:GAR}, 2}}{n^{\gamma/2}} + \ConstC_{\ref{theo:GAR}, 3} \varphi_{n} + \frac{\ConstC_{\ref{theo:GAR}, 4} \norm{\theta_0 - \thetas}}{n} + \frac{\ConstC_{\ref{th:shao2022_berry}}}{n^{1-\gamma}} \eqsp,
\end{align}
where the constant $\ConstC_{\ref{th:shao2022_berry}}$ is given in \eqref{eq:const_C_10_def}. 
\end{theorem}
The proof of \Cref{th:shao2022_berry} is provided in \Cref{sec:proof:GAR}. 
\paragraph{Discussion.} 
The bound established in \Cref{th:shao2022_berry} achieves the optimal normal approximation error rate of $n^{-1/3}$ for Polyak-Ruppert averaged estimates. This optimal rate is attained using step sizes $\alpha_{k} = c_{0}/(k+k_0)^{2/3}$, corresponding to the decay exponent $\gamma = 2/3$ in \eqref{eq:kolmogorov_bound_non_optimized}. 
\par 
This $n^{-1/3}$ rate aligns with recent results for policy evaluation in reinforcement learning. Wu et al.~\cite{wu2024statistical} established the same convergence rate for the temporal difference (TD) learning algorithm. Their analysis employs step sizes scaling as $c_0/k^{2/3}$, which is consistent with the optimal choice predicted by \Cref{th:shao2022_berry}. Related work Wu et al.~\cite{wu2025markovchains} studies TD learning under Markov noise, achieving a slightly slower rate of $n^{-1/4}$ (up to logarithmic factors) in convex distance. Another relevant contribution is provided by Srikant~\cite{srikant2024rates}, who analyzed temporal-difference learning with Markov noise and established a convergence rate of $n^{-1/6}$ in Wasserstein distance for the step sizes $\alpha_k = c_0/k^{2/3}$. Applying the relation between convex distance and Wasserstein distance \cite[Eq.~(3)]{nourdin2022multivariate}, this bound translates to a convergence rate of order $n^{-1/12}$ in convex distance.
\par 
The fastest known rate for $\kolmogorov(\sqrt{n}(\bar{\theta}_{n} - \thetas), \Sigma_{\infty}^{1/2} \eta)$ in the general LSA problem is $n^{-1/4}$ and is due to \cite{samsonov2024gaussian}. Our rate improvement compared to this work is achieved through a tighter analysis of the normal approximation with $\mathcal{N}(0,\Sigma_n)$, which is carried out in \Cref{theo:GAR}. We then estimate $\kolmogorov(\mathcal{N}(0,\Sigma_n),\mathcal{N}(0,\Sigma_\infty))$ using the Gaussian comparison inequality \cite{BarUly86,Devroye2018}. The authors of \cite{samsonov2024gaussian} used a different error decomposition for the statistic $\sqrt{n}(\bar{\theta}_{n} - \thetas)$ based on the summation by parts representation \cite{mou2020linear,durmus2022finite}, with a linear statistic with covariance matrix $\Sigma_{\infty}$. This approach avoids the Gaussian comparison step but induces a slower approximation rate compared to \Cref{th:shao2022_berry}.
\par 
Several related studies \cite{pmlr-v99-anastasiou19a,shao2022berry,sheshukova2025gaussian} have investigated the normal approximation problem \eqref{eq:CLT_fort} for stochastic gradient descent (SGD) algorithms targeting strongly convex objective functions. We provide a comparative analysis of these results relative to our LSA framework. Anastasiou et al.~\cite{pmlr-v99-anastasiou19a} studied both SGD setting and LSA with symmetric positive-definite system matrix $\barA = \barA^{\top} \succ 0$, achieving normal approximation rates of order $n^{-1/2}$ for integral probability metrics $\metricd[{[2]}]$ induced by twice-differentiable test functions. Precise definition of $\metricd[{[2]}]$ is given in \Cref{appendix:proofs}. This result has two important limitations. First, the relation between the Kolmogorov distance and $\metricd[{[2]}]$ metric (see e.g. \cite[Proposition~2.1]{gaunt2023bounding}) suggests that the rate $n^{-1/2}$ translates to the one of $n^{-1/6}$ when considering the Kolmogorov distance. Hence, the implied rate in convex distance is not faster than $n^{-1/6}$, which is substantially slower than the $n^{-1/3}$ rate achieved in \eqref{eq:kolmogorov_bound_non_optimized}. Second, a detailed examination of \cite[Theorem~4]{pmlr-v99-anastasiou19a} reveals that their bound depends on a quantity $\rho(\eta,t)$ which scales, in notations of the current paper, with the sample size $n$ and the step-size exponent $\gamma$. It is not clear that this term can be uniformly bounded independently of $n$, suggesting that the convergence rate in a sense of $\metricd[{[2]}]$ is actually slower than $n^{-1/2}$, depending upon $\gamma$. 
\par 
Shao and Zhang~\cite{shao2022berry} developed the SGD counterpart of our result of \Cref{theo:GAR}. Their analysis focused on Gaussian approximation with the normal distribution $\mathcal{N}(0,\Sigma_n)$ from \eqref{eq:sigma_n_definition}, rather than $\mathcal{N}(0,\Sigma_\infty)$. These results were further developed in Sheshukova et al.~\cite{sheshukova2025gaussian}, where the authors shown a counterpart of \Cref{th:shao2022_berry} with a convergence rate of order $n^{-1/4}$ when setting $\gamma = 3/4$. This rate is slower than the one corresponding to the LSA setting. This gap arises from the nonlinearity of SGD recursions, which introduces an additional error term in the r.h.s. of \eqref{eq:shao_zhang_bound}. 

\subsection{Lower bounds for the LSA algorithm.}
Lower bounds for the convex distance $\kolmogorov(\sqrt{n}(\bar{\theta}_{n} - \thetas), \Sigma_{\infty}^{1/2} \eta)$ were studied in \cite{sheshukova2025gaussian} for the setting of SGD algorithm. The particular instance of this algorithm, which covers also to the LSA setting, can be written as follows. Consider the simplest $1$-dimensional LSA problem with $\bA = 1$, $\barb = 0$, that is, simply the equation
\[
\textstyle 
\theta = 0\eqsp.
\]
and $\funcAw_k = 1$, $\funcbw_{k} \sim \mathcal{N}(0,1)$ for any $k \in \nset$. Here $\thetas = 0$. The corresponding sequence of LSA updates can be written as follows:
\begin{equation}
\label{eq:lower_bound_iter}
\textstyle
\theta_{k+1} = \theta_k - \alpha_k (\theta_k +\xi_{k+1}), \eqsp k \geq 0 \eqsp,
\end{equation}
where $\theta_0 \in \rset$, $\alpha_k = c_0 (1+k)^{-\gamma}$, $1/2 < \gamma < 1$, and $\xi_k = -\funcbw_{k}$ are i.i.d. standard gaussian random variables. Then \cite[Proposition 1]{sheshukova2025gaussian} shows that for large enough $n$ it holds that 
\begin{align}
    \label{eq:lower_bound_kolmogorov}
    \kolmogorov(\sqrt{n} (\btheta_n - \thetas), \mathcal{N}(0, 1)) > \frac{C(\gamma, c_0)}{n^{1-\gamma}} \eqsp,
\end{align}
where $C(\gamma, c_0)$ is a constant that depends upon $\gamma, c_0$. This result implies that the rate of convergence in \Cref{th:shao2022_berry} is optimal for $\gamma \in [2/3, 1)$, since the term $\frac{\ConstC_{\ref{th:shao2022_berry}}}{n^{1-\gamma}}$ dominates the r.h.s. in this regime. Similar result for TD learning was shown in \cite{wu2024statistical}. However, to the best of our knowledge, there is no matching lower bound for the setting when $\gamma \in (1/2, 2/3)$. 

\section{Multiplier bootstrap for LSA}
\label{sec:bootstrap}
To perform statistical inference with the Polyak-Ruppert estimator $\bar{\theta}_{n}$, we propose an online bootstrap procedure that recursively updates the LSA estimate and a set of randomly perturbed LSA trajectories using the same set of noise variables $Z_k$. The proposed method follows the procedure outlined in \cite{JMLR:v19:17-370}. This approach does not rely on the asymptotic distribution of the error $\sqrt{n}(\bar{\theta}_{n} - \thetas)$ and does not require approximation of the covariance matrix $\Sigma_{\infty}$, which is known to be computationally expensive \cite{chen2020aos}. 
\par 
We describe the suggested procedure as follows. We assume that on the same probability space $(\Omega,\mathcal{F},\PP)$ where the sequence $\{Z_k\}_{k \in \nset}$ is defined, we can construct $M \in \nset$ sequences of \iid\ random variables $\{w_{k}^{\ell}\}$, $1 \leq k \leq n$ and $1 \leq \ell \leq M$, which are independent of $\{Z_k\}_{k \in \nset}$. We assume that $\PE[w_{k}^{\ell}] = 1$, $\var[w_{k}^{\ell}] = 1$, and $\PE[|w_{k}^{\ell}-1|^3] = m_3 < \infty$. Using these weight sequences, we recursively update $M$ randomly perturbed LSA estimates according to:
\begin{equation}
\label{eq:lsa_bootstrap}
\begin{split}
\textstyle \theta_{k}^{\boot,\ell} 
&= \textstyle \theta_{k-1}^{\boot,\ell} - \alpha_{k} w_k^{\ell}\{ \funcA{Z_k} \theta_{k-1}^{\boot,\ell} - \funcb{Z_k} \} \eqsp,~~ k \geq 1 \eqsp, ~~ \theta_{0}^{\boot,\ell} = \theta_{0} \eqsp, \\
\textstyle \prtheta_{n}^{\boot,\ell} 
&= \textstyle n^{-1}     \sum_{k=0}^{n-1} \theta_k^{\boot,\ell} \eqsp, ~~n \geq 1 \eqsp.
\end{split}
\end{equation}
These weights add additional random perturbations to the LSA process \eqref{eq:lsa}. We set $\mathcal{Z}^{n-1} = \{Z_{\ell}\}_{1 \leq \ell \leq n-1}$ and use the notation $\PPb = \PP(\cdot | \mathcal{Z}^{n-1})$ and $\PEb = \PE(\cdot | \mathcal{Z}^{n-1})$ for the corresponding conditional probability and expectation. We refer to them as the "bootstrap world" probability and expectation, respectively. We adopt the shorthand notation $\bar{\theta}_{n}^\boot$ for $\bar{\theta}_{n}^{\boot,1}$.
\par 
The fundamental principle behind \eqref{eq:lsa_bootstrap} is that the conditional distribution of the perturbed bootstrap samples $\sqrt{n} (\bar{\theta}_{n}^\boot - \bar{\theta}_n)$ given the observed data $\mathcal{Z}^{n-1}$ (the "bootstrap world" distribution) approximates the distribution of the target quantity $\sqrt{n} (\bar{\theta}_n - \thetas)$. Specifically, \cite{JMLR:v19:17-370} established that
\begin{equation}
\label{eq:boot_validity_supremum}
\sup_{B \in \Conv(\rset^{d})} |\PPb(\sqrt n (\bar{\theta}_{n}^\boot - \bar{\theta}_n) \in B ) - \PP(\sqrt n (\bar{\theta}_n - \thetas) \in B)| \to 0
\end{equation}
in $\PP$-probability as $n \to \infty$. We refer to this result as the asymptotic validity of the procedure \eqref{eq:lsa_bootstrap} and aim to quantify the rate in \eqref{eq:boot_validity_supremum}. While no closed-form expression exists for $\PPb(\sqrt{n}(\bar{\theta}_{n}^{\boot}-\bar{\theta}_n)\in B)$, this probability can be approximated numerically via \eqref{eq:lsa_bootstrap} by simulating a sufficiently large number $M$ of perturbed trajectories. Standard Monte Carlo theory (see, e.g.,~\cite[Section~5.1]{shao2003mathematical}) indicates that this approximation achieves accuracy of order $M^{-1/2}$. Consider the following assumption:
\begin{assum} 
\label{assum:step-size-bootstrap} 
The step size offset $k_0$ satisfies
\begin{equation}
\label{eq:k0:conditions}
\textstyle 
k_0^\gamma \geq \max\bigl\{2h(n) \bConst{A} \sqrt{\qcond}, \; \frac{c_0 h(n)}{\min\{1, \alpha_\infty\}}, \; \frac{8 \bConst{A}^2 c_0 \sqrt{\qcond} \rme h(n)}{a(2-2^\gamma)}, \; \frac{c_0 \log^2(5n)}{\min\{1, \; a\}}\bigr\}\eqsp, 
\end{equation}
where $h(n)$ is defined as
\begin{equation} 
\label{eq:block_size_constraint}
\textstyle 
h(n) := \biggl\lceil \biggl(\frac{8\bConst{A} \sqrt{\qcond} (1+2\log(10n^3d))}{a(2-2^\gamma)}\biggr)^2 \biggr\rceil\eqsp.
\end{equation} 
Additionally, the sample size $n$ must be sufficiently large such that 
\begin{equation}
\label{eq:sample_size_constraint_A5}
\lambda_{\min}(\Sigma_\infty) \geq \frac{8\sqrt{2} \supconsteps^2 \ConstC_{\ref{prop:Qell:bound}}^2 \sqrt{\log{(10 d n)}}}{\sqrt{n}} + \frac{8 \supconsteps^2 \ConstC_{\ref{prop:Qell:bound}}^2 \log{(10 d n)}}{3n} \eqsp.
\end{equation}
\end{assum} 

The condition \Cref{assum:step-size-bootstrap} ensures that the initial step sizes are not too large, which is crucial for the bootstrap approximation to be valid. We now present the main theoretical result of this section. Our analysis focuses on polynomially decaying step sizes $\gamma_n = c_0/(k_0 + n)^\gamma$ with decay exponent $\gamma \in (1/2, 1)$. 
\begin{theorem} 
\label{th:bootstrap_validity}
Assume \Cref{assum:iid}, \Cref{assum:noise-level}, \Cref{assum:step-size}($\log (5n^3) \vee \log d$), \Cref{assum:sample_size}, \Cref{assum:step-size-bootstrap}. Then with $\PP$ -- probability at least $1 - 1/n$ it holds that
\begin{align}
\sup_{B \in \Conv(\rset^{d})} |\PPb(\sqrt n (\bar{\theta}_{n}^\boot - \bar{\theta}_n) \in B ) - \PP(\sqrt n (\bar{\theta}_n - \thetas) \in B)| \leq \frac{\ConstC_{\ref{th:bootstrap_validity}} \norm{\theta_0 - \thetas} + \Delta_{\ref{th:bootstrap_validity}, 1} }{\sqrt{n}} + \frac{\Delta_{\ref{th:bootstrap_validity}, 2}}{n^{\gamma/2}} + \Delta_{\ref{th:bootstrap_validity}, 3} \varphi_n + \frac{\Delta_{\ref{th:bootstrap_validity}, 4}}{n}  \eqsp, 
\end{align}
where $\ConstC_{\ref{th:bootstrap_validity}}$ is a constant and $\{\Delta_{\ref{th:bootstrap_validity}, i}\}_{i=1}^4$ are polynomials in $\log (n)$ that are defined in \Cref{sec:proof:boot_validity}, see \eqref{eq:delta_th_boot_validity}. 
\end{theorem} 
\begin{proof}
\; We provide here a high-level overview of the proof and refer the reader to \Cref{sec:proof:boot_validity} for a detailed exposition. The main ingredient of the proof is a Gaussian approximation via the randomized concentration inequalities approach \cite{shao2022berry}. The latter is carried out both for $\sqrt{n}(\bar{\theta}_n - \thetas)$ under $\PP$ and for $\sqrt{n}(\bar{\theta}_n^\boot - \bar{\theta}_n)$ under $\PPb$. These two results are then combined using a suitable Gaussian comparison inequality. The main steps of the proof are outlined in the diagram presented below:
\begin{figure}[htbp]
\label{fig:high-level-sketch-of-proof}
\centering
\begin{tikzcd}[column sep=140pt, row sep=40pt]
\text{Real world:\quad} \sqrt{n}(\bar{\theta}_n - \thetas)
  \arrow[<->]{r}{\parbox{4cm}{\centering Gaussian approximation\\under $\mathbb{P}$}}
  &\mathcal{N}(0, \Sigma)
  \arrow[<->]{d}{\parbox{3cm}{\centering Gaussian\\comparison}} \\
\text{Bootstrap world: } \sqrt{n}(\bar{\theta}_n^\boot - \bar{\theta}_n)
  \arrow[<->]{r}{\parbox{4cm}{\centering Gaussian approximation\\under $\mathbb{P}^\boot$}}
  &\mathcal{N}(0, \Sigma^\boot)
\end{tikzcd}
\end{figure}
The principal question that arises here is related with the choice of the approximating normal distribution $\Sigma$ and its bootstrap counterpart $\Sigma^{\boot}$. In the earlier work \cite{samsonov2024gaussian}, the authors used $\Sigma = \Sigma_{\infty}$. As indicated by \Cref{theo:GAR} and \Cref{th:shao2022_berry}, this does not appear to be an optimal choice, as it fails to provide an approximation rate faster than $n^{-1/3}$—at least when $\gamma \in (2/3; 1)$—due to the lower bound \eqref{eq:lower_bound_kolmogorov}. At the same time, $\sqrt{n}(\bar{\theta}_n - \thetas)$ can be approximated by $\mathcal{N}(0, \Sigma_n)$ at a rate approaching $1/\sqrt{n}$. This is the reason why we use $\Sigma = \Sigma_n$ in the present paper. The second principal difficulty in the proof is more technical and is related to the fact that applying a randomized concentration approach under $\PPb$ requires a representation
\begin{equation}
\label{eq:linear_part_proof}
\sqrt{n}(\bar{\theta}_n^\boot - \bar{\theta}_n) = W^{\boot} + D^{\boot}\eqsp,
\end{equation}
where $W^{\boot}$ is a linear statistic with $\PEb[W^{\boot} \{W^{\boot}\}^{\top}] =: \Sigma_n^\boot$. Since we aim to prove Gaussian approximation under $\PPb$, by "linear statistic" we mean linearity in the bootstrap weights $w_{\ell}$. In addition to \eqref{eq:linear_part_proof}, we need to ensure that $\Sigma_n^\boot$ is "close" to $\Sigma_n$ in an appropriate sense. We provide a detailed exposition, together with the definition of the statistics $W^{\boot}$, $D^{\boot}$, and $\Sigma_n^\boot$, in \Cref{subsec:linear_part_def}.
\end{proof}

The terms $\{\Delta_{\ref{th:bootstrap_validity}, i}\}_{i=1}^4$ exhibit similar behavior as the constants from \Cref{theo:GAR} and scale with the factor $1/(1-\gamma)$. Thus the special setting of $\gamma = 1$ requires separate treatment and is not covered by \Cref{th:bootstrap_validity} in the present form. Similarly to \Cref{theo:GAR}, we expect that in the particular setting of $\gamma = 1$, the conclusion of \Cref{th:bootstrap_validity} is still valid, probably with additional $\log{n}$ factors appearing in the r.h.s. and given additional constraints on $c_0$. 

\begin{remark}
    \label{remark_th_7}
    Assuming a natural scaling $\supconsteps \leq \sqrt{d} \ConstC_\eps$, where $\ConstC_\eps$ is dimension-free, \Cref{th:bootstrap_validity} writes as
    \begin{multline}
        \sup_{B \in \Conv(\rset^{d})} |\PPb(\sqrt n (\bar{\theta}_{n}^\boot - \bar{\theta}_n) \in B ) - \PP(\sqrt n (\bar{\theta}_n - \thetas) \in B)| \lesssim \frac{d^2 + d^{3/2} \sqrt{\log(dn)}}{\sqrt{n}} \\ + \frac{d^{3/2} \log n + \sqrt{d} \log^{2\gamma} n}{n^{\gamma/2}}  + d \varphi_n + \frac{d^{3/2} \log (dn)}{n} \eqsp.
    \end{multline}
\end{remark}

\paragraph{Discussion.} 

The direct counterpart of \Cref{th:bootstrap_validity} with the slower approximation rate (with order up to order $n^{-1/4}$ up to logarithmic factors in $n$) was obtained in \cite{samsonov2024gaussian}. The main reason for improvement in the current paper is the choice of the approximating matrix $\Sigma$ in \Cref{fig:high-level-sketch-of-proof}. The authors in \cite{samsonov2024gaussian} used $\Sigma = \Sigma_n$ contrary to the choice $\Sigma = \Sigma_\infty$ employed in \Cref{th:bootstrap_validity}. To our knowledge, the closest result to ours is the one of \cite{wu2024statistical}. In this paper within the plug-in methods the authors obtain an estimator $\widehat{\Sigma}_n$ of the asymptotic covariance $\Sigma_{\infty}$ and provide high-probability error bounds 
\[
\kolmogorov(\sqrt{n}(\bar{\theta}_{n} - \thetas),\mathcal{N}(0,\widehat{\Sigma}_n)) \lesssim \frac{1}{n^{1/3}}\eqsp,
\]
which is attained when the step size exponent $\gamma = 2/3$. \Cref{th:bootstrap_validity} provides approximation of order up to $1/\sqrt{n}$ when $\gamma \to 1$.

\section{Proofs}
\label{sec:proofs}

\subsection{Proofs of \Cref{sec:moment_bounds_lsa_pr}}
\label{sec:proof:proofs_moments}
In this section we provide additional details on the perturbation-expansion technique \cite{aguech2000perturbation,durmus2022finite}. Recall that we can represent the fluctuation component of the error $\vtheta_{k}$ defined in \eqref{eq:lsa_error} as $\vtheta_{k} = \Jnalpha{k}{0}+ \Hnalpha{k}{0}$
where the terms $\Jnalpha{k}{0}$ and $\Hnalpha{k}{0}$ are given in \eqref{eq:jn0_main} and \eqref{eq:hn0_main}, respectively. The term $\Hnalpha{k}{0}$ can be further expanded. One can check with simple algebra that for any $L \in \nset$ the term $\Hnalpha{k}{0}$ can be decomposed as 
\begin{equation}
\label{eq:error_decomposition_LSA}
\Hnalpha{k}{0} = \sum_{\ell=1}^{L}\Jnalpha{k}{\ell} + \Hnalpha{k}{L}\eqsp,
\end{equation}
where the terms $\Jnalpha{k}{\ell}$ and $\Hnalpha{k}{\ell}$ are given by the following recurrences:
\begin{equation}
\label{eq:jn_allexpansion_main}
\begin{aligned}
&\Jnalpha{k}{\ell} =\left(\Id - \alpha_{k} \bA\right) \Jnalpha{k-1}{\ell} - \alpha_{k} \zmfuncA{\State_{k}} \Jnalpha{k-1}{\ell-1}\eqsp,
&& \Jnalpha{0}{\ell}=0  \eqsp, \\
& \Hnalpha{k}{\ell} =\left( \Id - \alpha_{k} \funcA{\State_{k}} \right) \Hnalpha{k-1}{\ell} - \alpha_{k} \zmfuncA{\State_{k}} \Jnalpha{k-1}{\ell} \eqsp, && \Hnalpha{0}{\ell}=0 \eqsp.
\end{aligned}
\end{equation}
It is possible to show that, under assumptions \Cref{assum:iid}, \Cref{assum:noise-level}, and \Cref{assum:step-size}, it holds that $\PE^{1/p}[\norm{\Jnalpha{k}{\ell}}^{p}] \leq c_{\ell} \alpha_{k}^{(\ell+1)/2}$, and similarly $\PE^{1/p}[\norm{\Hnalpha{k}{\ell}}^{p}] \leq c_{\ell} \alpha_{k}^{(\ell+1)/2}$, where the constant $c_{\ell}$ can depend upon $d,p$, and problem-related quantities, but not upon $k$. Thus the expansion depth $L$ in \eqref{eq:error_decomposition_LSA} controls the desired approximation accuracy. Our analysis of the $p$-th moment of the last iterate error $\theta_k - \thetas$ will not require the expansion \eqref{eq:decomp_fluctuation}. At the same time, more delicate bounds for $\PE^{1/p}[\norm{\bar{\theta}_{n} - \thetas}^p]$ will require to use \eqref{eq:decomp_fluctuation} and \eqref{eq:error_decomposition_LSA} with $L = 2$. We recall the $p$-th moment bound of last iterate, adapted from \cite[Proposition~4]{samsonov2024gaussian}.
\begin{proposition}
\label{lem:last_moment_bound}
Let $p \geq 2$ and assume \Cref{assum:iid}, \Cref{assum:noise-level}, and \Cref{assum:step-size}($p \vee \log d$). Then for any $1 \leq k \leq n-1$, it holds that  
\begin{align}
\label{eq:last_iter_bound}
\PE^{1/p}[\norm{\theta_k - \thetas}^p] \leq \sqrt{\qcond} \rme \norm{\theta_0 - \thetas} \prod_{\ell=1}^k (1 - \frac{a}{2} \alpha_\ell)  + \ConstC_{\ref{lem:last_moment_bound}} p \sqrt{\alpha_{k}} \eqsp, \text{ where } \ConstC_{\ref{lem:last_moment_bound}} = \sqrt{6} \rme \supconsteps \sqrt{\qcond/a}\eqsp.
\end{align}
\end{proposition}

Now we provide moment bounds for the terms $\Jnalpha{k}{\ell},\Hnalpha{k}{\ell}$, $\ell \in \{0,\ldots,L\}$. 
\begin{lemma}
\label{lem:Jk_ell:p_moment}
Let $p \geq 2$. Assume \Cref{assum:iid}, \Cref{assum:noise-level}, and \Cref{assum:step-size}($p \vee \log d$). Then for any $\ell \in \{0, 1, 2\}$ it holds that 
\begin{subequations}
\begin{equation}
\label{eq:J_0_moment_bound}
\PE^{1/p}[\norm{\Jnalpha{k}{\ell}}^p] \leq \ConstJ{\ell} p^{\ell+1/2} \alpha_k^{(\ell+1)/2}\eqsp, 
\end{equation}
\begin{equation}
\label{eq:H_k_bound}
\PE^{1/p}[\norm{H_k^{(\ell)}}^p] \leq \ConstH{\ell} p^{\ell+1/2} \alpha_{k}^{(\ell+1)/2}\eqsp,
\end{equation}
\end{subequations}
where the constants $\ConstJ{\ell}$, $\ConstH{\ell}$ satisfy the recurrence 
\begin{align}
\label{eq:const_J_0_def}
\ConstJ{0} = \frac{4\sqrt{3} \qcond^{1/2} \supconsteps}{a^{1/2}}\eqsp, \eqsp
\ConstJ{\ell} = \frac{2 \sqrt{6} \qcond^{1/2} \bConst{A}}{a^{1/2}} \ConstJ{\ell-1}\eqsp, \eqsp
\ConstH{\ell} = \frac{12 \qcond^{1/2} \rme}{a} \ConstJ{\ell} \eqsp. 
\end{align}
\end{lemma}
We also state here the lemma, which is instrumental for our further results and bounds $\normop{Q_\ell}$ for matrices $Q_\ell$ defined in \eqref{eq:sigma_n_definition}. 
\begin{lemma}
\label{prop:Qell:bound}
Assume \Cref{assum:iid}, \Cref{assum:noise-level},  and \Cref{assum:step-size}($2 \vee \log(d)$). Then, for any $\ell \in \{1,\ldots,n-1\}$, 
\begin{equation}
\label{def:Lq}
\normop{Q_\ell} \leq \alpha_\ell \sum_{j=\ell}^{n-1} \normop{G_{\ell+1:j}} \leq \ConstC_{\ref{prop:Qell:bound}} \eqsp, \text{ where } \ConstC_{\ref{prop:Qell:bound}} = \qcond^{1/2} \bigl(c_0 + \frac{2}{a (1-\gamma)}\bigr)\eqsp. 
\end{equation}
Moreover, $\sum_{j=1}^{n-1} \norm{G_{1:j}} \leq (1+k_0)^\gamma \ConstC_{\ref{prop:Qell:bound}}/c_0 $.
\end{lemma}

\begin{proof}[Proof of \Cref{th:pth_moment_bound}]
We first define the constants outlined in the statement of the theorem:
\begin{align}
\label{def:c_i_thrm1}
 &\ConstC_{\ref{th:pth_moment_bound}, 1} = 60\rme \eqsp, \, \ConstC_{\ref{th:pth_moment_bound}, 2} = \frac{\sqrt{c_0} \ConstJ{0} \ConstC_{\ref{prop:Qell:bound}} \bConst{A}}{\sqrt{1-\gamma}}\eqsp, \, \ConstC_{\ref{th:pth_moment_bound}, 3} = \ConstJ{2} + \ConstH{2}\eqsp, \, \\
&\ConstC_{\ref{th:pth_moment_bound}, 4} = 60 \ConstC_{\ref{prop:Qell:bound}} \supconsteps\eqsp, \, \ConstC_{\ref{th:pth_moment_bound}, 5} = 1 + \frac{\sqrt{\qcond} \rme (1+k_0)^\gamma}{c_0} (c_0+\frac{2}{a(1-\gamma)}) \eqsp.
\end{align}
Combining the representations \eqref{eq:lsa_error} and \eqref{eq:decomp_fluctuation}, we get 
\begin{equation}
\label{repr:teta_n:minus:teta_star}
\bar{\theta}_{n} - \thetas = n^{-1}\sum_{k=1}^{n-1} \Gamma_{1:k} (\theta_0 - \thetas) + n^{-1}\sum_{k=1}^{n-1}\Jnalpha{k}{0} + n^{-1} \sum_{k=1}^{n-1}\Hnalpha{k}{0}\eqsp.  
\end{equation}
Now we proceed with different terms in \eqref{repr:teta_n:minus:teta_star} separately. Applying \Cref{lem:summ_alpha_k}, we obtain that 
\begin{equation}
\label{eq:moment_bound_init_cond}
\PE^{1/p}\bigl[\norm{n^{-1}\sum_{k=0}^{n-1} \Gamma_{1:k} (\theta_0 - \thetas)}^p\bigr] \leq \frac{\norm{\theta_0 - \thetas} \ConstC_{\ref{th:pth_moment_bound}, 5}}{n}\eqsp.
\end{equation}
Now we proceed with the term $n^{-1}\sum_{k=1}^{n-1}\Jnalpha{k}{0}= -n^{-1} \sum_{\ell=1}^{n-1} Q_{\ell} \funnoisew_{\ell}$,
where $Q_{\ell}$ is defined in \eqref{eq:sigma_n_definition}. Applying the version of Rosenthal inequality due to Pinelis \cite[Theorem~4.1]{pinelis_1994}, we get 
\begin{equation}
\PE^{1/p}\norm{n^{-1}\sum_{\ell=1}^{n-1} Q_{\ell} \funnoisew_{\ell}}^{p} \leq \frac{\Constros[1] p^{1/2} \{\trace{\Sigma_n}\}^{1/2}}{n^{1/2}} + \frac{\Constros[2] p \PE^{1/p}[\max_{1 \leq \ell \leq n} \norm{Q_{\ell} \funnoisew_{\ell}}^{p}]}{n}\eqsp,
\end{equation}
where $\Constros[1] = 60\rme$ and $\Constros[2] = 60$ are constants from \cite{pinelis_1994}. Applying \Cref{prop:Qell:bound}, we get that $\norm{Q_{\ell}} \leq \ConstC_{\ref{prop:Qell:bound}}$, where $\ConstC_{\ref{prop:Qell:bound}}$ is defined in \eqref{def:Lq}. Hence,  
\begin{equation}
\PE^{1/p}\norm{n^{-1}\sum_{\ell=1}^{n-1} Q_{\ell} \funnoisew_{\ell}}^{p} \leq \frac{\Constros[1] p^{1/2} \{\trace{\Sigma_n}\}^{1/2}}{n^{1/2}} + \frac{\Constros[2] p \ConstC_{\ref{prop:Qell:bound}} \supconsteps}{n} \eqsp. 
\end{equation}
Now we proceed with the next-order terms in $n$ corresponding to $n^{-1} \sum_{k=1}^{n-1}\Hnalpha{k}{0}$. 
Note that $\Hnalpha{k}{0} = \Jnalpha{k}{1} + \Hnalpha{k}{1}$ where $\Jnalpha{k}{1}$ and $\Hnalpha{k}{1}$ 
are given by $\Jnalpha{k}{1} = -\sum_{\ell=1}^{k} \alpha_{\ell} \ProdBdet_{\ell+1:k} \zmfuncAw[\ell] \Jnalpha{\ell-1}{0}$ and $\Hnalpha{k}{1} = -\sum_{m=1}^k \alpha_m \Gamma_{m+1:k} \Jnalpha{m}{1}$.
Applying \Cref{lem:Jk_ell:p_moment} and Minkowski's inequality, we obtain that
\begin{align}
\PE^{1/p}[\norm{\Hnalpha{k}{1}}^p] \leq \PE^{1/p}[\norm{\Jnalpha{k}{2}}^p] + \PE^{1/p}[\norm{\Hnalpha{k}{2}}^p] \leq (\ConstJ{2} + \ConstH{2}) p^{5/2} \alpha_{k}^{3/2}\eqsp.
\end{align}
which implies that 
\begin{align}
n^{-1}\PE^{1/p}[\norm{\sum_{k=1}^{n-1}\Hnalpha{k}{1}}^p] \leq n^{-1} (\ConstJ{2} + \ConstH{2}) p^{5/2} \sum_{k=1}^{n-1} \alpha_{k}^{3/2} \leq \frac{\varphi_n}{n^{1/2}} (\ConstJ{2} + \ConstH{2}) p^{5/2}\eqsp,
\end{align}
where the sequence $\varphi_n$ is defined in \eqref{eq:varphi_n_def}. It remains to proceed with $\sum_{k=1}^{n-1}\Jnalpha{k}{1}$. Note that 
\begin{align}
\sum_{k=1}^{n-1} \Jnalpha{k}{1} = -\sum_{k=1}^{n-1} \sum_{\ell=1}^{k} \alpha_{\ell} \ProdBdet_{\ell+1:k} \zmfuncAw[\ell] \Jnalpha{\ell-1}{0} = -\sum_{\ell=1}^{n-1} Q_{\ell} \zmfuncAw[\ell] \Jnalpha{\ell-1}{0}\eqsp. 
\end{align}
Using the fact that $Q_{\ell} \zmfuncAw[\ell] \Jnalpha{\ell-1}{0}$ is a martingale-difference w.r.t. $\mcf_{\ell-1} = \sigma(Z_s : s \leq \ell-1)$, we obtain, applying Burkholder's inequality \cite[Theorem 8.6]{osekowski:2012} and \Cref{lem:Jk_ell:p_moment}, 
\begin{align}
\frac{1}{n}\PE^{1/p}\bigl[\norm{\sum_{k=1}^{n} \Jnalpha{k}{1}}^p\bigr] \leq \frac{p}{n} \biggl(\sum_{\ell=1}^{n-1} \PE^{2/p} \normop{Q_\ell \tilde\funcAw_\ell \Jnalpha{\ell-1}{0}}^p\biggr)^{1/2} 
&\leq  \frac{\ConstJ{0} \ConstC_{\ref{prop:Qell:bound}} \bConst{A} p^{3/2} \bigl\{\sum_{\ell=1}^{n-1} \alpha_\ell \bigr\}^{1/2}}{n}  \\
&\leq  \frac{\sqrt{c_0} \ConstJ{0} \ConstC_{\ref{prop:Qell:bound}} \bConst{A}}{\sqrt{1-\gamma}} \frac{p^{3/2}}{n^{(1+\gamma)/2}}\eqsp.
\end{align}
It remains to combine the above bounds in \eqref{repr:teta_n:minus:teta_star}.
\end{proof}

\begin{proof}[Proof of \Cref{cor:pth_moment_bound}]
\; First, we define the constant $\ConstC_{\ref{cor:pth_moment_bound}}$ outlined in the statement:
\begin{align}
\label{def:c5}
    \ConstC_{\ref{cor:pth_moment_bound}} = \frac{ \ConstC_{\ref{th:pth_moment_bound}, 1} \ConstC_{\ref{lem:sigma_n_bound}}}{2\sqrt{\trace{(\Sigma_\infty)}}} \eqsp,
\end{align}
where the constant $\ConstC_{\ref{lem:sigma_n_bound}}$ is introduced in \Cref{lem:sigma_n_bound}.
Using that for $a > 0$, $b \geq 0$, $\sqrt{a+b} = \sqrt{a(1+b/a)} \leq \sqrt{a} + b/(2\sqrt{a})$ and \Cref{lem:sigma_n_bound}, we get
\begin{align}
    \sqrt{\trace{(\Sigma_n)}} \leq \sqrt{\trace{(\Sigma_\infty)}} + \frac{|\trace{(\Sigma_n -\Sigma_\infty)|}}{2\sqrt{\trace{(\Sigma_\infty)}}} \leq \sqrt{\trace{(\Sigma_\infty)}} + \frac{d\ConstC_{\ref{lem:sigma_n_bound}} n^{\gamma-1}}{2\sqrt{\trace{(\Sigma_\infty)}}}  \eqsp.  
\end{align}
The proof is concluded using \Cref{th:pth_moment_bound}.
\end{proof}

\subsection{Proofs of \Cref{sec:clt_lsa_pr}}
\label{sec:proof:GAR}
We first define the constants outlined in the statement of the theorem:
\begin{align}
   \label{eq:constants_th_2_def}
   &\ConstC_{\ref{theo:GAR}, 1} = 259 d^{3/2} \ConstC_{\ref{lem:sigma_n_labmda_min_lowerbound}}^{-1/2} \supconsteps \ConstC_{\ref{prop:Qell:bound}} \eqsp, \\ 
   &\ConstC_{\ref{theo:GAR}, 2} = d^{1/2} \frac{2^{11/2} \bConst{A} \ConstC_{\ref{prop:Qell:bound}} \ConstJ{0} \sqrt{c_0} \supconsteps}{\ConstC_{\ref{lem:sigma_n_labmda_min_lowerbound}}^{1/2} \sqrt{1-\gamma}}   + 2d^{1/2} \ConstC_{\ref{lem:sigma_n_labmda_min_lowerbound}}^{-1} \supconsteps \ConstC_{\ref{prop:Qell:bound}} \frac{\sqrt{c_0}}{1-\gamma} \ConstC_{\ref{lem:Hk:copy_difference}} \eqsp, \\ 
   &\ConstC_{\ref{theo:GAR}, 3} = \sqrt{d} (2^{7/2} \ConstC_{\ref{lem:sigma_n_labmda_min_lowerbound}}^{-1/2} \ConstJ{2} + 2^{7/2} \ConstC_{\ref{lem:sigma_n_labmda_min_lowerbound}}^{-1/2} \ConstH{2}) \eqsp, \\ &\ConstC_{\ref{theo:GAR}, 4} = d^{1/2} \ConstC_{\ref{lem:sigma_n_labmda_min_lowerbound}}^{-1/2} \ConstC_{\ref{th:pth_moment_bound}, 5} + 2d^{1/2} \frac{2(1+k_0)^\gamma}{\ConstC_{\ref{lem:sigma_n_labmda_min_lowerbound}}} \supconsteps\ConstC_{\ref{lem:gamma_copies_difference}} (1 + \frac{4}{ac_0(1-\gamma)}) \eqsp,
\end{align}
\begin{proof}[Proof of \Cref{theo:GAR}]
Recall that we use the representation 
$\sqrt{n}\Sigma_{n}^{-1/2}(\bar{\theta}_{n} - \thetas) = W + D\eqsp,$
where
\begin{equation}
\label{eq:linear_and_nonlinear_terms_appendix}
W =  \frac{1}{\sqrt{n}} 
\Sigma_{n}^{-1/2} \sum_{k=1}^{n-1} J_k^{(0)}, \eqsp D =  \frac{1}{\sqrt{n}} \Sigma_{n}^{-1/2} \sum_{k=1}^{n-1} H_k^{(0)} + \frac{1}{\sqrt n} \Sigma_{n}^{-1/2} \sum_{k = 0}^{n-1} \ProdB_{1:k} (\theta_0 - \thetas)\,.
\end{equation}
Alternative representation for $W$ is given in \eqref{eq: linear part}. Recall that we write $\eta$ for a random vector with standard normal distribution $\eta \sim \mathcal{N}(0,\Id_d)$ under $\PP$. Then, setting 
\begin{equation}
\label{eq:xi_ell_def}
\xi_k = \frac{1}{\sqrt{n}} (\Sigma_n)^{-1/2} Q_k \eps_k\eqsp, \quad \Upsilon_n = \sum_{k=1}^{n-1}\PE[\norm{\xi_k}^3]\eqsp,
\end{equation}
we obtain from \cite[Theorem 2.1]{shao2022berry}:
\begin{equation}
\label{eq:shao_zhang_bound_normal_approx}
\kolmogorov(\sqrt{n}\Sigma_{n}^{-1/2}(\bar{\theta}_{n} - \thetas), \eta) \leq 259 d^{1/2} \Upsilon_n + 2 \PE[\|W\| \|D\|] + 2 \sum_{\ell=1}^{n-1} \PE[\|\xi_\ell\| \|D - D^{(\ell)}\|].
\end{equation}
Note that \Cref{lem:sigma_n_labmda_min_lowerbound} and \Cref{prop:Qell:bound} imply $\normop{\xi_k} \leq \frac{1}{\sqrt{n}} \ConstC_{\ref{lem:sigma_n_labmda_min_lowerbound}}^{-1/2} \supconsteps \ConstC_{\ref{prop:Qell:bound}}$.
Hence, 
\begin{align}
\label{eq:gar_upsilon_bound}
\Upsilon_n \leq \frac{1}{\sqrt{n}} \ConstC_{\ref{lem:sigma_n_labmda_min_lowerbound}}^{-1/2} \supconsteps \ConstC_{\ref{prop:Qell:bound}} \sum_{k=1}^{n-1} \PE[\norm{\xi_k}^2] = \frac{d}{\sqrt{n}} \ConstC_{\ref{lem:sigma_n_labmda_min_lowerbound}}^{-1/2} \supconsteps \ConstC_{\ref{prop:Qell:bound}}\eqsp.  
\end{align}
To proceed with the second term in \eqref{eq:shao_zhang_bound_normal_approx}, we use the representation for the statistic $D$ from \eqref{eq:linear and nonlinear terms}:
\begin{align}
D = \frac{1}{\sqrt{n}} \Sigma_n^{-1/2} \sum_{k=1}^{n-1} \Hnalpha{k}{0} + \frac{1}{\sqrt{n}} \Sigma_n^{-1/2} \sum_{k=0}^{n-1} \Gamma_{1:k} (\theta_0 - \thetas)\eqsp.
\end{align}   
\Cref{lem:sum_Gamma_1:k} implies that
\begin{align}
\frac{\ConstC_{\ref{lem:sigma_n_labmda_min_lowerbound}}^{-1/2}}{\sqrt{n}} \PE\bigl[\norm{\sum_{k=0}^{n-1} \Gamma_{1:k} (\theta_0 - \thetas)}^2\bigr] \leq \frac{\ConstC_{\ref{lem:sigma_n_labmda_min_lowerbound}}^{-1/2} \normop{\theta_0 - \thetas}  \ConstC_{\ref{th:pth_moment_bound}, 5}}{\sqrt{n}} \eqsp.
\end{align}
Using the representation $\Hnalpha{k}{0} = \Jnalpha{k}{1} + \Jnalpha{k}{2} + \Hnalpha{k}{2}$ and Minkowski's inequality, 
\begin{align}
\PE^{1/2}\bigl[\normop{\sum_{k=1}^{n-1} \Hnalpha{k}{0}}^2\bigr] \leq \PE\bigl[\normop{\sum_{k=1}^{n-1} \Jnalpha{k}{1}}^2\bigr] + \PE\bigl[\normop{\sum_{k=1}^{n-1} \Jnalpha{k}{2}}^2\bigr] + \PE\bigl[\normop{\sum_{k=1}^{n-1} \Hnalpha{k}{2}}^2\bigr]\eqsp.
\end{align}
Applying \Cref{lem:Jk_ell:p_moment} with $p=2$ and Minkowski's inequality, we get
\begin{align}
n^{-1/2} \PE^{1/2}\bigl[\normop{\sum_{k=1}^{n-1} \Jnalpha{k}{2}}^2\bigr] 
\leq n^{-1/2} \sum_{k=1}^{n-1} \ConstJ{2} 2^{5/2} \alpha_k^{3/2} 
\leq 2^{5/2} \ConstJ{2} \varphi_n\eqsp,
\end{align}
where the function $\varphi_n$ is defined in \eqref{eq:varphi_n_def}. Similarly, it holds that
\begin{align}
n^{-1/2} \PE^{1/2}\bigl[\norm{\sum_{k=1}^{n-1} \Hnalpha{k}{2}}^2\bigr] \leq n^{-1/2} \sum_{k=1}^{n-1} 2^{5/2} \ConstH{2}  \alpha_k^{3/2} \leq 2^{5/2} \ConstH{2} \varphi_n\eqsp.
\end{align}
Rewrite the sum of $\Jnalpha{k}{1}$:
\begin{align}
    n^{-1/2} \PE^{1/2}\bigl[\norm{\sum_{k=1}^{n-1} \Jnalpha{k}{1}}^2\bigr] &= n^{-1/2} \PE^{1/2}\bigl[\norm{\sum_{k=1}^{n-1} \sum_{\ell=1}^{k} \alpha_\ell G_{\ell+1:k} \zmfuncAw[\ell] \Jnalpha{\ell-1}{0}}^2\bigr] = n^{-1/2} \PE^{1/2}\bigl[\norm{\sum_{\ell=1}^{n-1} Q_\ell \zmfuncAw[\ell] \Jnalpha{\ell-1}{0}}^2\bigr] \eqsp. 
\end{align}
Since $Q_\ell \zmfuncAw[\ell] \Jnalpha{\ell-1}{0}$ is a martingale-difference sequence, \Cref{prop:Qell:bound} and Burkholder's inequality \cite[Theorem 9.1]{osekowski:2012} imply
\begin{align}
n^{-1/2} \PE^{1/2}\bigl[\norm{\sum_{\ell=1}^{n-1} Q_\ell \zmfuncAw[\ell] \Jnalpha{\ell-1}{0}}^2\bigr]
 &\leq 2n^{-1/2}\ConstC_{\ref{prop:Qell:bound}} \bigl(\sum_{\ell=1}^{n-1} \PE[\normop{\zmfuncAw[\ell] \Jnalpha{\ell-1}{0}}^2]\bigr)^{1/2} \leq 2n^{-1/2}\bConst{A} \ConstC_{\ref{prop:Qell:bound}} \bigl(\sum_{\ell=1}^{n-1} \PE[\normop{\Jnalpha{\ell-1}{0}}^2]\bigr)^{1/2}\eqsp.
\end{align}
Now we use \Cref{lem:Jk_ell:p_moment} and get:
\begin{align}
n^{-1/2} \norm{\PE^{1/2}\bigl[\sum_{k=1}^{n-1} \Jnalpha{k}{1}}^2\bigr] &\leq 2^{9/2} n^{-1/2}\bConst{A} \ConstC_{\ref{prop:Qell:bound}} \ConstJ{0} \sqrt{c_0 \sum_{\ell=1}^{n-2} \ell^{-\gamma}} \leq \frac{2^{9/2} \bConst{A} \ConstC_{\ref{prop:Qell:bound}} \ConstJ{0} \sqrt{c_0}}{\sqrt{1-\gamma}} n^{-\gamma/2}\eqsp. 
\end{align}
Combining the above bounds with $\PE[\norm{W}^2] = d$ we obtain: 
\begin{align}
\label{eq:gar_wd_bound}
\PE\bigl[\norm{W}\norm{D}\bigr] \leq  \frac{2^{9/2} d^{1/2} \bConst{A} \ConstC_{\ref{prop:Qell:bound}} \ConstJ{0} \sqrt{c_0} \supconsteps}{n^{\gamma/2}\ConstC_{\ref{lem:sigma_n_labmda_min_lowerbound}}^{1/2} \sqrt{1-\gamma}} &+ d^{1/2} 2^{5/2} \ConstC_{\ref{lem:sigma_n_labmda_min_lowerbound}}^{-1/2} \varphi_{n} (\ConstJ{2} + \ConstH{2}) + \frac{d^{1/2} \normop{\theta_0 - \thetas}}{\ConstC_{\ref{lem:sigma_n_labmda_min_lowerbound}}^{1/2}n^{1/2}}  \ConstC_{\ref{th:pth_moment_bound}, 5}  \eqsp,
\end{align}
To derive a bound for the third term in \eqref{eq:shao_zhang_bound_normal_approx}, we introduce the following notations:
    \begin{equation}
        D_1^{(i)} = n^{-1/2}\sum_{k=1}^{n-1} (\Hnalpha{k}{0} - \Hnalpha{k}{0, i})\eqsp, \quad  D_2^{(i)} = n^{-1/2}\sum_{k=1}^{n-1} (\Gamma_{1:k} - \Gamma_{1:k}^{(i)})(\theta_0 - \thetas)\eqsp.
    \end{equation}
    Hence, one can check that
    \begin{equation}
        D - D^{(i)} = \Sigma_n^{-1/2}(D_1^{(i)} + D_2^{(i)})\eqsp.
    \end{equation}
    Thus, combining \Cref{prop:Qell:bound}, \Cref{lem:gamma_copies_difference}, \Cref{lem:Hk:copy_difference} with Minkowski's inequality, we get
    \begin{align}
        \PE\bigl[\norm{\xi_i} \norm{D - D^{(i)}}\bigr] &\leq n^{-1/2} \ConstC_{\ref{lem:sigma_n_labmda_min_lowerbound}}^{-1} \supconsteps \ConstC_{\ref{prop:Qell:bound}} \left(\PE^{1/2}\bigl[\norm{D_1^{(i)}}^2\bigr] + \PE^{1/2}\bigl[\norm{D_2^{(i)}}^2\bigr]\right) \\
        &\leq \ConstC_{\ref{lem:sigma_n_labmda_min_lowerbound}}^{-1} \supconsteps  \biggl(\frac{\norm{\theta_0-\thetas}}{n} \ConstC_{\ref{lem:gamma_copies_difference}} \prod_{m=1}^{i-1} (1-\frac{a\alpha_m}{2}) + \frac{1}{n} \sqrt{\alpha_i} \ConstC_{\ref{lem:Hk:copy_difference}}\biggr) \eqsp. 
    \end{align}
    Hence, using \Cref{lem:summ_alpha_k}, we finish the proof: 
    \begin{align}
        \label{eq:gar:xi_d_bound}
         \sum_{i=1}^{n-1} \PE[\normop{\xi_i} \normop{D - D^{(i)}}] &\leq \frac{ \supconsteps \ConstC_{\ref{prop:Qell:bound}} \frac{\sqrt{c_0}}{1-\gamma} \ConstC_{\ref{lem:Hk:copy_difference}}}{\ConstC_{\ref{lem:sigma_n_labmda_min_lowerbound}} n^{\gamma/2}}  +  \frac{\norm{\theta_0-\thetas}}{n} \frac{\supconsteps\ConstC_{\ref{lem:gamma_copies_difference}}}{\ConstC_{\ref{lem:sigma_n_labmda_min_lowerbound}}}  (c_0 + \frac{2}{a(1-\gamma)}) \frac{(1+k_0)^\gamma}{c_0} \eqsp.
    \end{align}
    The proof follows from \eqref{eq:gar_upsilon_bound}, \eqref{eq:gar_wd_bound}, \eqref{eq:gar:xi_d_bound} by rearranging the terms.
\end{proof}
We now state the technical lemmas that we use in the proof of \Cref{th:shao2022_berry}.
\begin{lemma}
\label{lem:sum_Gamma_1:k}
Let $p \geq 2$. Assume \Cref{assum:iid}, \Cref{assum:noise-level}, \Cref{assum:step-size}($p \vee \log d$), \Cref{assum:sample_size}. Then, it holds that
\begin{align}
\PE^{1/p}\bigl[\norm{\sum\nolimits_{k=0}^{n-1} \Gamma_{1:k} (\theta_0 - \thetas)}^p\bigr] \leq \norm{\theta_0 - \thetas}  \ConstC_{\ref{th:pth_moment_bound}, 5} \eqsp,
\end{align}
where the constant $ \ConstC_{\ref{th:pth_moment_bound}, 5}$ is given in \eqref{def:c_i_thrm1}.
\end{lemma}

\begin{lemma}
    \label{lem:sigma_n_labmda_min_lowerbound}
    Let $p \geq 2$. Assume \Cref{assum:iid}, \Cref{assum:noise-level}, \Cref{assum:step-size}($p \vee \log d$), \Cref{assum:sample_size}. Then it holds that 
    \begin{equation}
        \label{eq:lambda_min_sigma_n_boot}
        \lambda_{\min}(\Sigma_{n}) \geq \ConstC_{\ref{lem:sigma_n_labmda_min_lowerbound}} \eqsp, \text{ where } \ConstC_{\ref{lem:sigma_n_labmda_min_lowerbound}} = \frac{\lambda_{\min}(\Sigma_{\infty})}{2}\eqsp.
    \end{equation}
\end{lemma}

Now we introduce the vector $(Z_{1}^{\prime},\ldots,Z_{n-1}^{\prime})$ an independent copy of $(Z_1, \ldots, Z_{n-1})$, and introduce the following notation for $\ell \leq m$:
\begin{align}
    &\Gamma^{(i)}_{\ell:m} = \begin{cases}
        \Gamma_{\ell:m}, &\text{if } i \not\in [\ell, m], \\
        \Gamma_{\ell:m}(Z_\ell, \ldots, Z_{i-1}, Z_{i}', Z_{i+1}, \ldots, Z_m), &\text{if } i \in [\ell, m]\eqsp;
    \end{cases} \\
    &\Jnalpha{k}{\ell, i} = \begin{cases}
        \Jnalpha{k}{\ell}, &\text{if } k <i, \\
        \Jnalpha{k}{\ell}(Z_1, \ldots, Z_{i-1}, Z_{i}', Z_{i+1}, \ldots, Z_k), &\text{if } k \geq i\eqsp;
    \end{cases} \\
    &\Hnalpha{k}{\ell, i} = \begin{cases}
        \Hnalpha{k}{\ell}, &\text{if } k <i, \\
        \Hnalpha{k}{\ell}(Z_1, \ldots, Z_{i-1}, Z_{i}', Z_{i+1}, \ldots, Z_k), &\text{if } k \geq i\eqsp;
    \end{cases} \\
    &D^{(i)} = n^{-1/2} \Sigma_n^{-1/2} \sum_{k=1}^{n-1} \Hnalpha{k}{0, i} + n^{-1/2} \Sigma_n^{-1/2} \sum_{k=1}^{n-1} \Gamma_{1:k}^{(i)} (\theta_0 - \thetas)\eqsp.
\end{align}
Here $D^{(i)}$, $1 \leq i \leq n-1$, is a counterpart of $D$ with $Z_i$ substituted with $Z'_i$. In order to control $\PE[\norm{D - D^{(i)}}^2]$, we use the following two auxiliary lemmas. 
\begin{lemma}    
\label{lem:gamma_copies_difference}
    Assume \Cref{assum:iid}, \Cref{assum:noise-level}, \Cref{assum:step-size}($2 \vee \log d$) and \Cref{assum:sample_size}. Then it holds that
    \begin{align}
        \PE^{1/2}\bigl[\normop{\frac{1}{\sqrt{n}} \sum_{k=1}^{n-1} (\Gamma_{1:k}-\Gamma_{1:k}^{(i)}) (\theta_0 - \thetas)}^2\bigr] \leq \frac{ \normop{\theta_0 -  \thetas}}{\sqrt{n}} \ConstC_{\ref{lem:gamma_copies_difference}} \prod_{m=1}^{i-1} (1-\frac{a\alpha_m}{2})\eqsp, 
    \end{align}
where $\ConstC_{\ref{lem:gamma_copies_difference}} = \bConst{A} \qcond \rme^2 (c_0 + 2/((1-\gamma)a))$. 
\end{lemma}

\begin{lemma}
\label{lem:Hk:copy_difference}
Assume \Cref{assum:iid}, \Cref{assum:noise-level}, \Cref{assum:step-size}($2 \vee \log d$) and \Cref{assum:sample_size}. Then it holds that
\begin{align}
n^{-1/2}\PE^{1/2}\bigl[\norm{\sum_{k=1}^{n-1} (\Hnalpha{k}{0}-\Hnalpha{k}{0, i})}^2\bigr] \leq \frac{\ConstC_{\ref{lem:Hk:copy_difference}}}{\sqrt{n}} \sqrt{\alpha_i} \eqsp, 
\end{align}
where $\ConstC_{\ref{lem:Hk:copy_difference}}$ is given by
\begin{align}
    \label{def:RH_const_def}
    \ConstC_{\ref{lem:Hk:copy_difference}} = \qcond \rme^2 \supconsteps \bConst{A}  \biggl(c_0 + \frac{2}{a(1-\gamma)}\biggr)^{3/2} + 2\bConst{A} \sqrt{\qcond} \rme  (\ConstJ{0} + \ConstH{0}) \biggl(c_0 + \frac{2}{a(1-\gamma)}\biggr) \eqsp.
\end{align}
\end{lemma}

\paragraph{Proof of \Cref{th:shao2022_berry}}
\label{appendix:proof_of_shao2022_berry}
We first introduce the constant 
\begin{equation}
\label{eq:const_C_10_def}
\ConstC_{\ref{th:shao2022_berry}} = \frac{3\sqrt{d}\ConstC_{\ref{lem:sigma_n_bound}}}{2\lambda_{\min}(\Sigma_\infty)}\eqsp.
\end{equation}
Applying the triangle inequality, 
\begin{equation}
\label{eq:conex_distance_triangle_inequality}
\kolmogorov(\sqrt{n}(\bar{\theta}_{n} - \thetas), \Sigma_{\infty}^{1/2} \eta) \leq \kolmogorov(\sqrt{n}(\bar{\theta}_{n} - \thetas), \Sigma_{n}^{1/2} \zeta) + \kolmogorov(\Sigma_{n}^{1/2} \zeta,\Sigma_{\infty}^{1/2} \eta)\eqsp,
\end{equation}
where $\eta,\zeta \sim \mathcal{N}(0,\Id_d)$. Then the first term in r.h.s. is controlled with \Cref{theo:GAR}, and it remains to upper bound $\kolmogorov(\Sigma_{\infty}^{1/2} \eta, \Sigma_{n}^{1/2} \zeta)$. Towards this aim, we apply the Gaussian comparison inequality of \cite[Theorem 1.1]{Devroye2018}, see also \cite{BarUly86}. Assumption \Cref{assum:sample_size} and \Cref{lem:sigma_n_bound} imply that
\begin{align}
\normop{\Sigma_{\infty}^{-1/2}\Sigma_n\Sigma_{\infty}^{-1/2} - \Id_d} \leq \normop{\Sigma_{\infty}^{-1}} \normop{\Sigma_n - \Sigma_\infty} \leq \frac{\ConstC_{\ref{lem:sigma_n_bound}}}{\lambda_{\min}(\Sigma_\infty) n^{1-\gamma}} \leq \frac{1}{2} \eqsp.
\end{align}
On the other hand, the following bound holds:
\begin{align}
\trace{(\Sigma_{\infty}^{-1/2}\Sigma_n\Sigma_{\infty}^{-1/2} - \Id_d)^2} \leq \frac{\ConstC_{\ref{lem:sigma_n_bound}}^2 d}{\lambda^2_{\min}(\Sigma_\infty) n^{2(1-\gamma)}} \eqsp.
\end{align}
Hence, applying \cite[Theorem 1.1]{Devroye2018}, we get 
\begin{align}
\kolmogorov(\Sigma_{n}^{1/2} \zeta,\Sigma_{\infty}^{1/2} \eta) \leq \frac{3\ConstC_{\ref{lem:sigma_n_bound}}}{2\lambda_{\min}(\Sigma_\infty)} \frac{\sqrt{d}}{n^{1-\gamma}} \eqsp,
\end{align}
and it remains to substitute this bound into \eqref{eq:conex_distance_triangle_inequality}.

\subsection{Proofs of \Cref{sec:bootstrap}}
\label{sec:proof:boot_validity}
\label{sec:bootstrap_validity}
\subsubsection{Preliminary steps for Gaussian approximation under $\PPb$}
\label{subsec:linear_part_def}
We first identify the linear ($W^{\boot}$) and non-linear ($D^{\boot}$) parts of the error decomposition \eqref{eq:linear_part_proof}. We start from the decomposition 
\begin{equation}
\label{eq:one_step_expand}
\theta_{k}^\boot - \theta_{k} = (\Id - \alpha_{k} w_k \funcAw_k)(\theta_{k-1}^\boot - \theta_{k-1})  - \alpha_{k} (w_k-1) \tilde{\funnoisew}_{k},
\end{equation}
where we have set
\begin{equation}
\label{eq:definition-tilde-funnoisew}
\tilde \funnoisew_{k} =  \funcAw_{k} (\theta_{k-1} - \thetas) + \funnoisew_{k}\eqsp.
\end{equation}
To simplify the notation, we omit the bootstrap replication index, which is implicit in the sequel.
Expanding the recurrence above till $k = 0$, and using the fact that $\theta_{0}^\boot = \theta_{0}$, we obtain from \eqref{eq:one_step_expand} that 
\[
\theta_{k}^\boot - \theta_{k} = -\sum_{\ell=1}^k \alpha_\ell (w_\ell - 1)\ProdB^{\boot}_{\ell+1:k} \tilde \funnoisew_{\ell}\eqsp.
\]
where we have defined, similarly to \eqref{eq:prod_rand_matr}, the product of random matrices
\begin{equation}
\label{eq:prod_rand_matr-bootstrap}
\ProdB^{\boot}_{m:k} = \prod_{\ell = m}^k (\Id - \alpha_\ell w_\ell \funcAw_\ell)\eqsp,\,  m \leq k\eqsp,\ \text{ and } \ProdB^{\boot}_{m:k} = \Id\eqsp, \quad m > k\eqsp.
\end{equation}
Proceeding as in \eqref{eq:jn0_main}, we consider the decomposition  $\theta_{k}^\boot - \theta_{k}= \Jnalpha{k}{\boot,0} + \Hnalpha{k}{\boot,0} $, where we have set
\begin{align}
\label{eq:jn0_bootstrap}
&\Jnalpha{k}{\boot,0} =\left(\Id - \alpha_{k} \funcAw_k \right) \Jnalpha{k-1}{\boot,0} - \alpha_{k}(w_k - 1)\tilde{\funnoisew_{k}}\eqsp, && \Jnalpha{0}{\boot,0}=0\eqsp, \\[.1cm]
\label{eq:hn0_bootstrap}
&\Hnalpha{k}{\boot,0} =\left( \Id - \alpha_{k} 
 w_{k} \funcAw_k \right) \Hnalpha{k-1}{\boot,0} - \alpha_{k} (w_k - 1) \funcAw_k \Jnalpha{k-1}{\boot,0}\eqsp, && \Hnalpha{0}{\boot,0}=0\eqsp.
\end{align}
The idea of the decomposition \eqref{eq:jn0_bootstrap}-\eqref{eq:hn0_bootstrap} is similar to the one outlined before in \eqref{eq:jn0_main}-\eqref{eq:hn0_main}, since the statistic $\Jnalpha{k}{\boot,0}$ is linear when considered under $\PP^{\boot}$ (that is, when we consider only the randomness due to the bootstrap weights $(w_{k})$). With the decomposition \eqref{eq:jn0_bootstrap}-\eqref{eq:hn0_bootstrap}, we get by averaging the iterates
\begin{align}
\label{eq:bootstrap_world_shao_decomposition}
\sqrt{n}(\btheta^\boot_n - \btheta_{n}) = \frac{1}{\sqrt{n}} \sum_{k=1}^{n-1} \Jnalpha{k}{\boot,0} +  \frac{1}{\sqrt{n}} \sum_{k=1}^{n-1} \Hnalpha{k}{\boot, 0} \eqsp.
\end{align}
Unfortunately, the representation \eqref{eq:bootstrap_world_shao_decomposition} does not exactly match the one for $\sqrt{n}(\bar{\theta}_{n} - \thetas)$ outlined in \eqref{eq:linear and nonlinear terms}. Indeed, the latter one shows that $\sqrt{n}(\bar{\theta}_{n} - \theta^\star) = \Sigma_n^{1/2}W + \Sigma_n^{1/2} D$, and $\PE[\Sigma_n^{1/2}W W^{\top}\{\Sigma_n^{1/2}\}^{\top}] = \Sigma_n$. At the same time, simple calculations show that 
\[
\PE\bigl[\var^\boot \bigl[\frac{1}{\sqrt{n}} \sum_{k=1}^{n-1} \Jnalpha{k, 0}{\boot, 0}\bigr]\bigr] \neq \Sigma_n\eqsp.
\]
This issue is due to additional term $\funcAw_{\ell} (\theta_{\ell-1} - \thetas)$ arising in the definition of $\tilde{\funnoisew_{k}}$ in \eqref{eq:definition-tilde-funnoisew}. In order to overcome this problem, we further represent $\Jnalpha{k}{\boot,0} = \sum_{i=0}^{2}\Jnalpha{k, i}{\boot,0}$, where
\begin{equation}
\label{eq:decomposition-J-k-0-boot}
\begin{split}
\Jnalpha{k, 0}{\boot, 0} &= -\sum_{\ell=1}^k \alpha_\ell (w_\ell-1) G_{\ell+1:k} \eps_\ell ,  \quad 
\Jnalpha{k, 1}{\boot, 0} = -\sum_{\ell=1}^k \alpha_\ell (w_\ell-1) (\Gamma_{\ell+1:k} - G_{\ell+1:k}) \eps_\ell, \\
\Jnalpha{k, 2}{\boot, 0} &= -\sum_{\ell=1}^k \alpha_\ell (w_\ell-1) \Gamma_{\ell+1:k} \funcAw_\ell (\theta_{\ell-1}-\thetas),
\end{split}
\end{equation} 
It is easily seen that $\sum_{k=1}^{n-1} \Jnalpha{k,0}{\boot,0}= \sum_{\ell=1}^{n-1}  (w_\ell-1) Q_\ell \eps_\ell$, where $(Q_\ell)$ is defined in \eqref{eq:sigma_n_definition}, moreover, 
\[
\Sigma_n^\boot := \var^\boot \bigl[\frac{1}{\sqrt{n}} \sum_{k=1}^{n-1} \Jnalpha{k, 0}{\boot, 0}\bigr]= \frac{1}{n} \sum_{\ell=1}^{n-1} Q_\ell \eps_{\ell} \eps_{\ell}^\top Q_\ell^\top \eqsp, \text{ and } \PE[\Sigma_n^\boot] = \Sigma_n\eqsp.
\]
Later we show that $\Jnalpha{k, i}{\boot, 0}$, $i=1,2$  are negligible relative to the leading term $\Jnalpha{k, 0}{\boot, 0}$. Now we rely on the "bootstrap-world" decomposition 
\begin{equation}
\label{eq:shao_bootstrap_world_proof}
\sqrt{n}(\btheta^\boot_n - \btheta_{n}) = (\Sigma_n^\boot)^{1/2} W^\boot + (\Sigma_n^\boot)^{1/2} D^\boot\eqsp,
\end{equation}
where we have set 
\begin{align}
\label{eq:definition-W-boot}
W^\boot &= n^{-1/2} (\Sigma_n^\boot)^{-1/2} \sum_{k=1}^{n-1} \Jnalpha{k,0}{\boot, 0} =: \sum_{k=1}^{n-1} \xi_k^\boot , \quad \text{where 
$\xi_k^\boot=  n^{-1/2} (\Sigma_n^\boot)^{-1/2} (w_k-1) Q_k \eps_k$}, \\
\label{eq:definition-D-boot}
D^\boot &=  (\Sigma_n^\boot)^{-1/2} \biggl(\frac{1}{\sqrt{n}} \sum_{k=1}^{n-1} \Jnalpha{k, 1}{\boot, 0} + \frac{1}{\sqrt{n}} \sum_{k=1}^{n-1} \Jnalpha{k, 2}{\boot, 0}+  \frac{1}{\sqrt{n}} \sum_{k=1}^{n-1} \Hnalpha{k}{\boot, 0}\biggr) \eqsp.
\end{align}
In this decomposition, $W^{\boot}$ is the linear part whereas $D^{\boot}$ is the nonlinear part. 
With these notations and preliminary results, we are in a position to provide the proof of \Cref{th:bootstrap_validity}. Some proofs of technical lemmas are postponed to the appendix.

The following decomposition allows us to formalize the structure outlined in the sketch of proof given in \Cref{fig:high-level-sketch-of-proof}:
\begin{equation}
\label{eq:decomposition}
\sup_{B \in \Conv(\rset^{d})} |\PPb(\sqrt n (\bar{\theta}_{n}^\boot - \bar{\theta}_n) \in B ) - \PP(\sqrt n (\bar{\theta}_n - \thetas) \in B)| \leq T_1+ T_2+T_3\eqsp,
\end{equation}
where
\begin{equation}
\label{eq:T_1_till_T_3_def}
\begin{split}
T_1 &\eqdef \sup_{B \in \Conv(\rset^{d})} \big| \P\bigl(\sqrt{n}(\bar{\theta}_{n} - \thetas) \in B\bigr) - \P(\Sigma_n^{1/2} \eta \in B) \big|\eqsp, \\
T_2 &\eqdef \sup_{B \in \Conv(\rset^{d})} \big| \PP(\Sigma_n^{1/2} \eta \in B) - \PPb((\Sigma_n^{\boot} )^{1/2} \eta \in B) \big| \eqsp, \\
T_3 &\eqdef \sup_{B \in \Conv(\rset^{d})} \big| \PPb(\sqrt n (\bar{\theta}_{n}^\boot - \bar{\theta}_n) \in B) - \PPb((\Sigma_n^{\boot} )^{1/2} \eta \in B)\big| \eqsp, 
\end{split}
\end{equation}
and $\eta \sim \mathcal{N}(0, \Id)$ under $\PP$ and $\PP^{\boot}$. 
Our next objective is to obtain bounds on these three terms.  For the term $T_1$, it suffices to apply \Cref{theo:GAR}. Consider now $T_2$. In this case, we are comparing two centered Gaussian distributions that differ in their covariance matrices. We begin by applying Pinsker’s inequality to bound the total variation distance, which itself serves as an upper bound for the convex distance, using \cite[Theorem~1.1]{Devroye2018}:
\begin{equation}
\label{eq: Pinsker}
\| \mathcal N(0, \Sigma_1) - \mathcal N(0,\Sigma_2)\|_{\mathsf{TV}} \le \frac{3 \sqrt{d}}{2} \normop{\Sigma_1^{-1/2} \Sigma_2 \Sigma_1^{-1/2} - \Id} \,.
\end{equation}
Applying the inequality above yields  
\[
T_2 \le \frac{3 \sqrt{d}}{2 \, \lambda_{\min}(\Sigma_n)} \, \normop{\Sigma_n^\boot - \Sigma_n}.
\]  
Bounding $T_2$ therefore boils down to obtain a high-probability bound for $\normop{\Sigma_n^\boot - \Sigma_n}$. Such bound 
follow from the matrix Bernstein inequalities, developed in \cite{tropp2015introduction}. Detailed argument is given below.

The main technical challenge arises in controlling the term $T_3$, which requires decomposing the quantity $\Hnalpha{k}{\boot,0}$ in a manner analogous to the decomposition in \eqref{eq:error_decomposition_LSA}. It is worth noting, however, that once again, the quantities introduced and the method used to derive the bounds are markedly different than the ones used in \Cref{sec:proof:proofs_moments}. Along the lines of \eqref{eq:jn_allexpansion_main}, we 
expand $\Hnalpha{k}{\boot,0}$ as follows:
\begin{equation}
\label{eq:decomposition_error_bewtween_b_and_r_worlds}
\theta_{k}^\boot - \theta_{k} = J_k^{\boot, 0} + \sum_{j=1}^L J_k^{\boot, j}  + H_k^{\boot, L}, 
\end{equation}
where 
\begin{equation}
\label{eq:higher-order-expansion-H-boot}
\begin{split}
   \Jnalpha{k}{\boot, 0} &= - \sum_{\ell=1}^k \alpha_\ell (w_\ell - 1)\ProdB_{\ell+1:k} \tilde \funnoisew_{\ell}, \quad \Jnalpha{k}{\boot, j} = - \sum_{\ell=1}^k \alpha_\ell (w_\ell - 1)\ProdB_{\ell+1:k} A_\ell \Jnalpha{\ell-1}{\boot, j-1}, \quad j \in [1, L] \\
   \label{eq:H_k_boot_def}
   \Hnalpha{k}{\boot, L} & = -\sum_{\ell=n+1}^k \alpha_\ell (w_\ell - 1)\ProdB^{\boot}_{\ell+1:k} A_\ell \Jnalpha{\ell-1}{\boot, L}\eqsp,
\end{split}
\end{equation}
Similar to \Cref{sec:proof:proofs_moments}, we will establish upper bounds on the $p$-th moments under $\PPb$ of $\Jnalpha{k}{\boot,j}$, $\Hnalpha{k}{\boot,j}$, $j \in [0;L]$ and $\tilde \funnoisew_\ell$. 
However, the proofs differ significantly from the previous case, which relied heavily on the exponential stability of products of random matrices $\ProdB_{m:k}$ (see \Cref{lem:matr_product_as_bound}).
The proof strategy consists of two steps. First, we define certain events in the 'original world' under which the various quantities of interest can be controlled. Second, we show that these events occur with high probability—specifically, of order $1- \iota/n$, for an appropriately chosen $\iota > 0$. We define first:
\begin{equation}
\label{eq:definition-Omega-1}
\Omega_{1} = \bigcap_{k=1}^{n-1} \bigl\{\norm{\theta_k - \thetas} \leq g(k,\norm{\theta_0-\thetas},n)\eqsp
\bigr\} \eqsp,
\end{equation}
where we have set 
\[
g(k,t,n) = \sqrt{\qcond} \rme^2 t \prod_{\ell=1}^k (1 - \frac{a}{2} \alpha_\ell) + 2\rme \log(5n) \ConstC_{\ref{lem:last_moment_bound}} \sqrt{\alpha_{k}}\eqsp.
\]
Applying \Cref{lem:last_moment_bound} with for $2 \leq p \leq \log(5n^2)$ and  then  \Cref{lem:markov_inequality}, we get that for every fixed $k \in [1;n-1]$, 
\[
\PP \left(\norm{\theta_k - \thetas} \geq g(k,\norm{\theta_0-\thetas},n) \right) \leq \frac{1}{5n^2}\eqsp.
\]
By the union bound, we obtain that $\PP(\Omega_1) \geq 1 - 1/(5n)$. We may show that
\begin{lemma}
    \label{lem:tilde_eps_boot_bound}
    Under the assumptions of \Cref{th:bootstrap_validity}, on the event $\Omega_1$, it holds for any $\ell \geq 1$ that
    \begin{align}
        \norm{\tilde\eps_\ell} \leq \ConstC_{\ref{lem:tilde_eps_boot_bound}}, \text{ where } \ConstC_{\ref{lem:tilde_eps_boot_bound}} = \supconsteps + 2 \rme \bConst{A} \ConstC_{\ref{lem:last_moment_bound}} + \sqrt{\qcond} \rme^3 \bConst{A} \norm{\theta_0 - \thetas} \eqsp.
    \end{align}
\end{lemma}
\noindent Similarly, we introduce the following  event  
\begin{equation}
\label{eq:definition-Omega-2}
\Omega_{2} = \bigcap_{1 \le m \le k \le n} \bigl\{\|\ProdB_{m:k}\| \le \ConstC_{\ref{lem:matr_product_as_bound}} \prod_{j = m}^{k} \bigl(1 - \frac{a \alpha_{j}}{2}\bigr) \bigr\}\eqsp.
\end{equation}
Using the exponential stability of $\ProdB_{m:k}$  (see \Cref{lem:matr_product_as_bound}) 
with $p=\log(5 n^3)$  and \Cref{lem:markov_inequality}, we get that with probability at least $1-1/(5n^3)$,
\begin{equation}
\label{eq:matrix_hpd_bound}
\normop{\ProdB_{m:k}} 
\leq \ConstC_{\ref{lem:matr_product_as_bound}} \prod_{\ell=m}^{k}\bigl(1 - \frac{a \alpha_{\ell}}{2}\bigr) \leq  \ConstC_{\ref{lem:matr_product_as_bound}} \exp\bigl\{-\frac{a}{2} \sum_{\ell=m}^{k}\alpha_\ell\bigr\} \quad \text{where } \ConstC_{\ref{lem:matr_product_as_bound}} = \sqrt{\qcond} \rme^2\eqsp.
\end{equation}
By the union bound, we get $\PP(\Omega_2) \geq 1-1/(5n)$. It is also required to consider
\begin{equation}
\label{eq:definition-Omega-3}
\Omega_{3} = \bigcap_{\ell=1}^n \bigl\{\norm{\alpha_\ell \sum_{k=\ell}^{n-1} (\Gamma_{\ell+1:k} - G_{\ell+1:k}) \eps_\ell} \leq  \ConstC_{\ref{lem:gamma_deviation_bound}}\sqrt{\alpha_\ell} \log(5n) \bigr\}\eqsp, 
\end{equation} 
Here again, we may show that $\PP(\Omega_3) \geq 1 - 1/(5n)$. A detailed proof is given in \Cref{lem:proof-Omega-Matrix-Bernstein-0}. On the event $\bigcap_{i=1}^3 \Omega_i$, we derive below bounds for $J_{k,i}^{\boot,0}$, $i = 1, 2$, defined in \eqref{eq:decomposition-J-k-0-boot}.
\begin{lemma}
\label{lemma:EbJk1:bound}
Under the assumptions of \Cref{th:bootstrap_validity}, on the event $\Omega_3$, it holds that 
\begin{align}
\bigl\{\mathbb{E}^\boot \bigl[\norm{n^{-1/2} \sum\nolimits_{k=1}^{n-1} J_{k,1}^{\boot, 0}}^2 \bigr] \bigr\}^{1/2} \leq \ConstC_{\ref{lemma:EbJk1:bound}} \frac{\log{(5n)}}{n^{\gamma/2}}\eqsp, \text{ where } \ConstC_{\ref{lemma:EbJk1:bound}} = \frac{\supconsteps \sqrt{c_0} \ConstC_{\ref{lem:gamma_deviation_bound}}}{\sqrt{1-\gamma}}\eqsp.
\end{align}
\end{lemma}

\begin{lemma}
\label{lemma:EbJk2:bound}
Under the assumptions of \Cref{th:bootstrap_validity}, on the event $\Omega_1 \cap \Omega_2$, it holds
    \begin{align}
        \bigl\{\mathbb{E}^\boot \bigl[\norm{n^{-1/2} \sum\nolimits_{k=1}^{n-1} J_{k,2}^{\boot, 0}}^2\bigr] \bigr\}^{1/2} \leq    \frac{\ConstC_{\ref{lemma:EbJk2:bound},1} \log(5n)}{n^{\gamma/2}} +  \frac{\ConstC_{\ref{lemma:EbJk2:bound},2} (1+k_0)^{\gamma/2}\norm{\theta_0-\thetas}}{\sqrt{n}}\eqsp,
    \end{align}
    where we have defined
    \begin{align}
\ConstC_{\ref{lemma:EbJk2:bound},1} =  \sqrt{2} \bConst{A} \ConstC_{\ref{lem:matr_product_as_bound}} \rme \bigl(c_0 + \frac{2}{a(1-\gamma)}\bigr) \ConstC_{\ref{lem:last_moment_bound}} \frac{\sqrt{c_0}}{\sqrt{1-\gamma}}\eqsp, \quad \ConstC_{\ref{lemma:EbJk2:bound},2} = 2\sqrt{2} c_0^{-1/2} \bConst{A} \ConstC_{\ref{lem:matr_product_as_bound}} \qcond^{1/2} \rme^2 \bigl(c_0 + \frac{2}{a(1-\gamma)}\bigr)^{3/2} \eqsp.
    \end{align}
\end{lemma}
\noindent
We introduce an additional event, which is essential for establishing exponential stability of the \emph{bootstrap world} random matrix product $\ProdB^{\boot}_{m:k}$ defined in \eqref{eq:prod_rand_matr-bootstrap}.
\begin{equation}
\label{eq:definition-Omega-4}
\Omega_{4} = \bigcap_{h=1}^n \bigcap_{m=0}^{n-h-1} \bigl\{ \norm{\sum_{\ell = m+1}^{m+h} \alpha_\ell (\Am_\ell - \bA)} \leq  2 \bConst{A} \bigl\{\sum_{\ell=m+1}^{m+h}\alpha_\ell^{2}\bigr\}^{1/2} \log\left(10n^3 d\right)\bigr\}\eqsp, 
\end{equation}
Here again, we may show that $\PP(\Omega_4) \geq 1 - 1/(5n)$. The proof is a straightforward application of matrix Bernstein inequality; details are given in \Cref{lem:proof-Omega-Matrix-Bernstein}.

Under the event $\Omega_1 \cap \Omega_2 \cap \Omega_4$, we can provide bounds to the terms appearing in the expansion of $\Hnalpha{j}{\boot,0}$, given in \eqref{eq:higher-order-expansion-H-boot}. 
\begin{lemma}
\label{lem:expansion}
Under the assumptions of \Cref{th:bootstrap_validity}, on the event $\Omega_1 \cap \Omega_2 \cap \Omega_4$, for $j, L \in \{0, 1, 2\}$ it holds that
\begin{align}
\label{eq:J_k_boot_bound}
\{\PEb[\|J_k^{\boot, j}\|^2 ]\}^{1/2} \leq \ConstJb{j}{1} \alpha_{k}^{(j+1)/2} \eqsp, \quad \{\PEb[\|H_k^{\boot, L}\|^2 ]\}^{1/2} \leq \ConstHb{L}{1} \alpha_{k}^{(L+1)/2} \eqsp.
\end{align}
where $\ConstJb{j}{1}$, $\ConstHb{L}{1}$ satisfy the recurrence 
\begin{align}
    \label{def:JH_boot_constants}
    &\ConstJb{0}{1} = \frac{2\sqrt{3} \ConstC_{\ref{lem:tilde_eps_boot_bound}} \ConstC_{\ref{lem:matr_product_as_bound}}}{\sqrt{a}} \eqsp, \eqsp \ConstJb{j}{1} = \frac{2\sqrt{3}}{\sqrt{a}} \ConstJb{j-1}{1} \bConst{A} \ConstC_{\ref{lem:matr_product_as_bound}} \eqsp, \eqsp \ConstHb{L}{1} = \frac{4 \sqrt{3}}{\sqrt{a}}\ConstJb{L}{1} \bConst{A} \ConstC_{\ref{prop:product_random_matrix_bootstrap}}  \eqsp.
\end{align}
\end{lemma}

We may now proceed to the proof of the theorem. 
\begin{proof}[Proof of \Cref{th:bootstrap_validity}]
We start with $T_2$. 
We first show that the bootstrap word covariance $\Sigma_n^\boot$  approximates $\Sigma_n$. More precisely, set $\ConstC_{\ref{lem:matrix_bernstein}} = 2 \supconsteps^2 \ConstC_{\ref{prop:Qell:bound}}^2$ and consider the event
\begin{equation}
\label{eq:definition-Omega-5}
\Omega_{5} = \bigl\{ \normop{\Sigma_n^\boot - \Sigma_n} \leq \frac{\sqrt{2} \ConstC_{\ref{lem:matrix_bernstein}} \sqrt{\log{(10 d n)}}}{\sqrt{n}} + \frac{\ConstC_{\ref{lem:matrix_bernstein}} \log{(10 d n)}}{3n} \bigr\}\eqsp, 
\end{equation}
It is shown in \Cref{lem:matrix_bernstein} that $\PP(\Omega_5) \geq 1-1/5n$ and 
in \Cref{lem:sigma_n_boot_labmda_min_lowerbound}, that $\lambda_{\min}(\Sigma_n) \geq \lambda_{\min}(\Sigma_\infty)/2$.
Combining these two results, we then get that on the event $\Omega_5$, 
\begin{equation}
\label{eq:bound-T-2}
T_2 \leq \frac{3 \sqrt d}{2 \lambda_{\min}(\Sigma_n)} \left(\frac{\sqrt{2} \ConstC_{\ref{lem:matrix_bernstein}} \sqrt{\log{(10 d n)}}}{\sqrt{n}} + \frac{\ConstC_{\ref{lem:matrix_bernstein}} \log{(10 d n)}}{3n} \right)
\end{equation}
Finally, we consider $T_3$. As emphasized in the preliminaries of the proof, we again invoke the approach of \cite{shao2022berry}, where $W^b$ (defined in \eqref{eq:definition-W-boot}) plays the role of the linear term and $D^b$ (defined in \eqref{eq:definition-D-boot}) that of the nonlinear remainder. 
\cite[Theorem 2.1]{shao2022berry} shows that
\begin{align}
\label{eq:shao_zhang_bound_appendix}
T_3 \leq 259 d^{1/2} \Upsilon^\boot + 2 \PEb[\|W^\boot\| \|D^\boot\|] + 2 \sum_{\ell=1}^n \PEb[\|\xi_\ell^\boot\| \|D^\boot - D^{\boot, \ell}\|],
\end{align}
where $\Upsilon^\boot = \sum_{i=1}^{n-1} \PE[\norm{\xi_i^\boot}^3]$, and $\xi_\ell^\boot$ is defined in \eqref{eq:definition-W-boot}. It follows \Cref{prop:Qell:bound} and \Cref{lem:sigma_n_boot_labmda_min_lowerbound} that
\[
\norm{\xi_\ell^\boot} \leq n^{-1/2} \{\lambda_{\min}(\Sigma_n^\boot)\}^{-1/2} |w_\ell -1 | \normop{Q_\ell} \supconsteps \leq |w_\ell-1| \supconsteps \ConstC_{\ref{prop:Qell:bound}}/(\sqrt{\ConstC_{\ref{lem:sigma_n_boot_labmda_min_lowerbound}}} \sqrt{n}).
\]
Since by construction $\PE^\boot [\norm{W^\boot}^2] = \sum_{\ell=1}^n \PEb[\|\xi_\ell^\boot\|^2]= d$ and $\PE^\boot[|w_\ell - 1|^3] = m_3$ for all $\ell$, we get:
\begin{align}
    \label{eq:boot_upsilon_bound}
    \Upsilon^\boot \leq \frac{m_3 \supconsteps \ConstC_{\ref{prop:Qell:bound}}}{\sqrt{n}\sqrt{\ConstC_{\ref{lem:sigma_n_boot_labmda_min_lowerbound}}}} \sum_{i=1}^{n-1} \PE^\boot[\norm{\xi_\ell^\boot}^2]  \leq \frac{d}{\sqrt{n}} \frac{m_3\supconsteps \ConstC_{\ref{prop:Qell:bound}}}{\sqrt{\ConstC_{\ref{lem:sigma_n_boot_labmda_min_lowerbound}}}} \eqsp.
\end{align}    
To proceed with the second term in \eqref{eq:shao_zhang_bound_appendix}, first note that
    \begin{align}
        \PE^\boot \bigl[\norm{W^\boot} \norm{D^\boot}\bigr] \leq  \{\PE^\boot \bigl[\norm{W^\boot}^2\bigr]\}^{1/2} \{\PE^\boot \bigl[\norm{D^\boot}^2\bigr]\}^{1/2} = d^{1/2} \{\PE^\boot \bigl[\norm{D^\boot}^2\bigr]\}^{1/2} \eqsp .
    \end{align}
    We use the expression of $D^\boot$ given in \eqref{eq:definition-D-boot} and further expand $\Hnalpha{k}{\boot,0}= \sum_{j=1}^2 \Jnalpha{k}{\boot,j} + \Hnalpha{k}{\boot,2}$, using \eqref{eq:decomposition_error_bewtween_b_and_r_worlds} with $L=2$. \Cref{lemma:EbJk1:bound,lemma:EbJk2:bound} provide bounds for $\bigl\{\mathbb{E}^\boot \bigl[\norm{n^{-1/2} \sum\nolimits_{k=1}^{n-1} J_{k,j}^{\boot, 0}}^2 \bigr] \bigr\}^{1/2}$, $j=1,2$. \Cref{lem:jkb1:bound} give the bound for $\bigl\{\PE^\boot \bigl[\norm{\frac{1}{\sqrt{n}} \sum_{k=1}^{n-1} J_{k}^{\boot, 1}}^2\bigr]\bigr\}^{1/2}$. Finally, \Cref{lem:expansion} show that
    \begin{align}
        \label{eq:sum_j_h_k_b2_bound}
       \bigl\{\PE^\boot \bigl[\norm{\frac{1}{\sqrt{n}} \sum_{k=1}^{n-1} J_{k}^{\boot, 2}}^2\bigr]\bigr\}^{1/2} &\leq \frac{1}{\sqrt{n}} \sum_{k=1}^{n-1} \ConstJb{2}{1} \alpha_{k}^{3/2} \leq \ConstJb{2}{1} \varphi_n \eqsp, \\
       \bigl\{\PE^\boot \bigl[\norm{\frac{1}{\sqrt{n}} \sum_{k=1}^{n-1} H_{k}^{\boot, 2}}^2\bigr]\bigr\}^{1/2} &\leq \frac{1}{\sqrt{n}} \sum_{k=1}^{n-1}  \ConstHb{2}{1} \alpha_k^{3/2} \leq \ConstHb{2}{1} \varphi_n \eqsp.
    \end{align}
    By combining the inequalities above, we  obtain
    \begin{multline}
        \label{eq:boot_wd_bound}
        \PE^\boot \bigl[\norm{W^\boot} \norm{D^\boot}\bigr] \leq \frac{d^{1/2} \log(5n)}{ n^{\gamma/2}\ConstC_{\ref{lem:sigma_n_boot_labmda_min_lowerbound}}^{1/2}} (\ConstC_{\ref{lemma:EbJk1:bound}}  + \ConstC_{\ref{lemma:EbJk2:bound},1} + \ConstC_{\ref{lem:jkb1:bound}})  \\+ \frac{d^{1/2} (1+k_0)^{\gamma/2} \norm{\theta_0 - \thetas} \ConstC_{\ref{lemma:EbJk2:bound},2}}{n^{1/2} \ConstC_{\ref{lem:sigma_n_boot_labmda_min_lowerbound}}^{1/2} } + \frac{d^{1/2}}{\ConstC_{\ref{lem:sigma_n_boot_labmda_min_lowerbound}}^{1/2}}  (\ConstJb{2}{1} + \ConstHb{2}{1}) \varphi_n \eqsp.
    \end{multline}
    Cauchy-Schwarz inequality and \Cref{lem:D_boot_copies_diff_bound} imply the bound for the third term in \eqref{eq:shao_zhang_bound_appendix}:
    \begin{align}
    \PE^\boot[\norm{\xi^\boot_i} \norm{D^\boot - D^{\boot, i}}] \leq \{\PE^\boot[\norm{\xi^\boot_i}^2]\}^{1/2} \{\PE^\boot[\norm{D^\boot - D^{\boot, i}}^2]\}^{1/2} \leq \frac{1}{n} \frac{1}{\sqrt{\ConstC_{\ref{lem:sigma_n_boot_labmda_min_lowerbound}}}} \ConstC_{\ref{prop:Qell:bound}} \supconsteps \sqrt{\alpha_i} \log(5n) \ConstC_{\ref{lem:D_boot_copies_diff_bound}} \eqsp.
    \end{align}
    Thus, it holds that
    \begin{multline}
        \label{eq:boot_xid_bound}
        \sum_{i=1}^{n-1} \PE^\boot[\norm{\xi^\boot_i} \norm{D^\boot - D^{\boot, i}}] \leq n^{-\gamma/2} \log(n) \frac{1}{\sqrt{\ConstC_{\ref{lem:sigma_n_boot_labmda_min_lowerbound}}}} \ConstC_{\ref{prop:Qell:bound}} \supconsteps \frac{\sqrt{c_0}}{1-\gamma/2}  \ConstC_{\ref{lem:D_boot_copies_diff_bound}, 1} \\+ n^{-\gamma} \frac{\norm{\theta_0 - \thetas}}{\sqrt{\ConstC_{\ref{lem:sigma_n_boot_labmda_min_lowerbound}}}} \ConstC_{\ref{prop:Qell:bound}} \supconsteps \frac{c_0}{1-\gamma/2}  \ConstC_{\ref{lem:D_boot_copies_diff_bound}, 2} \eqsp.
    \end{multline}
By collecting the inequalities derived above, we ultimately obtain the final bound.
\begin{align}
T_3 \leq \ConstC_{\ref{th:bootstrap_validity}, 1} \frac{\log(5n)}{n^{\gamma/2}} +\ConstC_{\ref{th:bootstrap_validity}, 2} \frac{(1+k_0)^{\gamma/2}}{n^{\gamma/2}} +  \ConstC_{\ref{th:bootstrap_validity}, 3}  \varphi_n + \ConstC_{\ref{th:bootstrap_validity}, 4} \frac{1}{\sqrt{n}} + \ConstC_{\ref{th:bootstrap_validity}, 5} \frac{1}{n^{\gamma}} \eqsp,
\end{align}
where the constants $\ConstC_{\ref{th:bootstrap_validity}, 1}$, $\ConstC_{\ref{th:bootstrap_validity}, 2}$, $\ConstC_{\ref{th:bootstrap_validity}, 3}$,  are given by:
\begin{align}
    &\ConstC_{\ref{th:bootstrap_validity}, 1} = 2d^{1/2} \ConstC_{\ref{lem:sigma_n_boot_labmda_min_lowerbound}}^{-1/2} (\ConstC_{\ref{lemma:EbJk1:bound}}  + \ConstC_{\ref{lemma:EbJk2:bound},1} + \ConstC_{\ref{lem:jkb1:bound}}) + 2 \ConstC_{\ref{lem:sigma_n_boot_labmda_min_lowerbound}}^{-1/2} \ConstC_{\ref{prop:Qell:bound}} \supconsteps \frac{\sqrt{c_0}}{1-\gamma/2}  \ConstC_{\ref{lem:D_boot_copies_diff_bound}} \eqsp, \quad \ConstC_{\ref{th:bootstrap_validity}, 2} = 2d^{1/2} \ConstC_{\ref{lem:sigma_n_boot_labmda_min_lowerbound}}^{-1/2} \ConstC_{\ref{lemma:EbJk2:bound},2}  \eqsp,\\ &\ConstC_{\ref{th:bootstrap_validity}, 3} = 2d^{1/2} \ConstC_{\ref{lem:sigma_n_boot_labmda_min_lowerbound}}^{-1/2} (\ConstJb{2}{1} + \ConstHb{2}{1}) \eqsp, \eqsp \ConstC_{\ref{th:bootstrap_validity}, 4} = 259 m_3 \ConstC_{\ref{lem:sigma_n_boot_labmda_min_lowerbound}}^{-1/2} d^{3/2} \supconsteps \ConstC_{\ref{prop:Qell:bound}} \eqsp, \eqsp \ConstC_{\ref{th:bootstrap_validity}, 5} =  \frac{2}{\sqrt{\ConstC_{\ref{lem:sigma_n_boot_labmda_min_lowerbound}}}} \ConstC_{\ref{prop:Qell:bound}} \supconsteps \frac{c_0}{1-\gamma/2}  \ConstC_{\ref{lem:D_boot_copies_diff_bound}, 2} \eqsp.
\end{align}
Rearranging the terms above yields the statement with the expressions $\Delta_{\ref{th:bootstrap_validity}, 1}$ to $\Delta_{\ref{th:bootstrap_validity}, 4}$ and $\ConstC_{\ref{th:bootstrap_validity}}$ given by  
\begin{align}
    \label{eq:delta_th_boot_validity}
    &\ConstC_{\ref{th:bootstrap_validity}} = \ConstC_{\ref{theo:GAR}, 4} + \frac{2}{\sqrt{\ConstC_{\ref{lem:sigma_n_boot_labmda_min_lowerbound}}}} \ConstC_{\ref{prop:Qell:bound}} \supconsteps \frac{c_0}{1-\gamma/2}  \ConstC_{\ref{lem:D_boot_copies_diff_bound}, 2} \eqsp, \\
    &\Delta_{\ref{th:bootstrap_validity}, 1} = \ConstC_{\ref{theo:GAR}, 1} + \frac{3\sqrt{2}\sqrt{d} \ConstC_{\ref{lem:matrix_bernstein}} \sqrt{\log{(10 d n)}}}{2\ConstC_{\ref{lem:sigma_n_labmda_min_lowerbound}}} + \ConstC_{\ref{th:bootstrap_validity}, 4}  \eqsp, \\
    &\Delta_{\ref{th:bootstrap_validity}, 2} = \ConstC_{\ref{theo:GAR}, 2} + \ConstC_{\ref{th:bootstrap_validity}, 1} \log(5n) + \ConstC_{\ref{th:bootstrap_validity}, 2}(1 + k_0)^{\gamma/2} \eqsp, \\
    &\Delta_{\ref{th:bootstrap_validity}, 3} = \ConstC_{\ref{theo:GAR}, 3} + \ConstC_{\ref{th:bootstrap_validity}, 3} \eqsp, \\
    &\Delta_{\ref{th:bootstrap_validity}, 4} = \frac{3\sqrt{d} \ConstC_{\ref{lem:matrix_bernstein}} \log{(10 d n)}}{6\ConstC_{\ref{lem:sigma_n_labmda_min_lowerbound}}} \eqsp.
\end{align}
\end{proof} 

\begin{lemma}
\label{lem:proof-Omega-Matrix-Bernstein-0}
Under the assumptions of \Cref{th:bootstrap_validity}, $\PP(\Omega_3) \geq 1-1/(5n)$
\end{lemma}
\begin{proof}
\, The proof follows from \Cref{lem:gamma_deviation_bound} and union bound.
\end{proof}

\begin{lemma}
\label{lem:proof-Omega-Matrix-Bernstein}
Under the assumptions of \Cref{th:bootstrap_validity}, $\PP(\Omega_4) \geq 1-1/(5n)$
\end{lemma}

\begin{proof}
\, To show that $\PP(\Omega_4) \geq 1-1/(5n)$, we fix $h \in [1, n]$ and $m \in [0, n - h - 1]$, and define the random variable
\[
T_n = \norm{\sum_{\ell = m+1}^{m+h} \alpha_\ell (\Am_\ell - \bA)}.
\]
We first control its variance. By standard matrix inequalities, we have:
\begin{align}
\max\Bigg( 
\norm{\sum_{\ell = m+1}^{m+h} \alpha_{\ell}^2 \, \PE\big[(\Am_\ell - \bA)(\Am_\ell - \bA)^\top\big]},\ 
\norm{\sum_{\ell = m+1}^{m+h} \alpha_{\ell}^2 \, \PE\big[(\Am_\ell - \bA)^\top (\Am_\ell - \bA)\big]}
\Bigg)
\leq \bConst{A}^2 \sum_{\ell = m+1}^{m+h} \alpha_\ell^2 \eqsp,
\end{align}
and note that for each $\ell$, the operator norm satisfies $\norm{(\Am_\ell - \bA)(\Am_\ell - \bA)^\top} \leq \bConst{A}^2$. Applying the matrix Bernstein inequality \cite{tropp2015introduction}, we obtain that with probability at least $1 - \delta / n^2$,
\begin{align}
T_n \leq \bConst{A} \sqrt{2 \sum_{\ell = m+1}^{m+h} \alpha_\ell^2 \log\left(\frac{2n^2 d}{\delta}\right)}
+ \frac{\alpha_{m+1} \bConst{A}}{3} \log\left(\frac{2n^2 d}{\delta}\right) \eqsp.
\end{align}
The remainder of the proof follows by setting $\delta = 1/(5n)$ and applying a union bound over all valid pairs $(h, m)$, along with the inequality $\norm{B}[Q]^2 \leq \qcond \norm{B}^2$, which holds for any matrix $B \in \rset^{d \times d}$.
\end{proof}

\begin{lemma}
\label{lem:matrix_bernstein}
Under the assumption of \Cref{th:bootstrap_validity}, $\PP(\Omega_5) \geq 1 - 1/(5n)$, where $\ConstC_{\ref{lem:matrix_bernstein}} = 2 \supconsteps^2 \ConstC_{\ref{prop:Qell:bound}}^2$.  
\end{lemma}
\begin{proof}
    \,
Introduce a random matrix $U_\ell = Q_\ell (\eps_\ell \eps_\ell^\top - \noisecov) Q_\ell^\top$. Note that $\expe{U_\ell} = 0$ and $\Sigma_n^\boot - \Sigma_n = \frac{1}{n}\sum_{\ell=1}^{n-1} U_\ell$. Moreover, \Cref{prop:Qell:bound}, \Cref{assum:iid}, and \Cref{assum:noise-level} imply that 
    \begin{equation}
        \label{lemma:Uell_bound}
        \normop{U_\ell} \leq (\supconsteps^2 + \normop{\noisecov}) \ConstC_{\ref{prop:Qell:bound}}^2 \leq 2 \supconsteps^2 \ConstC_{\ref{prop:Qell:bound}}^2 = \ConstC_{\ref{lem:matrix_bernstein}} \eqsp.
    \end{equation}
    Hence, the matrix Bernstein inequality \cite[Theorem 6.1.1]{tropp2015introduction} implies
    \begin{align}
    \proba{\normop{\Sigma_n^\boot - \Sigma_n} \geq t} = \PP\bigl(\norm{\sum_{\ell=1}^{n-1} U_\ell} \geq n t\bigr) \leq 2 d  \exp\biggl(- \frac{n t^2}{2\ConstC_{\ref{lem:matrix_bernstein}}^2 + 2\ConstC_{\ref{lem:matrix_bernstein}} t/3}\biggr) \eqsp.
    \end{align}
    Equivalently (see e.g. \cite[Theorem 2.10]{blm:2013}), with probability at least $1-\delta$, it holds that  
    \[
    \normop{\Sigma_n^\boot - \Sigma_n} \leq \frac{\sqrt{2} \ConstC_{\ref{lem:matrix_bernstein}} \sqrt{\log{(2d/\delta)}}}{\sqrt{n}} + \frac{\ConstC_{\ref{lem:matrix_bernstein}} \log{(2d/\delta)}}{3n}\eqsp.
    \]
    To complete the proof it remains to take $\delta = 1/(5n)$. 
\end{proof}

\begin{lemma}
\label{lem:sigma_n_boot_labmda_min_lowerbound}
    Under the assumption of \Cref{th:bootstrap_validity}, on the event $\Omega_5$ it holds that
    \begin{equation}
        \label{eq:lambda_min_boot_bound}
        \lambda_{\min}(\Sigma_{n}^\boot) \geq \ConstC_{\ref{lem:sigma_n_boot_labmda_min_lowerbound}} \eqsp, \text{ where } \ConstC_{\ref{lem:sigma_n_boot_labmda_min_lowerbound}} =  \frac{\lambda_{\min}(\Sigma_\infty)}{4}\eqsp.
    \end{equation}
\end{lemma}
\begin{proof}
\, Using Lidskii's inequality for Hermitian matrices, we get that 
\[
\lambda_{\min}(\Sigma_n^\boot) \geq \lambda_{\min}(\Sigma_n) + \lambda_{\min}(\Sigma_n^{\boot} - \Sigma_n)
\geq 
\lambda_{\min}(\Sigma_n) - \normop{\Sigma_n^{\boot} - \Sigma_n}.
\]
Hence, on the event $\Omega_5$, we get that
\[
\lambda_{\min}(\Sigma_n^\boot) \geq \lambda_{\min}(\Sigma_n) 
- \left(\frac{\sqrt{2} \ConstC_{\ref{lem:matrix_bernstein}} \sqrt{\log{(10 d n)}}}{\sqrt{n}} + \frac{\ConstC_{\ref{lem:matrix_bernstein}} \log{(10 d n)}}{3n} \right).
\]
Under \Cref{assum:step-size-bootstrap}, the sample size $n$ is chosen large enough so that 
\[
\lambda_{\min}(\Sigma_n^\boot) \geq \lambda_{\min}(\Sigma_n) 
- \lambda_{\min}(\Sigma_\infty)/4.
\]
From \eqref{bound:norm_of_bL_minus_Lstar}, using again Lidskii's inequality this time with $\Sigma_n$ and $\Sigma_\infty$, we know that under \Cref{assum:step-size}, $\lambda_{\min}(\Sigma_n) \geq \lambda_{\min}(\Sigma_\infty)/2$. The proof follows.
\end{proof}

\begin{lemma}
\label{lem:gamma_deviation_bound}
Under the assumptions of \Cref{th:bootstrap_validity}, For each $\ell \in \{1,\ldots,n-1\}$, it holds that 
\begin{equation}
\label{eq:Jk2_terms_bound_high_proba}
\PP\biggl(\norm{\alpha_\ell \sum_{k=\ell}^{n-1} (\Gamma_{\ell+1:k} - G_{\ell+1:k}) \eps_\ell} \geq \log(5n) \ConstC_{\ref{lem:gamma_deviation_bound}} \sqrt{\alpha_\ell}\biggr) \leq \frac{1}{5n^2}\eqsp.
\end{equation}
where we have defined 
\begin{equation}
\label{eq:const_c_gamma_deviation_bound}
\ConstC_{\ref{lem:gamma_deviation_bound}} = 2 (\sqrt{8}/\sqrt{7}) \rme^2 \bConst{A} \supconsteps \ConstC_{\ref{prop:Qell:bound}} \qcond^{1/2} \biggl(c_0 + \frac{2}{a(1-\gamma)}\biggr)^{1/2}\eqsp. 
\end{equation}

\end{lemma}

\begin{lemma}
    \label{lem:jkb1:bound}
Under the assumptions of \Cref{th:bootstrap_validity},  conditionally on the event $\Omega_0$, it holds
    \begin{align}
        \bigl\{\PE^\boot \bigl[\norm{\frac{1}{\sqrt{n}} \sum_{k=1}^{n-1} J_{k}^{\boot, 1}}^2\bigr]\bigr\}^{1/2} \leq  \frac{\ConstC_{\ref{lem:jkb1:bound}}}{n^{\gamma/2}} \eqsp,
    \end{align}
    where the constant $\ConstC_{\ref{lem:jkb1:bound}}$ is defined as follows
    \begin{align}
        \ConstC_{\ref{lem:jkb1:bound}} = \frac{2\sqrt{3}}{\sqrt{a}} \bConst{A}  \bigl(c_0 + \frac{2}{a(1-\gamma)}\bigr) \sqrt{\frac{c_0}{1-\gamma}} \ConstC_{\ref{lem:matr_product_as_bound}} \ConstC_{\ref{lem:tilde_eps_boot_bound}}   \eqsp.
    \end{align}
\end{lemma}
\noindent Let $w'_i$ be a copy of $w_i$ independent from $w_1, \ldots, w_{n-1}$. Introduce the following notation for $\ell \leq m$:
\begin{align}
    \label{...}
    &\Gamma^{\boot, i}_{\ell:m} = \begin{cases}
        \Gamma_{\ell:m}^\boot, &\text{if } i \not\in [\ell, m], \\
        \Gamma_{\ell:m}^\boot(w_\ell, \ldots, w_{i-1}, w_{i}', w_{i+1}, \ldots, w_m), &\text{if } i \in [\ell, m]
    \end{cases} \\
    &\Jnalpha{k, m}{\boot, \ell, i} = \begin{cases}
        \Jnalpha{k, m}{\boot, \ell}, &\text{if } k <i, \\
        \Jnalpha{k, m}{\boot, \ell}(w_1, \ldots, w_{i-1}, w_{i}', w_{i+1}, \ldots, w_k), &\text{if } k \geq i
    \end{cases} \\
    &\Hnalpha{k}{\boot, \ell, i} = \begin{cases}
        \Hnalpha{k}{\boot, \ell}, &\text{if } k <i, \\
        \Hnalpha{k}{\boot, \ell}(w_1, \ldots, w_{i-1}, w_{i}', w_{i+1}, \ldots, w_k), &\text{if } k \geq i
    \end{cases} \\
    &D^{\boot, i} = \frac{1}{\sqrt{n}} \sum_{k=1}^{n-1} \Jnalpha{k, 1}{\boot, 0, i} + \frac{1}{\sqrt{n}} \sum_{k=1}^{n-1} \Jnalpha{k, 2}{\boot, 0, i}+  \frac{1}{\sqrt{n}} \sum_{k=1}^{n-1} \Hnalpha{k}{\boot, 0, i}
\end{align}
For simplicity we introduce the following constants:
\begin{align}
    \label{def:bootstrap_constants}
    & \ConstC_{\ref{def:bootstrap_constants}}^{(1)} = c_0 + \frac{16}{a(1-\gamma)} \eqsp, \eqsp \ConstC_{\ref{def:bootstrap_constants}}^{(2)} = \frac{1}{1-ac_0} (1+ \frac{16}{a c_0(1-\gamma)}) \eqsp.
\end{align}

\begin{lemma}
\label{lem:Hk0b_copies_sum_bound}
Under the assumptions of \Cref{th:bootstrap_validity},  let $w'_i$ be a copy of $w_i$ independent from $w_1, w_2, \ldots, w_{n-1}$. Then on the event $\Omega_0$ it holds that
\begin{align}
\bigl\{\PE^\boot\bigl[\norm{\frac{1}{\sqrt{n}} \sum_{k=1}^{n-1} \Hnalpha{k}{\boot, 0} - \frac{1}{\sqrt{n}} \sum_{k=1}^{n-1} \Hnalpha{k}{\boot, 0, i}}^2\bigr]\bigr\}^{1/2} \leq n^{-1/2} \sqrt{\alpha_i} \ConstC_{\ref{lem:Hk0b_copies_sum_bound}}
\end{align}
where the constant $\ConstC_{\ref{lem:Hk0b_copies_sum_bound}}$ is given by
\begin{align}
    \ConstC_{\ref{lem:Hk0b_copies_sum_bound}} = \bConst{A} \ConstC_{\ref{prop:product_random_matrix_bootstrap}}  (\ConstJb{0}{1} + \ConstHb{0}{1}) \ConstC_{\ref{def:bootstrap_constants}}^{(1)} + \bConst{A} (\ConstC_{\ref{def:bootstrap_constants}}^{(1)})^{3/2} \ConstC_{\ref{prop:product_random_matrix_bootstrap}}  \ConstC_{\ref{lem:matr_product_as_bound}} \ConstC_{\ref{lem:tilde_eps_boot_bound}}  \eqsp.
\end{align}
\end{lemma}

\begin{lemma}
\label{lem:D_boot_copies_diff_bound}
Under the assumptions of \Cref{th:bootstrap_validity}, conditionally on $\Omega_0$, it holds that
    \begin{align}
        \{\PE^\boot[\norm{D^\boot - D^{\boot, i}}^2]\}^{1/2} \leq \frac{1}{\sqrt{n}} \sqrt{\alpha_i} \log(5n) \ConstC_{\ref{lem:D_boot_copies_diff_bound}, 1} + \ConstC_{\ref{lem:D_boot_copies_diff_bound}, 2} \frac{\alpha_i \norm{\theta_0 - \thetas}}{\sqrt{n}} \eqsp,
    \end{align}
    where the constant $\ConstC_{\ref{lem:D_boot_copies_diff_bound}}$ is given by
    \begin{align}
        \ConstC_{\ref{lem:D_boot_copies_diff_bound}, 1} = \frac{ 2 \sqrt{2}  \ConstC_{\ref{lem:gamma_deviation_bound}}}{\sqrt{\ConstC_{\ref{lem:sigma_n_boot_labmda_min_lowerbound}}}} & + \frac{\ConstC_{\ref{lem:Hk0b_copies_sum_bound}}}{\sqrt{\ConstC_{\ref{lem:sigma_n_boot_labmda_min_lowerbound}}}} 
        + \frac{2\sqrt{2} \rme \ConstC_{\ref{lem:matr_product_as_bound}}  \bConst{A}  \qcond^{1/2} \supconsteps \ConstC_{\ref{def:bootstrap_constants}}^{(1)}}{\sqrt{\ConstC_{\ref{lem:sigma_n_boot_labmda_min_lowerbound}}}} \eqsp, \eqsp \ConstC_{\ref{lem:D_boot_copies_diff_bound}, 2} = \frac{\sqrt{2} \qcond^{1/2} \rme^2 \ConstC_{\ref{lem:matr_product_as_bound}}  \bConst{A} \ConstC_{\ref{def:bootstrap_constants}}^{(2)}}{\sqrt{\ConstC_{\ref{lem:sigma_n_boot_labmda_min_lowerbound}}}(1-a/2)^2} \eqsp.
    \end{align}
\end{lemma}

\subsection{Proofs on products of random matrices}
\label{appendix:sigmas_difference_bound}
We first introduce some notations and definitions. For a matrix $\MatB \in \rset^{d \times d}$ we denote by $(\sigma_\ell(\MatB))_{\ell=1}^d$ its singular values. For $\qexponent \geq 1$, the Shatten $\qexponent$-norm of $B$ is denoted by $\norm{\MatB}[\qexponent] = \{\sum_{\ell=1}^d \sigma_\ell^\qexponent (\MatB)\}^{1/\qexponent}$. For $\qexponent, \ppexponent \geq 1$ and a random matrix $\X$ we write $\norm{\X}[\qexponent,\ppexponent] = \{ \PE[\norm{\X}[\qexponent]^\ppexponent] \}^{1/\ppexponent}$. We use a result of \cite{huang2020matrix}, sharpened in \cite{durmus2022finite}.

\begin{lemma}[Proposition~15 in \cite{durmus2022finite}]
\label{th:general_expectation}
Let $\sequence{\Y}[\ell][\nset]$ be an independent sequence and $P$ be a positive definite matrix. Assume that for each $\ell \in \nset$ there exist $m_\ell \in (0,1)$  and $\sigma_{\ell} > 0$ such that \(\norm{\PE[\Y_\ell]}[P]^2  \leq 1 - m_\ell\) and \(\norm{\Y_\ell - \PE[\Y_\ell]}[P] \leq \sigma_{\ell}\) almost surely.  Define $\Zbf_k = \prod_{\ell = 0}^k \Y_\ell= \Y_k \Zbf_{k-1}$, for $k \geq 1$ and starting from $\Zbf_0$. Then, for any $2 \le q \leq p$ and $k \geq 1$,
\begin{equation} 
\label{eq:gen_expectation}
\norm{\Zbf_k}[p,q]^2 \leq \kappa_P \prod_{\ell=1}^{k} (1- m_\ell + (p-1)\sigma_{\ell}^2) \norm{P^{1/2}\Zbf_0 P^{-1/2}}[p, q]^2 \eqsp,
\end{equation}
where  $\kappa_P = \lambda_{\sf min}^{-1}( P )\lambda_{\sf max}( P )$.
\end{lemma}
Now we aim to bound $\ProdB_{m:k}$ defined in \eqref{eq:prod_rand_matr} using \Cref{th:general_expectation}. 
Set $\Y_\ell = \Id - \alpha_\ell \Am_\ell, \ell \geq 1$, and $\Y_0 = \Id$. Applying the bound \eqref{eq:contractin_q_norm}, we get $\norm{\PE[\Y_\ell]}[Q]^2 = \norm{\Id - \alpha_{\ell} \bA }[Q]^2 \leq 1 - a \alpha_{\ell}$. Further, assumption \Cref{assum:noise-level} implies that almost surely,
\[
\norm{\Y_\ell - \PE[\Y_\ell]}[Q] =  \alpha_\ell \norm{ \Am_\ell- \bA}[Q] \leq   \alpha_\ell \sqrt{\qcond} \bConst{A}  = b_{Q} \alpha_\ell \eqsp.
\]
Therefore, \eqref{eq:gen_expectation} holds with $m_\ell = a \alpha_\ell$ and  $\sigma_{\ell} = b_Q \alpha_\ell$. As $\norm{\Id}[p] = d^{1/p}$, we obtain the following corollary.

\begin{corollary}
\label{cor:norm_Gamma_m_n}
Assume \Cref{assum:iid} and \Cref{assum:noise-level}. Then, for any $\alpha_\ell \in [0, \alpha_{\infty}]$, $2 \le q \le p$, and $1 \leq m \leq k$, it holds
\begin{equation}
\label{eq:concentration iid}
\PE^{1/q}\left[ \normop{\ProdB_{m:k}}^{q} \right]  
\leq  \norm{\ProdB_{m:k}}[p,q] 
\leq \sqrt{\qcond} d^{1/p} \prod_{\ell=m}^{k}(1 - a \alpha_\ell + (p-1) b_Q^2 \alpha_\ell^2) \eqsp,
\end{equation}
where $\alpha_\infty$ was defined in \eqref{eq:alpha_infty_def}, and $b_{Q} =  \sqrt{\qcond} \bConst{A}$. 
\end{corollary}

\begin{proposition}
\label{prop:product_random_matrix_bootstrap}
Assume \Cref{assum:iid}, \Cref{assum:noise-level}, \Cref{assum:step-size}($\log (5n^2) \vee \log d$), \Cref{assum:sample_size}, \Cref{assum:step-size-bootstrap}. Then on the set $\Omega_4$ defined in \eqref{eq:definition-Omega-4}, it holds for any $0 \leq m \leq k \leq n$, that 
\begin{equation}
\label{eq:stability_bound_matrix_products_bootstrap}
\bigl\{\PEb[\norm{\ProdB^\boot_{m+1:k}}^2] \bigr\}^{1/2} \leq \ConstC_{\ref{prop:product_random_matrix_bootstrap}} \exp\bigl\{-\frac{a}{4}\sum_{\ell=m+1}^{k}\alpha_{\ell} \bigr\}\eqsp, \eqsp \ConstC_{\ref{prop:product_random_matrix_bootstrap}} = \qcond^{3/2} \rme^{9/8} \eqsp.
\end{equation}
\end{proposition}

\begin{lemma}
\label{lem:product_ramdom_matrix_aux}
Assume \Cref{assum:iid}, \Cref{assum:noise-level}, \Cref{assum:step-size}($\log (5n^2) \vee \log d$), \Cref{assum:sample_size}, \Cref{assum:step-size-bootstrap}. Then on the event $\Omega_4$, defined in \eqref{eq:definition-Omega-4}, with $h = h(n)$ defined in \eqref{eq:block_size_constraint}, it holds for any $m \in [0;n-h-1]$, that 
$$
\bigl\{\PEb[\norm{\ProdB^\boot_{m+1:m+h}}[Q]^2] \bigr\}^{1/2} \leq  \exp\bigl\{-\frac{a}{4}\sum_{\ell=m+1}^{m+h}\alpha_{\ell} \bigr\} \eqsp,
$$
where $\ProdB^\boot_{m+1:m+h}$ is defined in \eqref{eq:prod_rand_matr-bootstrap}.
\end{lemma}


\subsection{Proof of  \Cref{lem:sigma_n_bound}}
Before proceeding with the actual proof, we introduce a decomposition of $\Sigma_n - \Sigma_\infty$ that forms the backbone of the argument. This decomposition is built upon non-trivial identities involving both $Q_t - \bA^{-1}$ and the cumulative sum $\sum_{t=1}^{n-1} (Q_t - \bA^{-1})$, as outlined and established in \citep[pp. 26–30]{wu2024statistical}.
\begin{equation}
    \label{repr:qtminusa}
    Q_t - \bA^{-1} = S_t - \bA^{-1} G_{t:n},~ S_t = \sum_{j=t+1}^{n-1} (\alpha_t - \alpha_j) G_{t+1:j-1}
\end{equation}
\begin{equation}
    \label{repr:sumqtminusa}
    \sum_{t=1}^{n-1} (Q_t - \bA^{-1}) = -\bA^{-1} \sum_{j=1}^{n-1} G_{1:j}
\end{equation}
In the following, we will require a bound on the operator norm of the matrix $S_t$, which is provided below:
\begin{lemma}
    \label{lem:st_bound}
    Assume \Cref{assum:iid,assum:noise-level,assum:step-size} with $p= 2 \vee \log(d)$. Let $c_0 \in (0, \alpha_{\infty}]$ and $t \in \mathbb{N}$. Then    
    \begin{equation}
        \label{bound:stmatrix}
        \normop{S_t} \leq \sqrt{\qcond}  \ConstC_{\ref{lem:st_bound}}  (t+k_0)^{\gamma-1} \eqsp,
    \end{equation}
    where $\ConstC_{\ref{lem:st_bound}}$ is given by
    \begin{align}
        \label{def:const:lem:st_bound}
        \ConstC_{\ref{lem:st_bound}} =  \frac{c_0}{1-\gamma} \exp\biggl(ac_0+ \frac{ac_0}{2(1-\gamma)}\biggr) \biggl(\frac{ac_0}{2(1-\gamma)}\biggr)^{-\frac{1}{1-\gamma}}  \biggl(\max\{1, \phi(x_\gamma)\} \left(\frac{ac_0}{2} + x_\gamma\right) + \int_{x_\gamma}^{\infty} \phi(x)\biggr) \eqsp,
    \end{align}
    and $x_\gamma = \frac{\gamma}{1-\gamma}$, $\phi(x) = x^{\frac{\gamma}{1-\gamma}} \exp(-x)$. 
\end{lemma}
Since $\Sigma_\infty = \bA^{-1} \Sigma_\eps \bA^{-\top}$, elementary manipulations with $\Sigma_n - \Sigma_\infty$ imply the following equality:
\begin{multline}
\label{repr:Lstar_minus_Ln}
    \Sigma_n - \Sigma_{\infty} = \underbrace{\frac{1}{n}\sum_{t=1}^{n-1} (Q_t - \bA^{-1})\noisecov \bA^{-\top} + \frac{1}{n} \sum_{t=1}^{n-1} \bA^{-1} \noisecov (Q_t - \bA^{-1})^\top}_{D_1} \\ + \underbrace{\frac{1}{n} \sum_{t=1}^{n-1} (Q_t - \bA^{-1}) \noisecov (Q_t - \bA^{-1})^{\top}}_{D_2} - \frac{1}{n} \Sigma_{\infty}
\end{multline}
The decomposition \eqref{repr:Lstar_minus_Ln} is crucial to obtain the convergence rate of $\Sigma_n - \Sigma_\infty$. The proof would follow from estimating $D_1$ and $D_2$ separately by expressions of order $n^{\gamma-1}$.
For simplicity, we introduce the following notation: 
\begin{equation}
\label{def:gnm}
    g^{(\gamma)}_{k:m} = \sum_{\ell=k}^m (\ell+k_0)^{-\gamma},~ k \leq m \eqsp .
\end{equation}

\begin{proof}[Proof of \Cref{lem:sigma_n_bound}]
\label{appendix:proof_sigma_n_bound}
We first provide an expression for the constant $\ConstC_{\ref{lem:sigma_n_bound}}$:
\begin{equation}
\label{eq:def_mathcal_c_inf}
\ConstC_{\ref{lem:sigma_n_bound}} = \norm{\Sigma_\infty} + \frac{2^{1+\gamma} \norm{\Sigma_{\infty}}  \sqrt{\qcond}  \ConstC_{\ref{prop:Qell:bound}}}{c_0} + \frac{\normop{\noisecov} \qcond \left(\ConstC_{\ref{lem:st_bound}}\right)^2}{2\gamma - 1} + \frac{2^\gamma \qcond \normop{\Sigma_\infty}}{ac_0 - (ac_0/2)^2} + \frac{4\qcond \normop{\bA \Sigma_{\infty}} \ConstC_{\ref{lem:st_bound}}}{ac_0}\eqsp.
\end{equation}
Using \eqref{repr:Lstar_minus_Ln}, we get
\begin{equation}
    \normop{\Sigma_n - \Sigma_{\infty}} \leq \frac{1}{n} \normop{\Sigma_\infty} + \normop{D_1} + \normop{D_2}\eqsp.
\end{equation}
We first bound $D_1$. The operator norms of both terms are equal because one is a transposed version of another, so it is sufficient to bound only one of them. Note that $G_{n:m}$, $Q_t$, $\bA$, $\bA^{-1}$ commute as polynomials in $\bA$. Now we use \eqref{repr:sumqtminusa} and obtain
\begin{align}
    \norm{\frac{1}{n}\sum_{t=1}^{n-1} (Q_t - \bA^{-1})\noisecov \bA^{-\top}} &= \norm{-\frac{1}{n} \bA^{-1} \sum_{j=1}^{n-1} G_{1:j}\noisecov \bA^{-\top}} = \norm{n^{-1} \Sigma_{\infty}  \sum_{j=1}^{n-1} G_{1:j}} \leq n^{-1} \norm{\Sigma_{\infty}}  \norm{\sum_{j=1}^{n-1} G_{1:j}} \eqsp.
\end{align}
\Cref{prop:Qell:bound} directly implies the bound for $D_1$:
\begin{align}
    \frac{1}{n}\norm{\sum_{t=1}^{n-1} (Q_t - \bA^{-1})\noisecov \bA^{-\top}} \leq \frac{\norm{\Sigma_{\infty}}  \norm{\sum_{j=1}^{n-1} G_{1:j}}}{n} &\leq \frac{\norm{\Sigma_{\infty}}  \sqrt{\qcond}  \ConstC_{\ref{prop:Qell:bound}} (1+k_0)^\gamma}{nc_0} \leq  \frac{2^{\gamma} n^{\gamma-1} \norm{\Sigma_{\infty}}  \sqrt{\qcond}  \ConstC_{\ref{prop:Qell:bound}}}{c_0} \eqsp.
\end{align}
Hence,
\begin{align}
    \label{lemma:D1_bound}
    \normop{D_1} \leq \frac{2^{1+\gamma} n^{\gamma-1}\norm{\Sigma_{\infty}}  \sqrt{\qcond}  \ConstC_{\ref{prop:Qell:bound}}}{c_0} \eqsp.
\end{align}
We now consider $D_2$. Using \eqref{repr:qtminusa}, we get
\begin{align}
    \label{lemma:D2_expanded_repr}
    n^{-1}\sum_{t=1}^{n-1}(Q_t - \bA^{-1}) \noisecov (Q_t - \bA^{-1})^\top &= \underbrace{n^{-1} \sum_{t=1}^{n-1} S_t \noisecov S_t^\top}_{D_{21}} + \underbrace{n^{-1}\sum_{t=1}^{n-1}\bA^{-1} \prod_{k=t}^{n-1} (\Id - \alpha_k \bA) \noisecov \bA^{-\top} \prod_{k=t}^{n-1} (\Id - \alpha_k \bA)^\top}_{D_{22}} \\ 
    & \qquad  -\underbrace{n^{-1} \sum_{t=1}^{n-1} \bA^{-1} \prod_{k=t}^{n-1} (\Id - \alpha_k \bA)  \noisecov S_t^\top}_{D_{23}} - \underbrace{n^{-1}\sum_{t=1}^{n-1} S_t \noisecov \bA^{-\top} \prod_{k=t}^{n-1} (\Id - \alpha_k \bA)^{\top}}_{D_{24}}\eqsp.
\end{align}
{\Cref{lem:st_bound}}  reveal an evident bound for $\normop{D_{21}}$:
\begin{align}
    \label{lemma:D21_bound}
    \normop{D_{21}} = \norm{n^{-1} \sum_{t=1}^{n-1} S_t \noisecov S_t^\top} &\leq  n^{-1} \normop{\noisecov} \sum_{t=1}^{n-1} \qcond \left(\ConstC_{\ref{lem:st_bound}}\right)^2 t^{2(\gamma-1)} \leq n^{2(\gamma - 1)} \frac{\normop{\noisecov} \qcond \left(\ConstC_{\ref{lem:st_bound}}\right)^2}{2\gamma - 1} \eqsp.
\end{align}
Note that
\begin{align}
\sum_{t=1}^{n-1} \normop{G_{t:n-1}}^2 &\leq \qcond  \sum_{t=1}^{n-1} \prod_{k=t}^{n-1} (1 - \frac{ac_0}{2} (k+k_0)^{-\gamma})^2 
\leq  \qcond  \sum_{t=1}^{n-1} (1 - \frac{ac_0}{2} (n-1+k_0)^{-\gamma})^{n-t} \eqsp,
\end{align}
 The bound for $\normop{D_{22}}$ follows from the above inequality:
\begin{align} 
    \label{lemma:D22_bound}
    \normop{D_{22}} & \leq n^{-1} \normop{\Sigma_\infty} \sum_{t=1}^{n-1} \normop{G_{t:n-1}}^2 \leq n^{-1}  \frac{\qcond \normop{\Sigma_{\infty}}}{ac_0 (n+k_0)^{-\gamma} - (ac_0/2)^2 (n+k_0)^{-2\gamma}} \leq n^{\gamma - 1} \frac{2^\gamma \qcond \normop{\Sigma_\infty}}{ac_0 - (ac_0/2)^2} \eqsp.
\end{align}
Since $D_{23} = D_{24}^\top$, we concentrate on $\normop{D_{24}}$. {\Cref{lem:st_bound}} and \Cref{lem:summ_alpha_k}-\ref{lem:summ_alpha_k_first} imply the following bound:
\begin{align}
    \normop{D_{24}} &= \norm{n^{-1}\sum_{t=1}^{n-1} S_t \noisecov \bA^{-\top} \prod_{k=t}^{n-1} (\Id - \alpha_k \bA)^{\top}} \leq n^{-1} \normop{\noisecov \bA^{-\top}} \sum_{t=1}^{n-1} \normop{S_t} \norm{\prod_{k=t}^{n-1} (\Id - \alpha_k \bA)^{\top}} \\ &\leq n^{-1} \normop{\bA \Sigma_{\infty}} \sum_{t=1}^{n-1} \sqrt{\qcond} (t+k_0)^{\gamma-1} \ConstC_{\ref{lem:st_bound}} \sqrt{\qcond} \prod_{k=t+1}^{n-1} (1- \frac{ac_0}{2} (k+k_0)^{-\gamma}) \leq 
    n^{2(\gamma-1)} \frac{2\qcond \normop{\bA \Sigma_{\infty}} \ConstC_{\ref{lem:st_bound}}}{ac_0} \eqsp.
\end{align}
Hence, 
\begin{equation}
    \label{lemma:proof:d23_d24_bound}
    \normop{D_{23}} + \normop{D_{24}} \leq n^{2(\gamma-1)} \frac{4\qcond \normop{\bA \Sigma_{\infty}} \ConstC_{\ref{lem:st_bound}}}{ac_0} \eqsp.
\end{equation}
The needed result follows from \eqref{lemma:D1_bound}, \eqref{lemma:D21_bound}, \eqref{lemma:D22_bound}, \eqref{lemma:proof:d23_d24_bound}.
\end{proof}
\subsection{Technical Lemmas}
\begin{lemma}[Lemma 1 in \cite{samsonov2024gaussian}]
\label{lem:markov_inequality}
Fix $\delta \in (0,1/\rme^2)$ and let $Y$ be a positive random variable, such that 
$\PE^{1/p}[Y^{p}] \leq C_{1} + C_{2} p$  for any $2 \leq p \leq \log{(1/\delta)}$. Then it holds with probability at least $1-\delta$, that 
\begin{equation}
\label{eq:markov_ineqality_deviation}
Y \leq \rme C_{1} + \rme C_{2} \log{(1/\delta)}\eqsp.
\end{equation}
\end{lemma}


\begin{lemma}
\label{lem:summ_alpha_k}
The following statement holds:
\begin{enumerate}[(i)]
    \item \; \label{lem:summ_alpha_k_first} Let $b > 0$ and $(\alpha_k)_{k \geq 0}$ be a non-increasing sequence such that $\alpha_0 \leq 1/b$. Then
    \[
    \sum_{j=1}^{k} \alpha_j \prod_{l=j+1}^{k} (1 - \alpha_l b) = \frac{1}{b} \biggl\{1  - \prod_{l=1}^{k} (1 - \alpha_l b) \biggr\} \eqsp.
    \]
    \item \; \label{lem:summ_alpha_k_p_item} 
    Let $b > 0$ and $\alpha_k = \frac{c_0}{(k+k_0)^\gamma}$, $\gamma \in (0,1)$, such that $c_0 \leq 1/b$ and $k_0^{1-\gamma} \geq \frac{8\gamma}{bc_0}$. Then for any $q \in (1, 4]$ it holds that
    \begin{align}
        \sum_{j=1}^k \alpha_j^q \prod_{\ell=j+1}^k (1 - \alpha_\ell b) \leq \frac{6}{b} \alpha_k^{q-1} \eqsp.
    \end{align}
    \item \; \label{lem:sum_as_Qell_item} Let $b, c_0, k_0 > 0$ and $\alpha_\ell = c_0 (\ell+k_0)^{-\gamma}$ for $\gamma \in (1/2, 1)$ and $\ell \in \nset$. Assume that $bc_0 < 1$ and $k_0^{1-\gamma} \geq \frac{1}{bc_0}$. Then, for any $\ell, n \in \nset$, $\ell \leq n$, it holds that
    \begin{equation}
    \label{eq:const_L_b_def}
    \sum_{k=\ell}^{n-1} \alpha_\ell\prod_{j=\ell+1}^{k} (1-b\alpha_j) \leq c_0 + \frac{1}{b(1-\gamma)} \eqsp.
    \end{equation}
\end{enumerate}
\end{lemma}

\begin{proof}
\Cref{lem:summ_alpha_k}-\ref{lem:summ_alpha_k_first} follows from Lemma~24 in \cite{durmus2021stability}. \Cref{lem:summ_alpha_k}-\ref{lem:summ_alpha_k_p_item} follows from Lemma~33 in \cite{samsonov2025statistical}. 
\eqref{lem:sum_as_Qell_item} is elementary.
\end{proof}

\begin{lemma}[Lemma 36 in \cite{samsonov2025statistical}]
\label{lem:bound_sum_exponent}
    For any $A >0$, any $1 \leq i \leq n-1$,   and $\gamma\in(1/2, 1)$ it holds
   \begin{equation}
        \sum_{j=i}^{n-1}\exp\biggl\{-A(j^{1-\gamma} - i^{1-\gamma})\biggr\} \leq
        \begin{cases}
            1 + \exp\bigl\{\frac{1}{1-\gamma}\bigr\}\frac{1}{A^{1/(1-\gamma)}(1-\gamma)}\Gamma(\frac{1}{1-\gamma})\eqsp, &\text{ if } Ai^{1-\gamma} \leq \frac{1}{1-\gamma}\eqsp;\\
            1 + \frac{1}{A(1-\gamma)^2}i^\gamma\eqsp,  &\text{ if } Ai^{1-\gamma} >\frac{1}{1-\gamma}\eqsp.
        \end{cases}
    \end{equation}
\end{lemma} 

\section{Conclusion}
\label{sec:conclusion}
In this paper, we have obtained a novel bound for the Gaussian approximation of the distribution of the Polyak–Ruppert averaged LSA iterates in the sense of convex distance. Compared to the previous analysis established in \cite{samsonov2024gaussian}, the fastest achievable rate of normal approximation has been improved from $n^{-1/4}$ to $n^{-1/3}$. We also derived a bootstrap-based approximation for the distribution $\sqrt{n}(\bar{\theta}_n - \thetas)$ with an error of order up to $1/\sqrt{n}$. Importantly, this result does not rely on the Gaussian approximation with the limiting covariance matrix $\Sigma{\infty}$. Among further directions, we list the generalization of the randomized concentration approach of \cite{shao2022berry} to the Markov setting, which enables the analysis of stochastic approximation problems with Markov noise. Current approaches \cite{srikant2024rates,wu2025markovchains,samsonov2025statistical} rely on versions of the Berry–Esseen inequalities for martingales, which require additional step size constraints and introduce extra $\log{n}$ factors. Another research direction would be to tighten the lower bound \eqref{eq:lower_bound_kolmogorov} in the regime where the step size exponent $\gamma \in (1/2, 2/3)$. Establishing a counterpart of \eqref{eq:lower_bound_kolmogorov} with the term $n^{-\gamma/2}$ would imply the optimality of the rate $n^{-1/3}$ in \Cref{th:shao2022_berry}. 

\bibliography{references}

\begin{thebibliography}{10}

\bibitem{aguech2000perturbation}
Rafik Aguech, Eric Moulines, and Pierre Priouret.
\newblock On a perturbation approach for the analysis of stochastic tracking algorithms.
\newblock {\em SIAM Journal on Control and Optimization}, 39(3):872--899, 2000.

\bibitem{pmlr-v99-anastasiou19a}
Andreas Anastasiou, Krishnakumar Balasubramanian, and Murat~A. Erdogdu.
\newblock Normal {A}pproximation for {S}tochastic {G}radient {D}escent via {N}on-{A}symptotic {R}ates of {M}artingale {CLT}.
\newblock In Alina Beygelzimer and Daniel Hsu, editors, {\em Proceedings of the Thirty-Second Conference on Learning Theory}, volume~99 of {\em Proceedings of Machine Learning Research}, pages 115--137. PMLR, 25--28 Jun 2019.

\bibitem{BarUly86}
S.~{Barsov} and V.~{Ulyanov}.
\newblock {Estimates for the closeness of Gaussian measures}.
\newblock {\em Dokl. Akad. Nauk SSSR}, 291(2):273--277, 1986.

\bibitem{benveniste2012adaptive}
A.~Benveniste, M.~M{\'e}tivier, and P.~Priouret.
\newblock {\em Adaptive algorithms and stochastic approximations}, volume~22.
\newblock Springer Science \& Business Media, 2012.

\bibitem{bhandari2018finite}
J.~Bhandari, D.~Russo, and R.~Singal.
\newblock A finite time analysis of temporal difference learning with linear function approximation.
\newblock In {\em Conference On Learning Theory}, pages 1691--1692, 2018.

\bibitem{borkar:sa:2008}
Vivek~S Borkar.
\newblock {\em Stochastic Approximation: A Dynamical Systems Viewpoint}.
\newblock Cambridge University Press, 2008.

\bibitem{blm:2013}
Stéphane Boucheron, Gábor Lugosi, and Pascal Massart.
\newblock {\em Concentration inequalities: A nonasymptotic theory of independence}.
\newblock Oxford University Press, 2013.

\bibitem{chen2020aos}
Xi~Chen, Jason~D. Lee, Xin~T. Tong, and Yichen Zhang.
\newblock {Statistical inference for model parameters in stochastic gradient descent}.
\newblock {\em The Annals of Statistics}, 48(1):251 -- 273, 2020.

\bibitem{Chernozhukov2013}
Victor Chernozhukov, Denis Chetverikov, and Kengo Kato.
\newblock Gaussian approximations and multiplier bootstrap for maxima of sums of high-dimensional random vectors.
\newblock {\em Ann. Statist.}, 41(6):2786--2819, 2013.

\bibitem{Chernozhukov2015}
Victor Chernozhukov, Denis Chetverikov, and Kengo Kato.
\newblock Central limit theorems and bootstrap in high dimensions.
\newblock {\em Ann. Probab.}, 45(4):2309--2352, 2017.

\bibitem{Devroye2018}
Luc {Devroye}, Abbas {Mehrabian}, and Tommy {Reddad}.
\newblock {The total variation distance between high-dimensional Gaussians with the same mean}.
\newblock {\em arXiv e-prints}, page arXiv:1810.08693, October 2018.

\bibitem{durmus2022finite}
Alain Durmus, Eric Moulines, Alexey Naumov, and Sergey Samsonov.
\newblock Finite-time high-probability bounds for {P}olyak–{R}uppert averaged iterates of linear stochastic approximation.
\newblock {\em Mathematics of Operations Research}, 50(2):935--964, 2025.

\bibitem{durmus2021tight}
Alain Durmus, Eric Moulines, Alexey Naumov, Sergey Samsonov, Kevin Scaman, and Hoi-To Wai.
\newblock Tight high probability bounds for linear stochastic approximation with fixed stepsize.
\newblock In M.~Ranzato, A.~Beygelzimer, K.~Nguyen, P.~S. Liang, J.~W. Vaughan, and Y.~Dauphin, editors, {\em Advances in Neural Information Processing Systems}, volume~34, pages 30063--30074. Curran Associates, Inc., 2021.

\bibitem{durmus2021stability}
Alain Durmus, Eric Moulines, Alexey Naumov, Sergey Samsonov, and Hoi-To Wai.
\newblock On the stability of random matrix product with markovian noise: Application to linear stochastic approximation and td learning.
\newblock In Mikhail Belkin and Samory Kpotufe, editors, {\em Proceedings of Thirty Fourth Conference on Learning Theory}, volume 134 of {\em Proceedings of Machine Learning Research}, pages 1711--1752. PMLR, 15--19 Aug 2021.

\bibitem{efron1992bootstrap}
Bradley Efron.
\newblock Bootstrap methods: another look at the jackknife.
\newblock In {\em Breakthroughs in statistics: Methodology and distribution}, pages 569--593. Springer, 1992.

\bibitem{eweda:macchi:1983}
E.~Eweda and O.~Macchi.
\newblock Quadratic mean and almost-sure convergence of unbounded stochastic approximation algorithms with correlated observations.
\newblock {\em Ann. Inst. H. Poincar\'{e} Sect. B (N.S.)}, 19(3):235--255, 1983.

\bibitem{JMLR:v19:17-370}
Yixin Fang, Jinfeng Xu, and Lei Yang.
\newblock Online bootstrap confidence intervals for the stochastic gradient descent estimator.
\newblock {\em Journal of Machine Learning Research}, 19(78):1--21, 2018.

\bibitem{fort:clt:markov:2015}
{G. Fort}.
\newblock Central limit theorems for stochastic approximation with controlled {M}arkov chain dynamics.
\newblock {\em ESAIM: PS}, 19:60--80, 2015.

\bibitem{gaunt2023bounding}
Robert~E Gaunt and Siqi Li.
\newblock Bounding {K}olmogorov distances through {W}asserstein and related integral probability metrics.
\newblock {\em Journal of Mathematical Analysis and Applications}, 522(1):126985, 2023.

\bibitem{Bernolli2019}
Friedrich G\"{o}tze, Alexey Naumov, Vladimir Spokoiny, and Vladimir Ulyanov.
\newblock Large ball probabilities, {G}aussian comparison and anti-concentration.
\newblock {\em Bernoulli}, 25(4A):2538--2563, 2019.

\bibitem{guo1994stability}
L.~Guo.
\newblock Stability of recursive stochastic tracking algorithms.
\newblock {\em SIAM Journal on Control and Optimization}, 32(5):1195--1225, 1994.

\bibitem{huang2020matrix}
De~Huang, Jonathan Niles-Weed, Joel~A Tropp, and Rachel Ward.
\newblock Matrix concentration for products.
\newblock {\em Foundations of Computational Mathematics}, pages 1--33, 2021.

\bibitem{jirak2022quantitative}
Moritz Jirak and Martin Wahl.
\newblock Quantitative limit theorems and bootstrap approximations for empirical spectral projectors.
\newblock {\em Probability Theory and Related Fields}, 190(1):119--177, 2024.

\bibitem{kushner2003stochastic}
Harold Kushner and G~George Yin.
\newblock {\em Stochastic approximation and recursive algorithms and applications}, volume~35.
\newblock Springer Science \& Business Media, 2003.

\bibitem{lakshminarayanan2018linear}
C.~Lakshminarayanan and C.~Szepesvari.
\newblock Linear stochastic approximation: How far does constant step-size and iterate averaging go?
\newblock In {\em International Conference on Artificial Intelligence and Statistics}, pages 1347--1355, 2018.

\bibitem{li2024asymptotics}
Jiaqi Li, Johannes Schmidt-Hieber, and Wei~Biao Wu.
\newblock Asymptotics of {S}tochastic {G}radient {D}escent with {D}ropout {R}egularization in {L}inear {M}odels.
\newblock {\em arXiv preprint arXiv:2409.07434}, 2024.

\bibitem{pmlr-v178-li22b}
Xiang Li, Jiadong Liang, Xiangyu Chang, and Zhihua Zhang.
\newblock Statistical estimation and online inference via local sgd.
\newblock In Po-Ling Loh and Maxim Raginsky, editors, {\em Proceedings of Thirty Fifth Conference on Learning Theory}, volume 178 of {\em Proceedings of Machine Learning Research}, pages 1613--1661. PMLR, 02--05 Jul 2022.

\bibitem{liu2023statistical}
Ruiqi Liu, Xi~Chen, and Zuofeng Shang.
\newblock Statistical inference with {S}tochastic {G}radient {M}ethods under {$\phi$}-mixing {D}ata.
\newblock {\em arXiv preprint arXiv:2302.12717}, 2023.

\bibitem{mou2020linear}
Wenlong Mou, Chris~Junchi Li, Martin~J Wainwright, Peter~L Bartlett, and Michael~I Jordan.
\newblock On linear stochastic approximation: {F}ine-grained {P}olyak-{R}uppert and non-asymptotic concentration.
\newblock In {\em Conference on Learning Theory}, pages 2947--2997. PMLR, 2020.

\bibitem{mou2021optimal}
Wenlong Mou, Ashwin Pananjady, Martin~J Wainwright, and Peter~L Bartlett.
\newblock Optimal and instance-dependent guarantees for markovian linear stochastic approximation.
\newblock {\em Mathematical Statistics and Learning}, 7(1):41--153, 2024.

\bibitem{PTRF2019}
Alexey Naumov, Vladimir Spokoiny, and Vladimir Ulyanov.
\newblock Bootstrap confidence sets for spectral projectors of sample covariance.
\newblock {\em Probab. Theory Related Fields}, 174(3-4):1091--1132, 2019.

\bibitem{nemirovski2009robust}
Arkadi Nemirovski, Anatoli Juditsky, Guanghui Lan, and Alexander Shapiro.
\newblock Robust stochastic approximation approach to stochastic programming.
\newblock {\em SIAM Journal on optimization}, 19(4):1574--1609, 2009.

\bibitem{nourdin2022multivariate}
Ivan Nourdin, Giovanni Peccati, and Xiaochuan Yang.
\newblock Multivariate normal approximation on the wiener space: new bounds in the convex distance.
\newblock {\em Journal of Theoretical Probability}, 35(3):2020--2037, 2022.

\bibitem{osekowski:2012}
A.~Osekowski.
\newblock {\em Sharp Martingale and Semimartingale Inequalities}.
\newblock Monografie Matematyczne 72. Birkhäuser Basel, 1 edition, 2012.

\bibitem{pinelis_1994}
Iosif Pinelis.
\newblock {Optimum Bounds for the Distributions of Martingales in Banach Spaces}.
\newblock {\em The Annals of Probability}, 22(4):1679 -- 1706, 1994.

\bibitem{polyak1990new}
Boris~T Polyak.
\newblock New stochastic approximation type procedures.
\newblock {\em Automat. i Telemekh}, 7(98-107):2, 1990.

\bibitem{polyak1992acceleration}
Boris~T Polyak and Anatoli~B Juditsky.
\newblock Acceleration of stochastic approximation by averaging.
\newblock {\em SIAM journal on control and optimization}, 30(4):838--855, 1992.

\bibitem{rakhlin2012making}
Alexander Rakhlin, Ohad Shamir, and Karthik Sridharan.
\newblock Making gradient descent optimal for strongly convex stochastic optimization.
\newblock In {\em Proceedings of the 29th International Coference on International Conference on Machine Learning}, pages 1571--1578, 2012.

\bibitem{JASA2023}
Pratik Ramprasad, Yuantong Li, Zhuoran Yang, Zhaoran Wang, Will~Wei Sun, and Guang Cheng.
\newblock Online bootstrap inference for policy evaluation in reinforcement learning.
\newblock {\em J. Amer. Statist. Assoc.}, 118(544):2901--2914, 2023.

\bibitem{ruppert1988efficient}
David Ruppert.
\newblock Efficient estimations from a slowly convergent robbins-monro process.
\newblock Technical report, Cornell University Operations Research and Industrial Engineering, 1988.

\bibitem{samsonov2024gaussian}
Sergey Samsonov, Eric Moulines, Qi-Man Shao, Zhuo-Song Zhang, and Alexey Naumov.
\newblock Gaussian {A}pproximation and {M}ultiplier {B}ootstrap for {P}olyak-{R}uppert {A}veraged {L}inear {S}tochastic {A}pproximation with {A}pplications to {TD} {L}earning.
\newblock In A.~Globerson, L.~Mackey, D.~Belgrave, A.~Fan, U.~Paquet, J.~Tomczak, and C.~Zhang, editors, {\em Advances in Neural Information Processing Systems}, volume~37, pages 12408--12460. Curran Associates, Inc., 2024.

\bibitem{samsonov2025statistical}
Sergey Samsonov, Marina Sheshukova, Eric Moulines, and Alexey Naumov.
\newblock Statistical inference for {L}inear {S}tochastic {A}pproximation with {M}arkovian {N}oise.
\newblock {\em arXiv preprint arXiv:2505.19102}, 2025.

\bibitem{shao2003mathematical}
Jun Shao.
\newblock {\em Mathematical statistics}.
\newblock Springer Science \& Business Media, 2003.

\bibitem{shao2022berry}
Qi-Man Shao and Zhuo-Song Zhang.
\newblock Berry--{E}sseen bounds for multivariate nonlinear statistics with applications to {M}-estimators and stochastic gradient descent algorithms.
\newblock {\em Bernoulli}, 28(3):1548--1576, 2022.

\bibitem{sheshukova2025gaussian}
Marina Sheshukova, Sergey Samsonov, Denis Belomestny, Eric Moulines, Qi-Man Shao, Zhuo-Song Zhang, and Alexey Naumov.
\newblock Gaussian {A}pproximation and {M}ultiplier {B}ootstrap for {S}tochastic {G}radient {D}escent.
\newblock {\em arXiv preprint arXiv:2502.06719}, 2025.

\bibitem{spokoiny2015}
Vladimir Spokoiny and Mayya Zhilova.
\newblock {Bootstrap confidence sets under model misspecification}.
\newblock {\em The Annals of Statistics}, 43(6):2653 -- 2675, 2015.

\bibitem{srikant2024rates}
R~Srikant.
\newblock Rates of {C}onvergence in the {C}entral {L}imit {T}heorem for {M}arkov {C}hains, with an application to {TD} learning.
\newblock {\em arXiv preprint arXiv:2401.15719}, 2024.

\bibitem{srikant:1tsbounds:2019}
R.~{Srikant} and L.~{Ying}.
\newblock {Finite-Time Error Bounds For Linear Stochastic Approximation and TD Learning}.
\newblock In {\em Conference on Learning Theory}, 2019.

\bibitem{sutton1988learning}
R.~S Sutton.
\newblock Learning to predict by the methods of temporal differences.
\newblock {\em Machine learning}, 3(1):9--44, 1988.

\bibitem{sutton:book:2018}
R.~S. Sutton and Andrew~G. Barto.
\newblock {\em Reinforcement Learning: An Introduction}.
\newblock The MIT Press, second edition, 2018.

\bibitem{tropp2015introduction}
Joel~A Tropp et~al.
\newblock An introduction to matrix concentration inequalities.
\newblock {\em Foundations and Trends{\textregistered} in Machine Learning}, 8(1-2):1--230, 2015.

\bibitem{tsitsiklis:td:1997}
J.~N. {Tsitsiklis} and B.~{Van Roy}.
\newblock An analysis of temporal-difference learning with function approximation.
\newblock {\em IEEE Transactions on Automatic Control}, 42(5):674--690, May 1997.

\bibitem{zhu2023online_cov_matr}
Xi~Chen Wanrong~Zhu and Wei~Biao Wu.
\newblock Online {C}ovariance {M}atrix {E}stimation in {S}tochastic {G}radient {D}escent.
\newblock {\em Journal of the American Statistical Association}, 118(541):393--404, 2023.

\bibitem{wu2024statistical}
Weichen Wu, Gen Li, Yuting Wei, and Alessandro Rinaldo.
\newblock {S}tatistical {I}nference for {T}emporal {D}ifference {L}earning with {L}inear {F}unction {A}pproximation.
\newblock {\em arXiv preprint arXiv:2410.16106}, 2024.

\bibitem{wu2025markovchains}
Weichen Wu, Yuting Wei, and Alessandro Rinaldo.
\newblock Uncertainty quantification for {M}arkov chains with application to temporal difference learning.
\newblock {\em arXiv preprint arXiv:2502.13822}, 2025.

\bibitem{zolotarev1984probability}
Vladimir~Mikhailovich Zolotarev.
\newblock Probability metrics.
\newblock {\em Theory of Probability \& Its Applications}, 28(2):278--302, 1984.

\end{thebibliography}
\bibliographystyle{plain}
\newpage

\appendix
\begin{appendices}

\section{Definitions of integral probability metrics}
\label{appendix:proofs}
\label{sec:smooth_wasserstein_kolmogorov}
In this section we closely follow the exposition outlined in \cite{gaunt2023bounding}. Consider two $\rset^{d}$-valued random variables $X$ and $Y$. The integral probability metric \cite{zolotarev1984probability}, associated with a class of test functions $\H = \{h: \rset^{d} \to \rset : \PE[|h(X)|] < \infty, \PE[|h(Y)|] < \infty\}$, is defined as 
\begin{equation}
\label{eq:integral_prob_metrics_def}
\metricd[\H](X,Y) = \sup_{h \in \H}\bigl| \PE[h(X)] - \PE[h(Y)] \bigr|\eqsp.
\end{equation}
Different choices of the function class $\H$ induce different probability metrics. We consider the following important examples:
\begin{align} 
\H_{K} &= \{\indi{(-\infty,u_1] \times \cdots \times (-\infty,u_d]}, \quad u = (u_1,\ldots,u_d) \in \rset^{d}\} \\
\H_{Conv} &= \{\indi{B}, \quad B \in \Conv(\rset^{d})\} \\
\H_{W} &= \{h: \rset^{d} \to \rset, \quad \norm{h}[\operatorname{Lip}] \leq 1\} \\
\H_{[m]} &= \{h: \rset^{d} \to \rset, \quad h \in C^{m-1}(\rset^{d}) \text{ with } |h|_{j} \leq 1\text{ for } 1 \leq j \leq m\}\eqsp,
\end{align}
where $\Conv(\rset^{d})$ denotes the collection of all convex subsets of $\rset^{d}$, $\norm{h}[\operatorname{Lip}] = \sup_{x \neq y}\frac{\norm{h(x)-h(y)}}{\norm{x-y}}$ is the Lipschitz constant, $C^{m-1}(\rset^{d})$ represents the space of $(m-1)$-times continuously differentiable functions, and the seminorm $|h|_{j}$ is defined as
\[
|h|_{j} = \max_{i_1,\ldots,i_j \in \{1,\ldots,d\}} 
 \norm{\frac{\partial^{j} h}{\partial u_{i_1} \cdots \partial u_{i_j}}}[\infty]\eqsp.
\]
Thus, for each $m \in \nset$, the function class $\H_{[m]}$ consists of functions whose partial derivatives up to order $m$ are uniformly bounded.

These function classes generate well-established probability metrics in the literature. The class $\H_{K}$ induces the classical Kolmogorov metric between distributions \cite{zolotarev1984probability}, while $\H_{Conv}$ generates the convex distance $\kolmogorov$ defined for a pair of probability measures $\mu, \nu$ on $\rset^{d}$ as
\begin{equation}
\label{eq:berry-esseen}
\kolmogorov(\mu, \nu) = \sup_{B \in \Conv(\rset^{d})}\left|\mu(B) - \nu(B)\right|\,,
\end{equation} which is the primary focus of this paper. The class $\H_{W}$ yields the  Wasserstein-1 distance, and the classes $\H_{[m]}$ define the smoothed Wasserstein metrics of order $m$. We denote the corresponding metrics as $\metricd[K]$, $\kolmogorov$, $\metricd[W]$, and $\metricd[{[m]}]$, respectively. 

An important hierarchy exists among these metrics: for any pair of random vectors $X$ and $Y$, we have
\[
\metricd[K](X,Y) \leq \kolmogorov(X,Y),
\] 
since every rectangular set is convex, implying $\H_K \subset \H_{Conv}$.Other relations among these metrics are substantially more intricate. For instance, when $Y$ is a multivariate normal vector, it is well-established (see, e.g., \cite{nourdin2022multivariate}) that
\begin{align}
\kolmogorov(X,Y) \leq C \sqrt{\metricd[W](X,Y)}\eqsp,
\end{align}
where the constant $C$ depends explicitly on the covariance matrix of the vector $Y$. This inequality serves as the theoretical basis for comparing the bounds provided in \Cref{th:shao2022_berry} with the results obtained in \cite{srikant2024rates}. 

\section{Proofs of \Cref{sec:moment_bounds_lsa_pr}}
\label{supplement:moment_bounds}
\begin{proof}[Proof of \Cref{lem:matr_product_as_bound}]
Note that \Cref{assum:step-size}($p \vee \log d$) implies for all $\ell$ that
\begin{align}
    1 - a\alpha_\ell + (p-1) b_Q^2 \alpha_\ell^2 \leq 1 - \frac{a\alpha_\ell}{2} \eqsp. 
\end{align}
To finish the proof we combine the latter inequality with \Cref{cor:norm_Gamma_m_n}.
\end{proof}

\begin{proof}[Proof of \Cref{lem:last_moment_bound}.]

Using the decomposition \eqref{eq:lsa_error}, we obtain that, with $p \geq 2$, it holds
\begin{equation}
\label{eq:p-norm-minkowski}
\PE^{1/p}[\norm{\theta_k - \thetas}^p] \leq \PE^{1/p}[\norm{\ProdB_{1:k} \{ \theta_0 - \thetas \}}^{p}] + \PE^{1/p}[\norm{\sum_{j=1}^k \alpha_{j} \ProdB_{j+1:k} \funnoisew_j}^{p}]\eqsp,
\end{equation}
and we bound both terms separately. Applying \Cref{lem:matr_product_as_bound}, we get for $2 \leq p \leq \log{(5n^2)}$:
\begin{equation}
\label{eq:transient_term_bound}
\PE^{1/p}[\norm{\ProdB_{1:k} \{ \theta_0 - \thetas \}}^{p}] \leq \sqrt{\qcond} \rme \norm{\theta_0 - \thetas} \prod_{\ell=1}^k (1 - \frac{a}{2} \alpha_\ell) \eqsp.
\end{equation}
Now we proceed with the second term in \eqref{eq:p-norm-minkowski}. Applying Burholder's inequality \cite[Theorem 8.6]{osekowski:2012}, we obtain that 
\begin{align}
\PE^{1/p}[\norm{\sum\nolimits_{j=1}^k \alpha_{j} \ProdB_{j+1:k} \funnoisew_j}^{p}] &\leq p \left(\PE^{2/p}\left[\left(\sum\nolimits_{j=1}^{k}\alpha_{j}^2 \normop{\ProdB_{j+1:k} \funnoisew_j}^{2}\right)^{p/2}\right]\right)^{1/2}  \leq  p \left(\sum\nolimits_{j=1}^{k}\alpha_{j}^{2}\PE^{2/p}\bigl[\normop{\ProdB_{j+1:k} \funnoisew_j}^{p} \bigr]\right)^{1/2} \\
& \leq p \sqrt{\qcond} \rme \supconsteps \biggl(\sum\nolimits_{j=1}^{k} \alpha_{j}^{2} \prod_{\ell=j+1}^{k}\bigl(1 - \frac{a \alpha_{\ell}}{2}\bigr) \biggr)^{1/2}  \overset{(a)}{\leq} \ConstC_{\ref{lem:last_moment_bound}} p \sqrt{\alpha_{k}}\eqsp,
\end{align}
where in \textit{(a)} we additionally applied \Cref{lem:summ_alpha_k}-\ref{lem:summ_alpha_k_p_item}.
\end{proof}

\begin{proof}[Proof of \Cref{lem:Jk_ell:p_moment}.]
First, we derive a bound for $\Jnalpha{k}{0} = -\sum_{\ell=1}^{k} \alpha_\ell G_{\ell+1:k} \eps_\ell$ which is a sum of independent random vectors, satisfying $\norm{\alpha_\ell G_{\ell+1:k} \eps_\ell} \leq \alpha_{\ell} \qcond^{1/2} \prod_{j=\ell+1}^{k}(1-\alpha_j a)^{1/2} \supconsteps$. Hence, applying the  Pinelis inequality \cite[Theorem~3.5]{pinelis_1994}, we obtain that, for any $t \geq 0$, 
\[
\PP(\norm{\Jnalpha{k}{0}} \geq t) \leq 2 \exp\biggl( -\frac{t^2}{2\sigma_k^2}\biggr)\eqsp, \quad \text{ where } \sigma_k^2 = \qcond \supconsteps^2 \sum_{\ell=1}^{k} \alpha_{\ell}^2 \prod_{j=\ell+1}^{k} (1-\alpha_j a) \leq \alpha_{k} c_{1}\eqsp, 
\]
and  $c_1 = 24 \qcond \supconsteps^2/a$. Thus, applying \cite[Lemma~7]{durmus2022finite}, we obtain that, for $p \geq 2$, it holds 
\[
\expep{\norm{\Jnalpha{k}{0}}}{p} \leq 2^{1/p} \sqrt{p} \sqrt{\alpha_k} \sqrt{c_1}\eqsp,
\]
and the bound for $\Jnalpha{k}{0}$ follows. Now we bound $\Jnalpha{k}{\ell}$ by induction. Using the equation \eqref{eq:jn_allexpansion_main}, $\Jnalpha{k}{\ell}$, $\ell \geq 1$, can be represented as 
\begin{equation}
\Jnalpha{k}{\ell} = -\sum_{m=1}^k \alpha_m G_{m+1:k} \zmfuncAw[m] \Jnalpha{m-1}{\ell-1}\eqsp.
\end{equation}
Note that $\alpha_m G_{m+1:k} \Jnalpha{m-1}{\ell-1}$ is a martingale-difference sequence w.r.t. the filtration $\mathcal{F}_m = \sigma(Z_s:1\leq s\leq m)$. Hence, Burkholder's inequality \cite[Theorem 8.6]{osekowski:2012} implies that 
\begin{align}
\PE^{1/p}[\norm{\Jnalpha{k}{\ell}}^p] 
\leq p \biggl(\sum_{m=1}^k \PE^{2/p}[\normop{\alpha_m G_{m+1:k} \zmfuncAw[m]  \Jnalpha{m-1}{\ell-1}}^p]\biggr)^{1/2} &\leq  \bConst{A} p^{\ell+1/2} \ConstJ{\ell-1} \bigl(\sum_{m=1}^k \alpha_m^{\ell+2} \normop{G_{m+1:k}}^2 \bigr)^{1/2} \\
&\overset{(a)}{\leq}  \bConst{A} \ConstJ{\ell-1} \cdot \frac{2 \sqrt{6} \sqrt{\qcond}}{\sqrt{a}}  p^{\ell+1/2} \alpha_{k}^{(\ell+1)/2}  \eqsp,
\end{align}
where in (a) we used \Cref{lem:summ_alpha_k}. Now we prove the bound for $\Hnalpha{k}{\ell}$. Recall that
$\Hnalpha{k}{\ell} = -\sum_{m=1}^k \alpha_m \Gamma_{m+1:k} \Jnalpha{m}{\ell}$. 
Since $\Gamma_{m+1:k}$ and $\Jnalpha{m}{\ell}$ are independent for all $m$, the desired result follows from \Cref{lem:Jk_ell:p_moment}, \Cref{lem:summ_alpha_k}, and Minkowski's inequality:
\begin{align}
\PE^{1/2}[\norm{\Hnalpha{k}{\ell}}^p] 
 \leq \sum_{m=1}^{k} \alpha_m \PE^{1/p}[\norm{\Gamma_{m+1:k}}^p] \PE^{1/p}[\norm{\Jnalpha{m}{\ell}}^p] &\overset{(a)}{\leq} \sqrt{\qcond} \rme \ConstJ{\ell} p^{\ell+1/2} \sum_{m=1}^{k} \alpha_m^{(\ell+3)/2} \prod_{\ell=m}^k (1-\frac{a\alpha_\ell}{2}) \\
&\overset{(b)}{\leq} \ConstH{\ell}  p^{\ell+1/2} \alpha_{k}^{(\ell+1)/2}\eqsp,
\end{align}
where in $(a)$ we used the moment bound for $\Jnalpha{m}{\ell}$, and in $(b)$ we used \Cref{lem:summ_alpha_k}. 
\end{proof}

\begin{proof}[Proof of \Cref{prop:Qell:bound}.]
Using the triangle inequality we get:
    \begin{align}
        \label{prop:qell_bound:first}
        \normop{Q_\ell} &\leq \alpha_\ell \sum_{k=\ell}^{n-1} \normop{G_{\ell + 1:k}} \leq  \sqrt{\qcond} \sum_{k=\ell}^{n-1} \alpha_{\ell} \prod_{j=\ell+1}^k (1-\frac{a c_0}{2} j^{-\gamma}) 
    \end{align}
    The rest of the proof follows from \Cref{lem:summ_alpha_k}. 
\end{proof}

\section{Proofs of \Cref{sec:clt_lsa_pr}}
\label{supplement:normal_approximation}
\begin{proof}[Proof of \Cref{lem:sum_Gamma_1:k}.]
Applying Minkowski's inequality and \Cref{cor:norm_Gamma_m_n}, 
\begin{align}
\PE^{1/p}\bigl[\norm{\sum_{k=0}^{n-1} \Gamma_{1:k}}^p\bigr] \leq \sum_{k=0}^{n-1} \expep{\normop{\Gamma_{1:k}}}{p} \leq 1 +  \sqrt{\qcond} \rme \,  \sum_{k=1}^{n-1} \prod_{\ell=1}^k (1 - \frac{a\alpha_\ell}{2}) \eqsp.
\end{align}
Thus, applying \Cref{lem:summ_alpha_k}, 
$\sum_{k=1}^{n-1} \prod_{\ell=1}^k (1 - \frac{a\alpha_\ell}{2}) \leq   \frac{(1+k_0)^\gamma}{c_0} (c_0 + \frac{2}{a(1-\gamma)})$, and the statement follows.
\end{proof}

\begin{proof}[Proof of \Cref{lem:sigma_n_labmda_min_lowerbound}.]
Decomposing $\Sigma_n = \Sigma_{\infty} + (\Sigma_n - \Sigma_\infty)$ and then applying Lidskii's inequality, we obtain
$\lambda_{\min}(\Sigma_n) \geq \lambda_{\min}(\Sigma_{\infty}) - \normop{\Sigma_n - \Sigma_{\infty}}$.
The conclusion follows from \Cref{lem:sigma_n_bound} and \Cref{assum:step-size}, which imply $\normop{\Sigma_n - \Sigma_{\infty}} \leq \ConstC_{\ref{lem:sigma_n_bound}} n^{\gamma-1}  \leq \frac{\lambda_{\min}(\Sigma_{\infty})}{2}$.
\end{proof}

\begin{proof}[Proof of \Cref{lem:gamma_copies_difference}.]
First, we rewrite the sum:
\begin{align}
\sum_{k=1}^{n-1} (\Gamma_{1:k}-\Gamma_{1:k}^{(i)}) &= \sum_{k=i}^{n-1} \alpha_i \Gamma_{1:i-1} (\funcAw(Z_i) - \funcAw(Z_i')) \Gamma_{i+1:k} =  \Gamma_{1:i-1} (\funcAw(Z_i) - \funcAw(Z_i')) \sum_{k=i}^{n-1} \alpha_i \Gamma_{i+1:k}\eqsp.
\end{align}
Hence, it holds that
\begin{align}
\normop{\sum_{k=1}^{n-1} (\Gamma_{1:k}-\Gamma_{1:k}^{(i)})} \leq \bConst{A} \normop{\Gamma_{1:i-1}} \normop{\sum_{k=i}^{n-1} \alpha_{i} \Gamma_{i+1:k}}\eqsp.
\end{align}
\Cref{lem:matr_product_as_bound} implies that
$
\PE^{1/2}\bigl[\norm{\Gamma_{1:i-1}}^2\bigr] \leq \sqrt{\qcond} \rme \prod_{m=1}^{i-1} (1-a\alpha_m/2)
$.
On the other hand, combining Minkowski's inequality with \Cref{lem:summ_alpha_k}, we obtain
\begin{align}
\PE^{1/2}\bigl[\normop{\sum_{k=i}^{n-1} \alpha_i \Gamma_{i+1:k}}^{2}\bigr] \leq \sqrt{\qcond} \rme  \sum_{k=i}^{n-1} \alpha_i \prod_{m=i+1}^{k} (1-a\alpha_j/2) \leq \sqrt{\qcond} \rme (c_0 + \frac{2}{a(1-\gamma)})\eqsp.
\end{align}
To finish the proof, it remains to notice that $\Gamma_{1:i-1}$ is independent from $\Gamma_{i+1:k}$.
\end{proof}

\begin{proof}[Proof of \Cref{lem:Hk:copy_difference}.]
First, note that $\Hnalpha{k}{0} - \Hnalpha{k}{0, i} = 0, \text{ if } k < i$. On the other hand, for $k \geq i$ we get
\begin{align}
    \Hnalpha{k}{0} - \Hnalpha{k}{0, i} = \underbrace{\Gamma_{i+1:k} (\Hnalpha{i}{0} - \Hnalpha{i}{0, i})}_{T_1^{(k)}} - \underbrace{\sum_{j=i+1}^{k} \alpha_j \Gamma_{j+1:k} \zmfuncAw[j] (\Jnalpha{j-1}{0} - \Jnalpha{j-1}{0, i})}_{T_2^{(k)}} \eqsp.
\end{align}
Introduce $\eps_i' = \eps(Z_i')$ and $\funcAw_i' = \funcAw(Z_i')$. Then, for $\ell \geq i+1$: $\Jnalpha{\ell-1}{0} - \Jnalpha{\ell-1}{0, i} = \alpha_i G_{i+1:\ell-1} (\eps_i' - \eps_i)$. Thus, since $T_2^{(i)} = 0$, we obtain that
\begin{align}
    \sum_{k=i}^{n-1} T_2^{(k)} = \sum_{k=i+1}^{n-1} \sum_{j=i+1}^{k} \alpha_j \Gamma_{j+1:k} \zmfuncAw[k] \alpha_i G_{i+1:k-1} (\eps_i' - \eps_i) =  \sum_{j=i+1}^{n-1} \underbrace{\alpha_i \bigl(\sum_{k=j}^{n-1} \alpha_j \Gamma_{j+1:k}\bigr) \zmfuncAw[j] G_{i+1:j-1} (\eps_i' - \eps_i)}_{U_j} \eqsp.
\end{align}
Note that $U_j$ is a reverse martingale-difference sequence with respect to the filtration 
$\mathcal{F}_{j,i} = \sigma(Z_i, Z_i', Z_j, Z_{j+1}, \ldots, Z_{n-1})$. Hence, 
\begin{align}
    \PE^{1/2}\bigl[\norm{n^{-1/2}\sum_{\ell=i+1}^{n-1} U_j}^2\bigr] = n^{-1/2} \bigl(\sum_{j=i+1}^{n-1} \expe{\norm{U_j}^2}\bigr)^{1/2} \eqsp.
\end{align}
For simplicity we set $u_{\ell:m} = \prod_{t=\ell}^{m} (1-a\alpha_t/2)$. Applying \Cref{cor:norm_Gamma_m_n} and \Cref{lem:summ_alpha_k}, we obtain
\begin{align}
    \expe{\norm{U_j}^2} \leq \alpha_i^2 \biggl(\qcond \rme^2  \supconsteps \bConst{A} \biggr)^2 \left(\sum_{k=j}^{n-1} \alpha_j u_{j+1:k} \right)^2  u_{i+1:j-1}^2 \leq \alpha_i^2  \underbrace{\biggl(\qcond \rme^2  \supconsteps \bConst{A} (c_0 + \frac{2}{a(1-\gamma)}) \biggr)^2}_{R_U^2}u_{i+1:j-1} \eqsp.
\end{align}
Thus, it holds that
\begin{align}
    \PE^{1/2}\bigl[\norm{n^{-1/2}\sum_{\ell=i+1}^{n-1} U_j}^2\bigr] \leq \frac{1}{n^{1/2}} \sqrt{\alpha_i} R_U \bigl(\alpha_i \sum_{j=i}^{n-2} u_{i+1:j}\bigr)^{1/2} \leq \frac{R_U \sqrt{c_0 + \frac{2}{a(1-\gamma)}}}{n^{1/2}} \sqrt{\alpha_i} \eqsp.
\end{align}
The recurrent rule \eqref{eq:hn0_main} implies the following representation for $T_1^{(k)}$:
\begin{align}
T_1^{(k)} = \Gamma_{i+1:k}(-\alpha_i(\funcAw_i 
- \funcAw'_i)\Hnalpha{i-1}{0} - \alpha_i (\funcAw_i - \funcAw'_i) \Jnalpha{i-1}{0})\eqsp.
\end{align}
Therefore, \Cref{lem:Jk_ell:p_moment} together with $\alpha_{i-1} \leq 2 \alpha_i$ implies that 
\begin{align}
\PE^{1/2}\bigl[\norm{T_1^{(k)}}^2\bigr] \leq 2\bConst{A} \sqrt{\qcond} \rme u_{i+1:k} \alpha_i^{3/2} (\ConstJ{0} + \ConstH{0}) \eqsp.
\end{align}
Thus, using \Cref{lem:summ_alpha_k}, we get
\begin{align}
\PE^{1/2}\bigl[\norm{\sum_{k=i}^{n-1} T_1^{(k)}}^2\bigr] &\leq 2\bConst{A} \sqrt{\qcond} \rme  u_{i+1:k} \alpha_i^{1/2} (\ConstJ{0} + \ConstH{0}) \sum_{k=i}^{n-1} \alpha_i u_{i+1:k} \\
&\leq \alpha_i^{1/2} 2\bConst{A} \sqrt{\qcond} \rme  (\ConstJ{0} + \ConstH{0}) (c_0 + \frac{2}{a(1-\gamma)}) \eqsp.
\end{align}
It remains to note that 
\[
\sum_{k=1}^{n-1} (\Hnalpha{k}{0}-\Hnalpha{k}{0, i}) = \sum_{k=1}^{n-1} (T_1^{(k)} + T_2^{(k)}) = \sum_{k=i}^{n-1} (T_1^{(k)} + T_2^{(k)})\eqsp,
\]
and use Minkowski's inequality.
\end{proof}

\section{Proofs of \Cref{sec:bootstrap}}
\label{supplement:bootstrap}
\begin{proof}[Proof of \Cref{lem:tilde_eps_boot_bound}.]
Writing,  $\tilde \eps_\ell = \eps_\ell + \funcAw_\ell (\theta_{\ell-1} - \thetas)$,
the proof follows from the definition of $\Omega_1$ and \Cref{assum:step-size-bootstrap}.
\end{proof}

\begin{proof}[Proof of \Cref{lemma:EbJk1:bound}.]
Applying \Cref{lem:gamma_deviation_bound}, we get
\begin{align}
    \mathbb{E}^\boot \norm{\frac{1}{\sqrt{n}} \sum_{k=1}^{n-1} J_{k,1}^{\boot, 0}}^2 &=
    n^{-1} \mathbb{E}^\boot \norm{\sum_{\ell=1}^{n-1} \alpha_\ell (w_\ell - 1) \sum_{k=\ell}^{n-1} (\Gamma_{\ell+1:k} - G_{\ell+1:k}) \eps_\ell}^2 = n^{-1} \sum_{\ell=1}^{n-1} \norm{\alpha_\ell \sum_{k=\ell}^{n-1} (\Gamma_{\ell+1:k} - G_{\ell+1:k}) \eps_\ell}^2 \\
    &\leq n^{-1} \log^2(5n) \ConstC_{\ref{lem:gamma_deviation_bound}}^2 \supconsteps^2 \sum_{\ell=1}^{n-1} \alpha_\ell \leq \log^2(5n) \frac{c_0 \ConstC_{\ref{lem:gamma_deviation_bound}}^2 \supconsteps^2}{(1-\gamma) n^{\gamma}}\eqsp.
\end{align}
\end{proof}

\begin{proof}[Proof of \Cref{lemma:EbJk2:bound}.]
First we rewrite the sum and obtain
    \begin{align}
        \label{lem:EbJk2:first_bound}
        \mathbb{E}^\boot \norm{\frac{1}{\sqrt{n}} \sum_{k=1}^{n-1} J_{k,2}^{\boot, 0}}^2 &= \frac{1}{n} \mathbb{E}^\boot \norm{\sum_{k=1}^{n-1} \sum_{\ell=1}^{k} \alpha_\ell (w_\ell -1) \Gamma_{\ell+1:k} \funcAw_\ell (\theta_{\ell-1} - \thetas)}^2 \\
        &= \frac{1}{n} \sum_{\ell=1}^{n-1} \alpha_\ell^2 \norm{\sum_{k=\ell}^{n-1} \Gamma_{\ell+1:k} \funcAw_\ell (\theta_{\ell-1} - \thetas)}^2 \leq \frac{1}{n} \bConst{A}^2 \sum_{\ell=1}^{n-1} \alpha_\ell^2 \norm{\sum_{k=\ell}^{n-1} \Gamma_{\ell+1:k}}^2 \norm{\theta_{\ell-1} - \thetas}^2\eqsp. 
    \end{align}
Hence, using \Cref{lem:matr_product_as_bound} and \Cref{lem:summ_alpha_k} with $b = a/2$, we get that on the event $\Omega_2$, it holds that
\begin{align}
\mathbb{E}^\boot \norm{\frac{1}{\sqrt{n}} \sum_{k=1}^{n-1} J_{k,2}^{\boot, 0}}^2 
&\leq n^{-1} \bConst{A}^2 \ConstC_{\ref{lem:matr_product_as_bound}}^2 \sum_{\ell=1}^{n-1} \bigl( \alpha_\ell \sum_{k=\ell}^{n-1} \prod_{j=\ell+1}^{k}  (1 - \frac{a\alpha_j}{2}) \bigr)^2 \norm{\theta_{\ell-1} - \thetas}^2 \\
&\leq n^{-1} \bConst{A}^2 \ConstC_{\ref{lem:matr_product_as_bound}}^2  \bigl(c_0 + \frac{2}{a(1-\gamma)}\bigr)^2\eqsp  \sum_{\ell=1}^{n-1}\norm{\theta_{\ell-1} - \thetas}^2\eqsp. 
\end{align}
Using an elementary inequality $(a+b)^2 \leq 2a^2 + 2b^2$, we get that on the event $\Omega_1$, it holds
\begin{align}
\sum_{\ell=1}^{n-1}\norm{\theta_{\ell-1} - \thetas}^2 
&\leq \sum_{\ell=1}^{n-1} \bigl\{2 \qcond \rme^4 \norm{\theta_0 - \thetas}^2 \prod_{j=1}^{\ell-1} (1 - \frac{a\alpha_j}{2}) + 8 \rme^2 \log^2(5n) \ConstC_{\ref{lem:last_moment_bound}}^2 \alpha_{\ell-1}\bigr\} \\
&\leq \frac{2 \qcond \rme^4 (1 + k_0)^{\gamma} \bigl(c_0 + \frac{2}{a(1-\gamma)}\bigr)}{c_0} \norm{\theta_0-\thetas}^2 + 8 \rme^2 \log^2(5n) \ConstC_{\ref{lem:last_moment_bound}}^2 c_0 \frac{n^{1-\gamma}}{1-\gamma}\eqsp. 
\end{align}
It remains to combine the above bounds.
\end{proof}

\begin{proof}[Proof of \Cref{lem:expansion}.]
We start from the decomposition 
\begin{equation}
\label{eq:one_step_expand_supplement}
\theta_{k}^\boot - \theta_{k} = (\Id - \alpha_{k} w_k \funcAw_k)(\theta_{k-1}^\boot - \theta_{k-1})  - \alpha_{k} (w_k-1) \tilde \funnoisew_{k}.
\end{equation}
Expanding the recurrence above till $k = 0$, and using the fact that $\theta_{0}^\boot = \theta_{0}$, we get running the recurrence \eqref{eq:one_step_expand_supplement}, that 
\[
\theta_{k}^\boot - \theta_{k} = -\sum_{\ell=n+1}^k \alpha_\ell (w_\ell - 1)\ProdB^{\boot}_{\ell+1:k} \tilde \funnoisew_{\ell}\eqsp.
\]
Hence, proceeding as in \eqref{eq:jn0_main}, we obtain the representation
\begin{align}
\label{eq:jn0_bootstrap}
&\Jnalpha{k}{\boot,0} =\left(\Id - \alpha_{k} \funcAw_k \right) \Jnalpha{k-1}{\boot,0} - \alpha_{k}(w_k - 1)\tilde{\funnoisew_{k}}\eqsp, && \Jnalpha{0}{\boot,0}=0\eqsp, \\[.1cm]
\label{eq:hn0_bootstrap}
&\Hnalpha{k}{\boot,0} =\left( \Id - \alpha_{k} 
 w_{k} \funcAw_k \right) \Hnalpha{k-1}{\boot,0} - \alpha_{k} (w_k - 1) \funcAw_k \Jnalpha{k-1}{\boot,0}\eqsp, && \Hnalpha{0}{\boot,0}=0\eqsp.
\end{align}
Hence, using \Cref{lem:tilde_eps_boot_bound} together with the definition of $J_k^{\boot, 0}$, we obtain that
\begin{align}
\PEb[\|J_k^{\boot, 0}\|^2 ] & = \sum_{\ell=1}^k \alpha_\ell^2 \| \ProdB_{\ell+1:k} \tilde \funnoisew_{\ell}\|^2 \leq \alpha_k \ConstC_{\ref{lem:tilde_eps_boot_bound}}^2 \ConstC_{\ref{lem:matr_product_as_bound}}^2 \sum_{\ell=1}^k \alpha_\ell^2 \prod_{j=\ell+1}^k (1-a\alpha_\ell/2) \leq \alpha_k \underbrace{\frac{12 \ConstC_{\ref{lem:tilde_eps_boot_bound}}^2 \ConstC_{\ref{lem:matr_product_as_bound}}^2}{a}}_{(\ConstJb{0}{1})^2} \eqsp.
\end{align}
Assume now that the bound on $J_k^{\boot, j-1}$ has a form $
\{\PEb[\|J_k^{\boot,j-1}\|^2 ]\}^{1/2} \leq \ConstJb{j-1}{1} \alpha_k^{j/2}$. Then, using the martingale property of $J_k^{\boot, j}$, we write that
\begin{align}
\PEb[\|J_k^{\boot, j}\|^2 ] 
&= \sum_{\ell=1}^k \alpha_\ell^2 \PEb[\| \ProdB_{\ell+1:k} \funcAw_\ell J_{\ell-1}^{\boot, j-1}\|^2] \leq (\ConstJb{j-1}{1})^2 \sum_{\ell=1}^k \alpha_{\ell}^{j+2} \bConst{A}^2 \ConstC_{\ref{lem:matr_product_as_bound}}^2 \prod_{t=\ell+1}^{k} (1-a\alpha_t/2)^2 \eqsp.
\end{align}
Hence, applying \Cref{lem:summ_alpha_k} we get
\begin{align}
    \PEb[\|J_k^{\boot, j}\|^2 ] &\leq \alpha_{k}^{j+1} \underbrace{(\ConstJb{j-1}{1})^2 \ConstC_{\ref{lem:matr_product_as_bound}}^2 \bConst{A}^2 \frac{12}{a}}_{(\ConstJb{j}{1})^2} \qquad.
\end{align}
and, thus, the moment bound for $J_k^{\boot, j}$ is proved. Moreover, using the definition of $H_k^{\boot, L}$ and Minkowski's inequality, we obtain that 
\begin{align}
(\PEb[\|H_k^{\boot, L} \|^2] )^{1/2} 
\leq \bConst{A} \sum_{\ell=1}^k \alpha _\ell (\PEb[\|\ProdB^{\boot}_{\ell+1:k} \|^2] )^{1/2}  (\PEb[\|J_{\ell-1}^{\boot, L} \|^2] )^{1/2} &\leq \bConst{A} \ConstJb{L}{1} \ConstC_{\ref{prop:product_random_matrix_bootstrap}} \sum_{\ell=1}^k \alpha_{\ell}^{\frac{L+3}{2}} \prod_{t=\ell+1}^{k} (1-\frac{a\alpha_t}{8}) \\
&\leq \alpha_\ell^{(L+1)/2} \underbrace{\bConst{A} \ConstJb{L}{1} \ConstC_{\ref{prop:product_random_matrix_bootstrap}} \frac{48}{a}}_{\ConstHb{L}{1}} \eqsp.
\end{align}
and the moment bound for $H_k^{\boot, L}$ follows.
\end{proof}

\begin{proof}[Proof of \Cref{lem:gamma_deviation_bound}.]
For any matrix-valued sequences $(U_n)_{n\in \nset}$, $(V_n)_{n\in \nset}$ and any $M \in \nset$, it holds that:
\begin{equation}
\label{eq:matrix_products_identity}
\prod_{k=1}^M U_k - \prod_{k=1}^M V_k = \sum_{k=1}^M \biggl(\prod_{j=k+1}^M V_j \biggr) (U_k - V_k) \biggl(\prod_{j=1}^{k-1} U_j\biggr)\eqsp.
\end{equation}
Using \eqref{eq:matrix_products_identity} and changing the order of summation, we get
\begin{align}
\label{eq:matr_gamma_bgamma_difference}
\alpha_\ell \sum_{k=\ell}^{n-1} (\Gamma_{\ell+1:k} - G_{\ell+1:k}) \eps_\ell = \alpha_{\ell} \sum_{j=\ell+1}^{n-1} \underbrace{\bigl(\alpha_j \sum_{k=j}^{n-1} G_{j+1:k}\bigr) (\funcAw_j - \bA) \Gamma_{\ell+1:j-1} \eps_\ell}_{U_j}\eqsp.
\end{align}
    Applying \Cref{prop:Qell:bound}, we get $\norm{\alpha_j \sum_{k=j}^{n-1} G_{j+1:k}} = \norm{Q_j} \leq \ConstC_{\ref{prop:Qell:bound}}$, hence, $\normop{U_j} \leq 2\ConstC_{\ref{prop:Qell:bound}} \bConst{A} \normop{\Gamma_{\ell+1:j-1}}$. Consider the sigma-algebras
    \begin{equation}
        \mathcal{F}_{m:k} = \begin{cases}
            \sigma(Z_s: m \leq s \leq k), \qquad &\text{if } m \leq k\eqsp, \\
            \{\emptyset, \msz\}, &\text{otherwise.}
        \end{cases}
    \end{equation}
    Note that $U_j$ is a martingale-difference sequence w.r.t. the filtration
    $\mathcal{F}_{\ell+1:\ell+1} \subseteq \mathcal{F}_{\ell+1:\ell+2} \subseteq \ldots \subseteq \mathcal{F}_{\ell+1:2n}$, thus Burkholder's inequality \cite[Theorem 8.6]{osekowski:2012} implies 
    \[
    \PE^{1/p}\bigl[\norm{\sum_{j=\ell+1}^{n-1} U_j}^p\bigr]  \leq p   \bigl(\sum_{j=\ell+1}^{n-1} \PE^{2/p} \norm{U_j}^p \bigr)^{1/2} \leq 2p \sqrt{d} \ConstC_\eps  \bConst{A} \ConstC_{\ref{prop:Qell:bound}} \bigl(\sum_{j=\ell+1}^{n-1} \PE^{2/p} \norm{ \Gamma_{\ell+1:j-1}}^p \bigr)^{1/2}\eqsp.
    \]
    Applying now \Cref{lem:matr_product_as_bound} together with the fact $\alpha_{\infty} a \leq 1/2$, we get
    \begin{align}
    \alpha_\ell \sum_{j=\ell+1}^{n-1}   \PE^{2/p} \norm{ \Gamma_{\ell+1:j-1}}^p \leq \sum_{j=\ell+1}^{n-1} \alpha_\ell \qcond \rme^2 \prod_{t=\ell+1}^{j-1} (1 - \frac{a \alpha_t}{2}) &\leq (8/7) \qcond \rme^2 \sum_{j=\ell}^{n-1} \alpha_\ell \prod_{t=\ell+1}^j (1-\frac{a\alpha_m}{2}) \\
    & \overset{(a)}{\leq} (8/7) \qcond \rme^2 \left(c_0 + \frac{2}{a(1-\gamma)}\right)\eqsp.
    \end{align}
    In (a) we additionally used \Cref{lem:summ_alpha_k} with $b = a/2$. It remains to combine the above bounds in \eqref{eq:matr_gamma_bgamma_difference}. To conclude the proof, we need to apply \Cref{lem:markov_inequality} with $p = \log(5n^2)$.
\end{proof}

\begin{proof}[Proof of \Cref{lem:jkb1:bound}.]
First we rewrite the expression using the recurrent formula for $J_k^{\boot, 1}$ proven in \Cref{lem:expansion} and swapping the order of summation:
    \begin{align}
        \label{lem:Jkb1_bound_primary}
        \PE^\boot\bigl[\norm{\frac{1}{\sqrt{n}} \sum_{k=1}^{n-1} J_k^{\boot, 1}}^2\bigr] &= \frac{1}{n} \PE^\boot\bigl[\norm{\sum_{k=1}^{n-1} \sum_{\ell=1}^k \alpha_\ell (w_\ell - 1) \Gamma_{\ell+1:k} \funcAw_\ell J_{\ell-1}^{\boot, 0}}^2\bigr] = \frac{1}{n} \sum_{\ell=1}^{n-1} \alpha_{\ell}^2 \PE^{\boot} \bigl[\norm{\sum_{k=\ell}^{n-1} \Gamma_{\ell+1:k} \funcAw_\ell J_{\ell-1}^{\boot, 0}}^2\bigr] \\ 
        &\leq \frac{1}{n} \bConst{A}^2 \sum_{\ell=1}^{n-1} \alpha_{\ell}^2  \norm{\sum_{k=\ell}^{n-1} \Gamma_{\ell+1:k} }^2 \PE^{\boot}\bigl[\norm{J_{\ell-1}^{\boot, 0}}^2\bigr] \overset{(a)}{\leq} n^{-1} \bConst{A}^2 \ConstC_{\ref{lem:matr_product_as_bound}}^2 \bigl(c_0 + \frac{2}{a(1-\gamma)}\bigr)^2 \sum_{\ell=1}^{n-1} \PE^{\boot}\bigl[\norm{J_{\ell-1}^{\boot, 0}}^2\bigr]\eqsp. 
    \end{align}
    Here in (a) we applied \Cref{lem:matr_product_as_bound} and \Cref{lem:summ_alpha_k} with $b = a/2$.
    Now we will provide a bound for $\PE^{\boot}\bigl[\norm{J_{\ell-1}^{\boot, 0}}^2\bigr]$ using a technique similar to the written above. \Cref{lem:expansion} and \Cref{lem:tilde_eps_boot_bound} imply that
    \begin{align}
        \PE^{\boot}\bigl[\norm{J_{\ell-1}^{\boot, 0}}^2\bigr] = \PE^{\boot}\bigl[\norm{\sum_{j=1}^{\ell-1} \alpha_j (w_j - 1) \Gamma_{j+1:\ell-1} \tilde \eps_{j}}^2\bigr] &\leq \sum_{j=1}^{\ell-1} \alpha_j^{2} \norm{\Gamma_{j+1:\ell-1}}^2 \norm{\tilde \eps_j}^2
        \leq \ConstC_{\ref{lem:matr_product_as_bound}}^2 \ConstC_{\ref{lem:tilde_eps_boot_bound}}^2 \sum_{j=1}^{\ell-1} \alpha_{j}^2 \prod_{t=j+1}^{\ell-1} (1 - \frac{a\alpha_t}{2}) \eqsp.
    \end{align}
    Therefore, we obtain using \Cref{lem:summ_alpha_k}: 
    \begin{align}        \PE^{\boot}\bigl[\norm{J_{\ell-1}^{\boot, 0}}^2\bigr] \leq \ConstC_{\ref{lem:matr_product_as_bound}}^2 \ConstC_{\ref{lem:tilde_eps_boot_bound}}^2 \alpha_{\ell} \frac{12}{a} \eqsp.
    \end{align}
    Introduce the constant
    \begin{align}
        \ConstC_{\ref{lem:jkb1:bound}}^2 = \bConst{A}^2 \ConstC_{\ref{lem:matr_product_as_bound}}^2 \bigl(c_0 + \frac{2}{a(1-\gamma)}\bigr)^2 \frac{c_0}{1-\gamma} \ConstC_{\ref{lem:tilde_eps_boot_bound}}^2  \frac{12}{a} \eqsp.
    \end{align}
    Now we obtain that
    \begin{align}
        \label{lem:Jkb1_bount_secondary}
        \PE^\boot\bigl[\norm{\frac{1}{\sqrt{n}} \sum_{k=1}^{n-1} J_k^{\boot, 1}}^2\bigr] \leq \frac{\ConstC_{\ref{lem:jkb1:bound}}^2}{n} (1-\gamma) \sum_{k=1}^{n-1} k^{-\gamma} \leq \ConstC_{\ref{lem:jkb1:bound}}^2 n^{-\gamma}
    \end{align}
    which concludes the proof.
\end{proof}

\begin{proof}[Proof of \Cref{lem:Hk0b_copies_sum_bound}.]
First, note that $\Hnalpha{k}{\boot, 0} - \Hnalpha{k}{\boot, 0, i} = 0$ if  $k < i$.
On the other hand, for $k \geq i$ we get
\begin{align}
    \Hnalpha{k}{\boot, 0} - \Hnalpha{k}{\boot, 0, i} =\underbrace{\Gamma_{i+1:k}^\boot (\Hnalpha{i}{\boot, 0} - \Hnalpha{i}{\boot, 0, i})}_{T_1^{(k)}} - \underbrace{\sum_{\ell=i+1}^{k} \alpha_\ell (w_\ell - 1) \Gamma_{\ell+1:k}^\boot \funcAw_\ell (\Jnalpha{\ell-1}{\boot, 0} - \Jnalpha{\ell-1}{\boot, 0, i})}_{T_2^{(k)}}
\end{align}
For simplicity we introduce  $v_{m:k} = \prod_{j=m}^k (1-a\alpha_j/8)$.
Consider $T_1^{(k)}$. Note that the following decomposition holds:
\begin{align}
T_1^{(k)} = \Gamma_{i+1:k}^\boot (-\alpha_i \funcAw_i(w_i 
- w'_i)\Hnalpha{i-1}{\boot, 0} - \alpha_i \funcAw_i (w_i - w'_i) \Jnalpha{i-1}{\boot, 0})\eqsp.
\end{align}
Hence, we get
\begin{align}
    \PE^\boot\bigl[n^{-1/2}\norm{\sum_{k=i}^{n-1} T_{1}^{(k)}}\bigr] \leq \ConstC_{\ref{prop:product_random_matrix_bootstrap}} \bConst{A} n^{-1/2} \sqrt{\alpha_i} (\ConstJb{0}{1} + \ConstHb{0}{1}) \sum_{k=i}^{n-1} \alpha_{i} v_{i+1:k} \eqsp .
\end{align}
Recall the definition of $\ConstC_{\ref{def:bootstrap_constants}}^{(2)}$, $\ConstC_{\ref{def:bootstrap_constants}}^{(1)}$ \eqref{def:bootstrap_constants}.
Therefore, \Cref{lem:summ_alpha_k} and \Cref{prop:Qell:bound} imply that
\begin{align}
\PE^\boot\bigl[n^{-1/2}\norm{\sum_{k=i}^{n-1} T_{1}^{(k)}}\bigr] \leq \ConstC_{\ref{prop:product_random_matrix_bootstrap}} \bConst{A} n^{-1/2} \sqrt{\alpha_i} (\ConstJb{0}{1} + \ConstHb{0}{1}) \ConstC_{\ref{def:bootstrap_constants}}^{(1)} \eqsp .
\end{align}
Consider $T_2^{(k)}$. First, we note that
\begin{align}
    \Jnalpha{\ell-1}{\boot, 0} - \Jnalpha{\ell-1}{\boot, 0, i} = \alpha_i (w_i' -w_i) \Gamma_{i+1:\ell-1} \tilde\eps_i \eqsp.
\end{align}
Thus, we rewrite the sum and get
\begin{align}
    \sum_{k=i}^{n-1} T_2^{(k)} = \sum_{\ell=i+1}^{n-1} \underbrace{\alpha_i (w_\ell-1)(w_i'-w_i) \sum_{k=\ell}^{n-1} \alpha_\ell \Gamma_{\ell+1:k}^\boot \funcAw_\ell \Gamma_{i+1:\ell-1} \tilde \eps_i}_{U_\ell} \eqsp.
\end{align}
Note that $U_\ell$ is a martingale-difference sequence w.r.t. the filtration $
    \mathcal{F}_\ell = \sigma(W_i, W_i', W_{\ell}, \ldots, W_{n-1})$.
 Minkowski's inequality and \Cref{prop:product_random_matrix_bootstrap} reveal that
\begin{align}
    \{\PE^\boot \bigl[\norm{\sum_{k=\ell}^{n-1} \alpha_\ell \Gamma_{\ell+1:k}^\boot}^2\bigr]\}^{1/2} \leq \ConstC_{\ref{def:bootstrap_constants}}^{(1)} \ConstC_{\ref{prop:product_random_matrix_bootstrap}} \eqsp.
\end{align}
Therefore, we obtain
\begin{align}
    \PE^{\boot}\bigl[\normop{U_\ell}^2\bigr] \leq \alpha_i^2 (\ConstC_{\ref{def:bootstrap_constants}}^{(1)} \ConstC_{\ref{prop:product_random_matrix_bootstrap}})^2 \bConst{A}^2 \ConstC_{\ref{lem:matr_product_as_bound}}^2 \ConstC_{\ref{lem:tilde_eps_boot_bound}}^2 \prod_{j=i+1}^{\ell-1} (1-a\alpha_j/2)  \eqsp.
\end{align}
Hence, we get using \Cref{prop:Qell:bound}:
\begin{align}
    \PE^\boot \bigl[\norm{\sum_{k=i}^{n-1} T_2^{(k)}}^2\bigr] \leq \alpha_i (\ConstC_{\ref{def:bootstrap_constants}}^{(1)} \ConstC_{\ref{prop:product_random_matrix_bootstrap}})^2 \bConst{A}^2 \ConstC_{\ref{lem:matr_product_as_bound}}^2 \ConstC_{\ref{lem:tilde_eps_boot_bound}}^2 \sum_{\ell=i}^{n-1} \alpha_i \prod_{j=i+1}^{\ell} (1-a\alpha_j/2) \leq \alpha_i (\ConstC_{\ref{def:bootstrap_constants}}^{(1)} \ConstC_{\ref{prop:product_random_matrix_bootstrap}})^2 \bConst{A}^2 \ConstC_{\ref{lem:matr_product_as_bound}}^2 \ConstC_{\ref{lem:tilde_eps_boot_bound}}^2 \ConstC_{\ref{def:bootstrap_constants}}^{(1)} \eqsp.
\end{align}
Now we combine the obtained bounds with Minkowski's inequality and finish the proof:
\begin{align}
        \{\PE^\boot\bigl[\norm{ \sum_{k=1}^{n-1} \Hnalpha{k}{\boot, 0} - \sum_{k=1}^{n-1} \Hnalpha{k}{\boot, 0, i}}^2\bigr]\}^{1/2} &\leq \sum_{j=1}^2 \{\PE^\boot \bigl[\norm{\sum_{k=i}^{n-1} T_j^{(k)}}^2\bigr]\}^{1/2} \leq \sqrt{\alpha_i} \ConstC_{\ref{lem:Hk0b_copies_sum_bound}} \eqsp,
\end{align}
where we have set
\begin{align}
    \ConstC_{\ref{lem:Hk0b_copies_sum_bound}} = \bConst{A} \ConstC_{\ref{prop:product_random_matrix_bootstrap}}  (\ConstJb{0}{1} + \ConstHb{0}{1}) \ConstC_{\ref{def:bootstrap_constants}}^{(1)} + \bConst{A} (\ConstC_{\ref{def:bootstrap_constants}}^{(1)})^{3/2} \ConstC_{\ref{prop:product_random_matrix_bootstrap}}  \ConstC_{\ref{lem:matr_product_as_bound}} \ConstC_{\ref{lem:tilde_eps_boot_bound}}  \eqsp.
\end{align}
\end{proof}

\begin{proof}[Proof of \Cref{lem:D_boot_copies_diff_bound}.]

Applying Minkowski's inequality we obtain that
    \begin{multline}
        \{\PE^\boot[\norm{D^\boot - D^{\boot, i}}^2]\}^{1/2} \leq \frac{1}{\sqrt{n}\sqrt{\ConstC_{\ref{lem:sigma_n_boot_labmda_min_lowerbound}}}} \{\PE^\boot\bigl[\norm{\sum_{k=i}^{n-1} (\Jnalpha{k, 1}{\boot, 0} - \Jnalpha{k, 1}{\boot, 0, i})}^2\bigr]\}^{1/2} \\ + \frac{1}{\sqrt{n}\sqrt{\ConstC_{\ref{lem:sigma_n_boot_labmda_min_lowerbound}}}} \{\PE^\boot\bigl[\norm{\sum_{k=i}^{n-1} (\Jnalpha{k, 1}{\boot, 0} - \Jnalpha{k, 2}{\boot, 0, i})}^2\bigr]\}^{1/2} + \frac{1}{\sqrt{n}\sqrt{\ConstC_{\ref{lem:sigma_n_boot_labmda_min_lowerbound}}}} \{\PE^\boot\bigl[\norm{\sum_{k=i}^{n-1} (\Hnalpha{k, 1}{\boot, 0} - \Hnalpha{k, 1}{\boot, 0, i})}^2\bigr]\}^{1/2} \eqsp.
    \end{multline}
    Consider the first term. Note that
    \begin{align}
        \Jnalpha{k,1}{\boot, 0} - \Jnalpha{k,1}{\boot, 0, i} = \alpha_i (w_i' - w_i) (\Gamma_{i+1:k} - G_{i+1:k})\eps_i \eqsp.
    \end{align}
    Thus, using the definition of $\Omega_{5}$ we get
    \begin{align}
        \norm{\sum_{k=i}^{n-1} (\Jnalpha{k, 1}{\boot, 0} - \Jnalpha{k, 1}{\boot, 0, i})} \leq |w_i'-w_i| \ConstC_{\ref{lem:gamma_deviation_bound}}\sqrt{\alpha_i} \log(5n) \eqsp.
    \end{align}
    Hence, it holds that
    \begin{align}
         \{\PE^\boot\bigl[\norm{\sum_{k=i}^{n-1} (\Jnalpha{k, 1}{\boot, 0} - \Jnalpha{k, 1}{\boot, 0, i})}^2\bigr]\}^{1/2} \leq \sqrt{2}  \ConstC_{\ref{lem:gamma_deviation_bound}} \sqrt{\alpha_i} \log(5n) \eqsp.
    \end{align}
    On the other hand, we obtain the following representation for $\Jnalpha{k,2}{\boot, 0} - \Jnalpha{k,2}{\boot, 0, i}$:
    \begin{align}
        \Jnalpha{k,2}{\boot, 0} - \Jnalpha{k,2}{\boot, 0, i} &= \alpha_i (w_i' - w_i) \Gamma_{i+1:k} \funcAw_i (\theta_{i-1} - \thetas) \eqsp.
    \end{align}
    Introduce the following notation
    \begin{align}
        u_{m:k} = \prod_{j=m}^k (1-a\alpha_j/2) \eqsp.
    \end{align}
    Therefore, applying Minkowski's inequality, \Cref{prop:Qell:bound}, \Cref{lem:summ_alpha_k} and using the definition of $\Omega_1$, we obtain the following:
    \begin{align}
        \{\PE^\boot\bigl[\norm{\sum_{k=i}^{n-1} (\Jnalpha{k, 2}{\boot, 0} - \Jnalpha{k, 2}{\boot, 0, i})}^2\bigr]\}^{1/2} &\leq \sqrt{2} \ConstC_{\ref{lem:matr_product_as_bound}} \bConst{A} \sum_{k=i}^{n-1} \alpha_i u_{i+1:k} ( \qcond^{1/2} \rme^2 u_{1:i-2} \norm{\theta_0 - \thetas} + 2 \rme \log(5n) \qcond^{1/2} \supconsteps \sqrt{\alpha_i}) \\ &\leq \alpha_i \frac{\sqrt{2} \qcond^{1/2} \rme^2 \ConstC_{\ref{lem:matr_product_as_bound}}  \bConst{A} \norm{\theta_0 - \thetas}}{(1-a/2)^2} \ConstC_{\ref{def:bootstrap_constants}}^{(2)} + 2\sqrt{2} \rme \ConstC_{\ref{lem:matr_product_as_bound}}  \bConst{A}  \qcond^{1/2} \supconsteps \ConstC_{\ref{def:bootstrap_constants}}^{(1)}  \log(5n) \sqrt{\alpha_i} \eqsp.
    \end{align}
    Hence, applying \Cref{lem:Hk0b_copies_sum_bound} and gathering similar terms the proof follows.
\end{proof}

\section{Proofs of products of random matrices}
\label{supplement:random_matrices}
\begin{proof}[Proof of \Cref{prop:product_random_matrix_bootstrap}.]
Our proof relies on the auxiliary result of \Cref{lem:product_ramdom_matrix_aux} below together with the blocking technique. Indeed, let us represent 
\[
k - m = N h + r\eqsp,
\]
where $r < h$ and $h = h(n)$ is a block size defined in \eqref{eq:block_size_constraint}. Then we obtain, using the independence of bootstrap weights $w_{m+1},\ldots,w_{k}$, that 
\begin{align}
\{\PEb[\norm{\ProdB^\boot_{m+1:k}}^2]\}^{1/2} 
&\leq \sqrt{\qcond} \{\PEb[\norm{\ProdB^\boot_{m+1:k}}[Q]^2]\}^{1/2} \\
&= \sqrt{\qcond} \prod_{j=1}^{N}\bigl\{\PEb[\norm{\ProdB^\boot_{m+1+(j-1)h:m+1+jh}}[Q]^2]\bigr\}^{1/2} \bigl\{\PEb[\norm{\ProdB^\boot_{m+1+Nh:k}}[Q]^2]\bigr\}^{1/2} \\
&\leq \sqrt{\qcond} \exp\bigl\{-\frac{a}{4}\sum_{\ell=m+1}^{k}\alpha_{\ell} \bigr\} \bigl\{\PEb[\norm{\ProdB^\boot_{m+1+Nh:k}}[Q]^2]\bigr\}^{1/2} \exp\bigl\{\frac{a}{4}\sum_{\ell=m+1+Nh:k}^{k}\alpha_{\ell} \bigr\}\eqsp.
\end{align}
In the last inequality we applied \Cref{lem:product_ramdom_matrix_aux} to each of the blocks of length $h$ in the first bound. It remains to upper bound the residual terms. Since the remainder block has length less then $h$, we have due to \eqref{eq:sum_steps_bound_stability} (which holds according to \Cref{assum:step-size-bootstrap}), that
\[
\exp\bigl\{\frac{a}{4}\sum_{\ell=m+1+Nh:k}^{k}\alpha_{\ell} \bigr\} \leq \exp\left\{\frac{\alpha_{\infty} a}{4}\right\} \leq \rme^{1/8}\eqsp,
\]
where the last inequality is due to \Cref{prop:hurwitz_stability}. Next,
\begin{align}
\bigl\{\PEb[\norm{\ProdB^\boot_{m+1+Nh:k}}^2]\bigr\}^{1/2} 
&\leq \qcond \prod_{\ell=m+1+Nh:k}^{k}\{\PEb[\norm{(\Id - \alpha_{\ell} w_{\ell} \Am_{\ell})}^2]\}^{1/2} \\
&\leq \qcond \prod_{\ell=m+1+Nh:k}^{k} \{\PEb[1 + 2\alpha_{\ell}|w_{\ell}| \bConst{A} + \alpha_{\ell}^2 w_{\ell}^2\bConst{A}^2]\}^{1/2}\eqsp.
\end{align}
Since $\PE[|w_{\ell}|] \leq \PE^{1/2}[w_{\ell}^2] = \bigl\{(\PE[w_{\ell}])^2 + \var[w_{\ell}]\bigr\}^{1/2} = \sqrt{2}\eqsp$, we get from previous bound 
\begin{align}
\bigl\{\PEb[\norm{\ProdB^\boot_{m+1+Nh:k}}^2]\bigr\}^{1/2} 
&\leq \qcond \prod_{\ell=m+1+Nh:k}^{k} (1 + 2\sqrt{2}\alpha_{\ell} \bConst{A} + 2\alpha_{\ell}^2 \bConst{A}^2)^{1/2} \\
&\leq \qcond \exp\bigl\{\sqrt{2} \bConst{A} \sum_{\ell=m+1+Nh:k}^{k}\alpha_{\ell}\bigr\} \leq \qcond \exp\{\sqrt{2}\bConst{A} \frac{c_0h}{k_0^\gamma}\} \leq \qcond \rme\eqsp,
\end{align}
where in the last line we additionally used \Cref{assum:step-size-bootstrap} and the inequality
\begin{align}
    \label{eq:int_bound_sum_alpha_block}
    \sum_{\ell=m+1+Nh}^{k} \alpha_\ell \leq \sum_{\ell=1}^{h} \alpha_\ell \leq c_0 \int_{k_0}^{k_0+h} x^{-\gamma} \; dx = c_0 \frac{(h+k_0)^{1-\gamma}-k_0^{1-\gamma}}{1-\gamma} \leq c_0 \frac{h}{k_0^\gamma} \eqsp.
\end{align}
\end{proof}

\begin{proof}[Proof of \Cref{lem:product_ramdom_matrix_aux}.]

Let $h \in \nset$ be a block length given in \eqref{eq:block_size_constraint}. Then the product $\ProdB^\boot_{m+1:m+h}$ writes as 
\begin{equation} 
\label{eq:split_main}
\ProdB^\boot_{m+1:m+h} = \Id - \sum_{\ell = m+1}^{m+h} \alpha_\ell \Am_\ell   - \Mat{S}  + \Mat{R} = \Id - \sum_{\ell = m+1}^{m+h} \alpha_\ell \bA - \sum_{\ell = m+1}^{m+h} \alpha_\ell (\Am_\ell - \bA) - \Mat{S} + \Mat{R}\eqsp,
\end{equation}
where $\Mat{S} =  \sum_{\ell = m+1}^{m+h} \alpha_\ell (w_\ell - 1) \Am_\ell$ is a linear statistics in $\{w_{\ell}\}_{\ell=m+1}^{m+h}$, and the remainder $\Mat{R}$ collects the higher-order terms:
\begin{equation} 
\label{eq:RlRlbar_def}
\Mat{R} = \sum_{r=2}^{h}(-1)^{r}  \sum_{(i_1,\dots,i_r)\in\msi_r}\prod_{u=1}^{r} \alpha_{i_u} w_{i_u} \Am_{i_u}\eqsp.
\end{equation}
with $\msi_r = \{(i_1,\ldots,i_r) \in \{m+1,\ldots,m+h\}^r\, : \, i_1 < \cdots < i_r \}$. We first consider the contracting part in matrix $Q$-norm. Indeed, applying \eqref{eq:contractin_q_norm}, we obtain that 
\[
\label{eq:gamma_boot_first_term_bound}
\norm{\Id - \sum\nolimits_{\ell = m+1}^{m+h} \alpha_\ell \bA}[Q]^2 \leq 1 - a \sum\nolimits_{\ell = m+1}^{m+h} \alpha_\ell\eqsp,
\]
The definition of block size $h$ combined with an integral bound \eqref{eq:int_bound_sum_alpha_block} guarantees that that $\sum_{\ell = m+1}^{m+h} \alpha_\ell \leq \alpha_{\infty}$, where $\alpha_{\infty}$ is defined in \eqref{eq:alpha_infty_def}.
Thus, we get from \eqref{eq:gamma_boot_first_term_bound} that the following bound holds
\[
\textstyle \norm{\Id - \sum_{\ell = m+1}^{m+h} \alpha_\ell \bA}[Q] \leq 1 - (a/2) \sum_{\ell = m+1}^{m+h} \alpha_\ell\eqsp.
\]
Now we need to estimate the remainders in the representation \eqref{eq:split_main}. On the set $\Omega_4$, it holds that
\begin{equation}
\label{eq:lem_hoeffding_rd}
\norm{\sum\nolimits_{\ell = m+1}^{m+h} \alpha_\ell (\Am_\ell - \bA)}[Q] \leq 2\bConst{A} \sqrt{\qcond} \bigl\{\sum_{\ell=m+1}^{m+h}\alpha_\ell^{2}\bigr\}^{1/2} \log(10 n^3d)\eqsp.
\end{equation}
Moreover, it is straightforward to check that
\[
\PEb[\norm{ \Mat{S}}[Q]^2] \leq \bConst{A}^2 \qcond \sum_{\ell = m+1}^{m+h} \alpha_\ell^2\eqsp.
\]
In order to bound the remainder term $\Mat{R}$, we note that for any $i \in \{1,\ldots,n\}$, $\PEb[|w_{i_u}|] \leq \sqrt{2}$, and  
\begin{align}
\PEb[\norm{\Mat{R}}[Q]]
&\leq \sqrt{\qcond} \sum_{r=2}^{h}\binom{h}{r} \alpha_{m+1}^{r} 2^{r/2} \bConst{A}^{r} \leq 2\alpha_{m+1}^{2} \bConst{A} ^{2} \sqrt{\qcond} \sum_{r=0}^{h-2} \binom{h}{r+2} \alpha_{m+1}^{r} 2^{r/2} \bConst{A}^{r} \\
&\leq \alpha_{m+1}^{2} h^2 \bConst{A}^{2} \sqrt{\qcond} \exp\bigl\{\sqrt{2} h\alpha_{m+1} \bConst{A}\bigr\} \leq \alpha_{m+1}^{2} h^2 \bConst{A}^{2} \sqrt{\qcond} \rme\eqsp.
\end{align}
To complete the proof it remains to set the parameter $h$ in such a way that we can guarantee the following:
\begin{equation}
\label{eq:init_step_size_constr}
\bConst{A} \sqrt{\qcond} \bigl\{\sum_{\ell = m+1}^{m+h} \alpha_\ell^2\bigr\}^{1/2} \bigl(1 + 2\log(10n^3d)\bigr) + \alpha_{m+1}^{2} h^2 \bConst{A}^{2} \sqrt{\qcond} \rme \leq \frac{a}{4}\sum_{\ell=m+1}^{m+h}\alpha_{\ell}\eqsp.
\end{equation}
Now it remains to ensure that our choice of $h$ satisfies 
\begin{equation}
\label{eq:gamma_boot_blocks_cases}
\begin{cases}
\bConst{A} \sqrt{\qcond} \bigl\{\sum_{\ell = m+1}^{m+h} \alpha_\ell^2\bigr\}^{1/2} \bigl(1 + 2\log(10n^3d)\bigr) \leq \frac{a}{8}\sum_{\ell=m+1}^{m+h}\alpha_{\ell}  \\
\alpha_{m+1}^{2} h^2 \bConst{A}^{2} \sqrt{\qcond} \rme \leq \frac{a}{8}\sum_{\ell=m+1}^{m+h}\alpha_{\ell}\eqsp.
\end{cases}
\end{equation}
Using an integral bound, we get
\begin{align}
\label{eq:sum_steps_bound_stability}
\sum_{\ell=m+1}^{m+h}\alpha_{\ell} &\geq c_0 \frac{(m+k_0+h+1)^{1-\gamma} - (m+k_0+1)^{1-\gamma}}{1-\gamma} \geq c_0 (m+k_0+1)^{1-\gamma} \frac{(2^{1-\gamma}-1)h}{(m+k_0+1)} \geq \frac{c_0(2-2^\gamma)h}{(m+k_0)^\gamma} \eqsp,
\end{align}
Similarly, one can check that
\begin{align}
\label{eq:sum_steps_squared_bound_stability2}
    \sum_{\ell=m+1}^{m+h} \alpha_\ell^2 \leq c_0^2 \frac{(m+k_0)^{1-2\gamma}-(m+k_0+h)^{1-2\gamma}}{2\gamma-1} \leq c_0^2 \frac{h}{(m+k_0)(m+k_0+h)^{2\gamma-1}}
    \leq c_0^2 \frac{h}{(m+k_0)^{2\gamma}} \eqsp.
\end{align}
Hence, taking into account \eqref{eq:sum_steps_bound_stability} and \eqref{eq:sum_steps_squared_bound_stability2}, the inequality \eqref{eq:init_step_size_constr} would follow from the bounds  
\begin{equation}
\label{eq:gamma_boot_blocks_cases_new}
\begin{cases}
\bConst{A} \sqrt{\qcond} \bigl\{\sum_{\ell = m+1}^{m+h} \alpha_\ell^2\bigr\}^{1/2} \bigl(1 + 2\log(10n^3d)\bigr) \leq \frac{ac_0(2-2^\gamma)}{8} \frac{h}{(m+k_0)^\gamma}  \\
\alpha_{m+1}^{2} h^2 \bConst{A}^{2} \sqrt{\qcond} \rme \leq \frac{ac_0(2-2^\gamma)}{8} \frac{h}{(m+k_0)^\gamma}\eqsp. 
\end{cases}
\end{equation}
Note that the latter inequalities follow from \Cref{assum:step-size-bootstrap}.
Thus, all previous inequalities will be fulfilled. Hence, the following holds:
\[
\bigl\{\PEb[\norm{\ProdB^\boot_{m+1:m+h}}[Q]^2] \bigr\}^{1/2} \leq 1 - (a/4) \sum_{\ell = m+1}^{m+h} \alpha_\ell\eqsp,
\]
and the statement follows from an elementary inequality $1 + x \leq \rme^{x}$.
\end{proof}

\section{Proof of \Cref{lem:sigma_n_bound}}
\label{supplement:lem_sigma_n_bound}
\begin{proof}[Proof of \Cref{lem:st_bound}.]
We follow the approach of \citep[pp.~25–26]{wu2024statistical}.  By the definition of $S_t$
\[
S_t = \sum_{j=t+1}^{n-1} (\alpha_t - \alpha_j)\,G_{t+1:j-1},
\]
we have
\[
\|S_t\| \;\le\;
\sum_{j=t+1}^{n-1} (\alpha_t - \alpha_j)\,\|G_{t+1:j-1}\|
\;\le\;
\sqrt{\qcond}\,c_0
\sum_{j=t+1}^{n-1}
\Bigl[(j+k_0)^{-\gamma} - (j+1+k_0)^{-\gamma}\Bigr]
e^{-\tfrac{a c_0}{2}\,g_{t+1:j}},
\]
where 
\[
g_{t+1:j} \;=\;\sum_{\ell=t+1}^j (\ell + k_0)^{-\gamma}
\;\ge\;
\int_{t+k_0}^{j+1+k_0} x^{-\gamma}\,dx
=
\frac{(j+1+k_0)^{1-\gamma}-(t+k_0)^{1-\gamma}}{1-\gamma}.
\]
Set 
\[
a_{t:j} = 1 + (1-\gamma)\,g_{t:j},
\qquad
s_{t:j} = \frac{a c_0}{2(1-\gamma)}\,a_{t:j}.
\]
Then one checks
\[
(j+k_0)^{-\gamma} - (j+1+k_0)^{-\gamma}
\;\le\;
\frac{a_{t:j+1} - a_{t:j}}{1-\gamma}\,(t+k_0)^{-\gamma},
\qquad
e^{-\tfrac{a c_0}{2}\,g_{t+1:j}}
= e^{-s_{t:j}}\,e^{\tfrac{a c_0}{2(t+k_0)^\gamma}}.
\]
Hence
\[
\|S_t\|
\;\le\;
\frac{\sqrt{\qcond}\,c_0}{1-\gamma}\,
e^{\tfrac{a c_0}{2} + \tfrac{a c_0}{2(1-\gamma)}}
\,(t+k_0)^{\gamma-1}
\sum_{j=t}^{n-2}
\bigl(a_{t:j+1}-a_{t:j}\bigr)\,
a_{t:j}^{\tfrac{\gamma}{1-\gamma}}\,
e^{-s_{t:j}}.
\]
Since $\phi(x)=x^{\frac\gamma{1-\gamma}}e^{-x}$ attains its maximum at $x_\gamma=\frac\gamma{1-\gamma}$, a discrete summation‐by‐parts gives
\[
\sum_{j=t}^{n-2}
\bigl(a_{t:j+1}-a_{t:j}\bigr)\,
a_{t:j}^{\tfrac{\gamma}{1-\gamma}}e^{-s_{t:j}}
\;\le\;
\max\{1,\phi(x_\gamma)\}\Bigl(\tfrac{a c_0}{2}+x_\gamma\Bigr)
+\int_{x_\gamma}^\infty\phi(x)\,dx.
\]
Collecting constants yields
\[
\|S_t\|\;\le\;\sqrt{\qcond}\,(t+k_0)^{\gamma-1}\,\ConstC_{\ref{lem:st_bound}},
\]
as claimed.
\end{proof}

\end{appendices}

\end{document}